\documentclass[10pt,a4paper,logo]{paper}

\usepackage[all]{hypcap}
\usepackage[numbers,compress]{natbib}

\usepackage{nicefrac}
\usepackage{multirow}
\usepackage{makecell}
\usepackage{subcaption}
\usepackage{array}
\usepackage{scalerel}
\usepackage{adjustbox}
\usepackage{float}
\usepackage{pdflscape}
\usepackage{bm}
\usepackage{algorithm}
\usepackage{algpseudocode}
\usepackage{titletoc} 
\usepackage[most,skins,theorems]{tcolorbox}
\usepackage[capitalise,noabbrev]{cleveref}

\makeatletter
\@namedef{ver@everyshi.sty}{}
\makeatother

\definecolor{title_blue}{HTML}{204899}
\definecolor{cite_blue}{HTML}{044dc1}
\definecolor{cite_purple}{HTML}{7406a7}
\definecolor{GoogleRed}{RGB}{234, 67, 53}
\definecolor{GoogleBlue}{RGB}{66, 133, 244}
\definecolor{GoogleGreen}{RGB}{52, 168, 83}
\definecolor{ourlightblue}{RGB}{232, 240, 254}

\hypersetup{
    colorlinks = true,
    citecolor  = {cite_blue},
    linkcolor  = {cite_blue},
    urlcolor   = {cite_blue},
}

\titleformat{\paragraph}[runin]{\normalfont\bfseries}{}{0pt}{#1}
\titlespacing*{\paragraph}{0pt}{0.5ex plus 0.3ex minus 0.2ex}{1em}

\newtheorem{thm}{Theorem}
\newtheorem{lem}{Lemma}
\newtheorem{prop}{Proposition}

\crefname{thm}{Theorem}{Theorems}
\crefname{lem}{Lemma}{Lemmas}
\crefname{prop}{Proposition}{Propositions}
\crefname{cor}{Corollary}{Corollaries}
\crefname{defn}{Definition}{Definitions}
\crefname{assumption}{Assumption}{Assumptions}
\crefname{algorithm}{Algorithm}{Algorithms}

\newcommand{\E}{\mathbb{E}}

\newcommand{\KL}{D_{\mathrm{KL}}}

\newcommand{\Var}{\mathrm{Var}}
\newcommand{\TV}{\mathrm{TV}}
\newcommand{\cA}{\mathcal{A}}
\newcommand{\cS}{\mathcal{S}}

\newcommand{\cD}{\mathcal{D}}

\newcommand{\R}{\mathbb{R}}

\newtcolorbox{remarkbox}[1]{
    enhanced,
    colback=ourlightblue,
    colframe=GoogleBlue,
    colbacktitle=GoogleBlue,
    coltitle=white,
    fonttitle=\bfseries,
    title={#1},
    boxrule=0.8pt,
    arc=3pt,
    left=6pt,
    right=6pt,
    top=6pt,
    bottom=6pt,
    toptitle=2pt,
    bottomtitle=2pt,
    boxed title style={
        sharp corners,
        boxrule=0pt
    }
}

\title{Trust Region Q Adjoint Matching}

\reportnumber{}

\author[1]{Yonghoon Dong}
\author[1]{Kyungmin Lee}
\author[1]{Changyeon Kim}
\author[1]{Jaehyuk Kim}
\author[1,2]{Jinwoo Shin}
\affil[1]{KAIST AI}
\affil[2]{RLWRLD}

\correspondingauthor{Correspondence to: \href{mailto:yonghoon.dong@kaist.ac.kr,jinwoos@kaist.ac.kr}{\texttt{\{yonghoon.dong, jinwoos\}@kaist.ac.kr}}.}

\begin{abstract}
Off-policy reinforcement learning of pretrained flow policies remains challenging due to the instability of optimization arising from the multi-step sampling process. Recently, Q-learning with Adjoint Matching (QAM) addressed this by recasting policy improvement as a stochastic optimal control (SOC) problem guided by a learned critic. However, QAM inherits a fundamental fragility of critic-guided improvement, since small critic errors can be exponentially amplified and often lead to performance collapse. This paper introduces Trust Region Q Adjoint Matching (TRQAM), a stable off-policy fine-tuning algorithm that adaptively controls the path-space KL between the fine-tuned and pretrained policies through projected dual descent. Specifically, we adapt a trust-region parameter $\lambda$ inside the sampling process of the flow policy and prove that $\lambda$ exactly weights the path-space KL in the SOC objective. As a result, our method can tightly control the deviation from the pretrained policy within a prescribed KL budget, achieving stable off-policy RL. Through experiments on 50 OGBench tasks, TRQAM consistently outperforms prior methods in both offline RL and offline-to-online RL. In particular, TRQAM achieves an overall success rate of 68\% in offline RL, substantially improving over the strongest baseline at 46\%.

\vspace{0.5em}
\noindent Code: \href{https://github.com/yonghdong/trqam}{\nolinkurl{github.com/yonghdong/trqam}} \quad Blog: \href{https://yonghdong.github.io/blog/trqam/}{\nolinkurl{yonghdong.github.io/blog/trqam}}
\end{abstract}

\begin{document}

\maketitle

\section{Introduction}
\label{sec:intro}
Recently, flow matching policies \citep{albergo2023stochastic,chi2024diffusionpolicy,lipman2022flow, liu2022flowstraightfastlearning} have emerged as a promising approach to model rich and diverse action 
distributions, enabling high-capacity behavior generation beyond conventional uni-modal 
Gaussian policies. A pretrained flow policy captures useful skills, behavioral 
constraints, and broad coverage over plausible actions, making it an 
attractive prior
for downstream off-policy RL 
fine-tuning \citep{bjorck2025gr00t, dong2025expo, hung2025nora15, intelligence2025pi06vlalearnsexperience, black2025pi05, jiang2025galaxea, kim2026rldx1technicalreport, reuss2025flower, shukor2025smolvla, Wagenmaker2025DSRL, zhai2025walloss, zheng2026xvla}.

However, as the flow policy is defined implicitly through a multi-step denoising process, gradient-based policy improvement requires differentiating through the multi-step sampling chain, making direct backpropagation expensive and unstable~\citep{fql_park2025, sacflow}. Existing approaches sidestep this through residual-style methods that keep the pretrained policy frozen and learn an additive residual to its actions~\citep{dong2025expo, xiao2025selfimproving}, or noise-space RL methods that freeze the pretrained flow policy and run actor-critic over its input noise~\citep{Wagenmaker2025DSRL}. The improvement is thus carried by an auxiliary module instead of the generative model itself, so these methods either compromise the expressiveness of the flow policy or cannot go beyond what the frozen pretrained policy already generates.

\begin{figure}[t]
    \centering
    \includegraphics[width=0.99\linewidth]{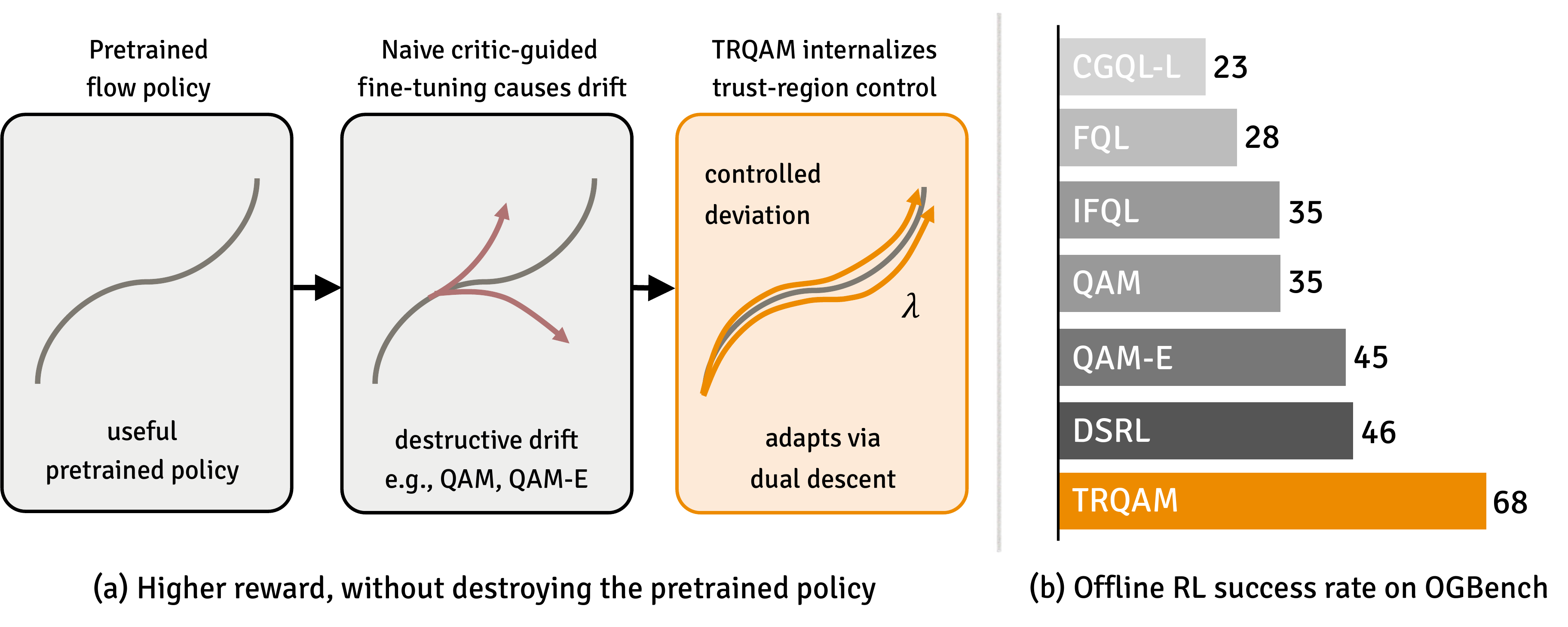}
    \vspace{-0.1in}
    \caption{\textbf{TRQAM builds adaptive trust-region control into the SOC sampling process.} \textbf{(a):} Each curve is a sampling trajectory of the flow policy, from noise $X_0$ to action $X_1$. Fixed-temperature methods such as QAM and QAM-E~\citep{qam} suffer from destructive drift, where small critic errors can be exponentially amplified into large policy errors (\cref{lem:exp_amp}). TRQAM instead bounds the deviation from the pretrained policy through a trust-region parameter $\lambda$ internalized in the SOC sampling process, which limits how far critic errors can move the policy. \textbf{(b):} Offline RL success rate on 50 OGBench~\citep{ogbench_park2025} tasks, where TRQAM outperforms all baselines.}
    \label{fig:main_figure}
    \vspace{-0.2in}
\end{figure}

Recently, Q-learning with Adjoint Matching (QAM)~\citep{qam} addressed this by recasting fine-tuning as a memoryless stochastic optimal control (SOC) problem. QAM uses a learned critic to steer the sampling process toward higher-value actions via Adjoint Matching. While this resolves the 
multi-step sampling instability, critic-induced instability still remains. 
In off-policy RL, the learned critic is inevitably imperfect. Its approximation errors compound through TD bootstrapping and produce systematic overestimation~\citep{fujimoto2018addressingfunctionapproximationerror}.
Critic-guided policy updates can then amplify these errors into large 
policy errors (see \cref{lem:exp_amp}). 
QAM~\citep{qam} acknowledges this and applies gradient clipping as a partial remedy, while calling for a more principled method beyond this heuristic. We empirically observe that gradient clipping does not prevent this fragility, as the adjoint loss can become unstable and collapse task performance on Robomimic~\citep{robomimic2021} (\cref{fig:robomimic_divergence,fig:full_robomimic_lift_curves_grid_4by2,fig:full_robomimic_can_curves_grid_4by2}). Thus, the central challenge becomes how to improve downstream performance without destructive drift from the pretrained prior, similar to the trust-region principle in on-policy RL~\citep{schulman2015trust, schulman2017proximal}.

We propose \emph{Trust Region Q Adjoint Matching} (TRQAM), an algorithm for 
stable off-policy fine-tuning of pretrained flow policies. TRQAM introduces a 
trust-region parameter $\lambda$ directly into the SOC sampling process and adapts it via projected dual descent to enforce a 
prescribed KL bound between the fine-tuned and pretrained policies. Our central theoretical result, proved using Girsanov's theorem, shows that letting $\lambda$ scale the control term in the drift makes $\lambda$ the exact weight of the path-space KL in the SOC objective. Unlike a conventional KL regularizer added to the loss, this weight lives inside the SOC sampling process, so the dual update enforces the target KL bound through the SOC formulation itself.

This distinction matters in practice. Conventional loss-level KL regularization only 
competes with critic guidance at the loss level, allowing strong critic signals 
to push the realized KL far beyond the target bound and leave the policy 
vulnerable to collapse. In contrast, TRQAM tightly tracks the target bound throughout both offline and online training 
(\cref{fig:internalization,fig:full_robomimic_lift_kl_compare_grid,fig:full_robomimic_can_kl_compare_grid,fig:full_robomimic_square_kl_compare_grid}). The prescribed target KL bound therefore provides a practical control over how much the fine-tuned policy can deviate from the pretrained policy, and its 
best setting varies systematically with task structure 
(\cref{sec:exp_sensitivity}). Through experiments on 50 OGBench tasks, 
TRQAM consistently outperforms prior methods in both offline RL and 
offline-to-online RL. In particular, TRQAM achieves an overall offline RL 
success rate of 68\%, substantially outperforming the strongest baseline at 
46\% (\cref{tab:offline_table_abb}).
\paragraph{Contributions.}
We highlight the key contributions of our paper below:
\begin{itemize}[leftmargin=*, itemsep=2pt, parsep=0pt, topsep=0pt]
    \item We identify the exponential amplification of critic errors as a fundamental fragility of fixed-temperature Adjoint Matching, formalized by \cref{lem:exp_amp} and confirmed empirically on \texttt{Robomimic}. 
    \item We prove via Girsanov (\cref{thm:path_kl}) that letting $\lambda$ scale the control term makes $\lambda$ the exact weight of the path-space KL in the SOC objective, turning $\lambda$ into a principled trust-region parameter.
    \item We propose \textit{Trust Region Q Adjoint Matching} (TRQAM), 
    which internalizes $\lambda$ inside the SOC sampling process and 
    adapts it via projected dual descent to enforce a target KL bound 
    at the sampling level rather than as a loss-level penalty.
    \item On 50 OGBench tasks, TRQAM consistently outperforms prior methods on offline RL and offline-to-online RL. In particular, TRQAM achieves an overall success rate of 68\% in offline RL, compared to 46\% for the strongest baseline.
\end{itemize}
\section{Background}
\label{sec:background}
\paragraph{Reinforcement learning.}
We consider a Markov decision process (MDP)~\citep{sutton2018reinforcement}
$(\cS,\cA,P,r,\gamma,\rho_0)$ with state space $\cS$, continuous action space $\cA \subseteq \R^d$,
transition kernel $P:\cS\times\cA\to\Delta(\cS)$, reward function $r:\cS\times\cA\to\R$,
discount factor $\gamma\in[0,1)$, and initial state distribution $\rho_0\in\Delta(\cS)$,
where $\Delta(\cdot)$ denotes the set of probability distributions.
The goal is to learn a policy $\pi:\cS\to\Delta(\cA)$ that maximizes the expected discounted return
$\E_{\rho_0,\pi,P}\big[\sum_{t\ge0}\gamma^{t}r(s_t,a_t)\big]$.
To this end, we learn a critic
$Q^{\pi}(s,a)=\E_{\pi}\big[\sum_{k\ge0}\gamma^{k}r(s_{t+k},a_{t+k})\mid s_t=s,a_t=a\big]$
that guides policy improvement. 

We focus on off-policy fine-tuning of a pretrained flow policy $\pi_{\mathrm{base}}(\cdot\mid s)$.
Training draws transitions $(s,a,r,s')$ from a replay buffer $\cD$.
The buffer initially holds a dataset collected by a possibly unknown or suboptimal behavior policy,
and $\pi_{\mathrm{base}}$ is pretrained on it by behavior cloning.
Fine-tuning first runs offline on this dataset alone.
It then continues online, where $\cD$ grows with the rollouts of the fine-tuned policy $\pi_\theta$.

\paragraph{Flow matching and flow policy.}
Flow Matching \citep{albergo2023stochastic, lipman2022flow} learns a velocity field $v_\theta(x, \tau)$ that transports samples from a simple source distribution $p_0 = \mathcal{N}(0, I_d)$ to a target distribution $p_1$ through an ODE:
\begin{equation*}
    dX_\tau = v_\theta(X_\tau, \tau)\, d\tau, \qquad X_0 \sim p_0,\quad X_1 \sim p_1\text{.}
\end{equation*}
A standard choice is the {optimal transport (OT) path}, with straight-line interpolation $X_\tau = (1-\tau) X_0 + \tau X_1$ between endpoint pairs $(X_0, X_1) \sim p_0 \otimes p_1$ and target velocity 
$\mathbb{E}[X_1 - X_0 | X_\tau = x]$.
Then, the velocity field is trained by regressing $v_\theta$ against the target velocity $X_1 - X_0$, giving
\begin{equation*}
    \mathcal{L}_{\mathrm{FM}}(\theta) = \mathbb{E}_{\tau,\, X_0,\, X_1}\!\Big[\, \big\| v_\theta(X_\tau, \tau) - (X_1 - X_0) \big\|^2 \,\Big]\text{.}
\end{equation*}
A flow policy $\pi_\theta(\cdot \mid s)$ applies flow matching to action generation, sampling an action $X_1$ for a state $s$ by integrating $dX_\tau = v_\theta(X_\tau, \tau; s)\, d\tau$ from $X_0 \sim \mathcal{N}(0, I_d)$.

\paragraph{Stochastic optimal control and Adjoint Matching.}
Stochastic optimal control (SOC) is a framework that fine-tunes a pretrained flow model by steering its sampling process with a learned control $u$ toward low terminal cost. Training the model directly would require backpropagating through the multi-step sampling chain, which is costly and unstable, much like backpropagation through time in RNNs~\citep{fql_park2025, sacflow}. \citet{domingoenrich2025adjointmatchingfinetuningflow} showed that the SOC objective can instead be solved with a lean adjoint ODE, which avoids backpropagating through the sampling chain and yields the Adjoint Matching algorithm.

Concretely, the flow ODE is replaced with an SDE that has the same marginal distribution at every $\tau \in [0,1]$~\citep{ma2024sit}, so its terminal $X_1$ still follows $\pi_{\mathrm{base}}$. Since this SDE is not unique, Adjoint Matching fixes its diffusion coefficient to $\sqrt{2(1-\tau)/\tau}$, which it calls the memoryless OT schedule. Under this schedule, $X_0$ is independent of $X_1$, so the optimum of the SOC objective below exponentially tilts $\pi_{\mathrm{base}}$ toward low terminal cost. Following \citet{domingoenrich2023stochastic}, we write this coefficient as $\sqrt{\lambda}\,\sigma(\tau)$. Because the coefficient itself is fixed, $\lambda$ enters only through $\sigma(\tau) = \sqrt{2(1-\tau)/(\lambda\tau)}$, which scales the control term $\sigma(\tau)\,u$ in the drift.
This gives the base and controlled SDEs
\begin{align}
    &d X_\tau^{\mathrm{base}} = b(X_\tau^{\mathrm{base}}, \tau)\, d\tau + \sqrt{\lambda}\, \sigma(\tau)\, d B_\tau\text{,} \label{eq:prelim_base_sde} \\
    &d X_\tau^u = \bigl(\, b(X_\tau^u, \tau) + \sigma(\tau)\, u(X_\tau^u, \tau)\, \bigr)\, d\tau + \sqrt{\lambda}\, \sigma(\tau)\, d B_\tau\text{,} \label{eq:prelim_controlled_sde}
\end{align}
where $b(x, \tau) = 2\, v^{\mathrm{base}}(x, \tau) - x/\tau$ is the drift given by the pretrained velocity field $v^{\mathrm{base}}$ under the memoryless OT schedule~\citep{domingoenrich2025adjointmatchingfinetuningflow}. The base SDE generates the pretrained policy $\pi_{\mathrm{base}}$. The control $u$ modifies this process, and its magnitude determines how far the resulting policy moves from $\pi_{\mathrm{base}}$.

Given a terminal cost $g : \mathbb{R}^d \to \mathbb{R}$, SOC seeks the control $u$ by optimizing
\begin{align}
    \min_{u}\,\,\mathbb{E}_{\mathbf X \sim \mathbb{P}^u}\bigg[\,
        \frac{1}{2} \int_0^1 \big\|u(X_\tau^u, \tau)\big\|^2\, d\tau
        + g(X_1^u)
    \,\bigg]
    \quad \text{subject to \cref{eq:prelim_controlled_sde}.}
    \label{eq:prelim_soc}
\end{align}
Q-learning with Adjoint Matching (QAM)~\citep{qam} casts off-policy RL in this framework by taking $X_1^u$ as the action at state $s$ and the negative critic $-Q^\pi(s, \cdot)$ as the terminal cost $g$. The control $u$ then steers the pretrained policy toward critic-preferred actions, while the quadratic term penalizes deviation from $\pi_{\mathrm{base}}$~\citep{domingoenrich2025adjointmatchingfinetuningflow}. QAM solves \cref{eq:prelim_soc} with Adjoint Matching in two steps. First, it propagates the critic gradient at the final action back along the sampling trajectory. This gives the adjoint $\widetilde{a}_\tau$, obtained by integrating the following lean adjoint ODE backward in time, where the term in parentheses is the base drift $b$ of \cref{eq:prelim_base_sde},
\begin{align}
    \widetilde{a}_{\tau-h} = \widetilde{a}_\tau + h\, \widetilde{a}_\tau^\top \nabla_x\!\left( 2 v^{\mathrm{base}}(X_\tau, \tau) - \frac{1}{\tau} X_\tau \right),
    \qquad
    \widetilde{a}_1 = -\nabla_{x_1} Q^\pi(s, X_1)\text{.}
    \label{eq:prelim_adjoint}
\end{align}
Second, the fine-tuned velocity field $v^{\mathrm{ft}}_\theta$, whose deviation from $v^{\mathrm{base}}$ defines the control $u$, is updated by minimizing the Adjoint Matching loss
\begin{align}
    \mathcal{L}_{\mathrm{Adj\text{-}Match}}(\theta)
    =
    \sum_{\tau \in \{0, \ldots, 1-h\}}
    \bigg\|
        \frac{2}{\sigma(\tau)}
        \bigl( v^{\mathrm{ft}}_\theta(X_\tau, \tau) - v^{\mathrm{base}}(X_\tau, \tau) \bigr)
        + \sigma(\tau)\, \widetilde{a}_\tau
    \bigg\|^2\text{,}
    \label{eq:prelim_adj_loss}
\end{align}
without backpropagating through the sampling chain. Unlike our formulation, however, QAM does not introduce $\lambda$ in \cref{eq:prelim_controlled_sde}, so it can be viewed as the special case of ours with $\lambda = 1$.
\section{Method}
\label{sec:method}
In this section, we introduce \emph{Trust Region Q Adjoint Matching} (TRQAM), a framework 
for stable off-policy fine-tuning of pretrained flow policies. We first identify that small critic errors can be exponentially amplified into large policy errors, often leading to performance collapse. To address this, TRQAM limits the deviation from the pretrained policy with a trust-region parameter $\lambda$, which scales the control term in the drift (\cref{eq:prelim_controlled_sde}). Concretely, we show that this deviation is bounded by the path-space KL (\cref{prop:dpi}), and that $\lambda$ is exactly the weight of this KL in the SOC objective by Girsanov's theorem (\cref{thm:path_kl}), which makes $\lambda$ a natural dual variable for a KL budget. TRQAM adapts $\lambda$ by dual descent throughout training to keep this KL within a prescribed budget $\varepsilon_{\mathrm{KL}}$. Since this KL can be estimated at no extra cost, this trust region adds no overhead over QAM.

We develop our core contribution by investigating the following questions:
\begin{itemize}[leftmargin=*, itemsep=2pt, parsep=0pt, topsep=0pt]
    \item Why does a fixed $\lambda$ fail against critic errors in off-policy RL? (\cref{sec:method_instability})
    \item How can $\lambda$ limit the effect of these critic errors? (\cref{sec:method_path_kl})
    \item How do we adapt $\lambda$ throughout training? (\cref{sec:method_dual_update})
    \item Why must $\lambda$ act inside the SOC sampling process rather than through an external KL penalty? (\cref{sec:method_internal_external})
\end{itemize}
\subsection{Why does a fixed \texorpdfstring{$\lambda$}{lambda} fail against critic errors in off-policy RL?}
\label{sec:method_instability}

In off-policy RL, the learned critic is inevitably imperfect. Its approximation error compounds through bootstrapping and is especially severe under distributional shift \citep{fujimoto2018addressingfunctionapproximationerror, kumar2020conservativeqlearningofflinereinforcement}. The central risk of critic-guided policy improvement is that small critic errors can induce large deviations from the policy the true critic would produce, which we call policy errors. We formalize this risk in the following lemma, which applies to any policy obtained by exponentially tilting the pretrained policy $\pi_{\mathrm{base}}$. This class includes QAM, since the optimum of its memoryless SOC objective has this form
\citep{domingoenrich2025adjointmatchingfinetuningflow, qam}.

\begin{lem}[Exponential amplification of critic errors]
\label{lem:exp_amp}
Fix a state $s \in \mathcal{S}$. Let $Q: \mathcal{A} \to \mathbb{R}$ be the true critic and $\widetilde{Q}: \mathcal{A} \to \mathbb{R}$ a learned critic with $\|Q - \widetilde{Q}\|_\infty \leq \varepsilon$. Define the policies obtained by exponentially tilting $\pi_{\mathrm{base}}$ with each critic,
\begin{align*}
    \pi_Q(a \mid s) \propto \pi_{\mathrm{base}}(a \mid s)\, e^{\beta Q(a)}\text{,}
    \qquad
    \pi_{\widetilde{Q}}(a \mid s) \propto \pi_{\mathrm{base}}(a \mid s)\, e^{\beta \widetilde{Q}(a)}\text{,}
\end{align*}
where $\beta > 0$ is the inverse temperature. Then the following inequalities hold:
\begin{align*}
    D_{\mathrm{KL}}(\pi_Q \,\|\, \pi_{\widetilde{Q}}) \leq 2\beta\varepsilon\text{,}
    \qquad
    \mathrm{TV}(\pi_Q, \pi_{\widetilde{Q}}) \leq \frac{1}{2}\bigl(e^{2\beta\varepsilon} - 1\bigr)\text{.}
\end{align*}
\end{lem}
\begin{proof}
    See \cref{app:exp_amp} for the full proof.
\end{proof}
\begin{figure}[t]
    \centering
    \includegraphics[width=0.98\linewidth]{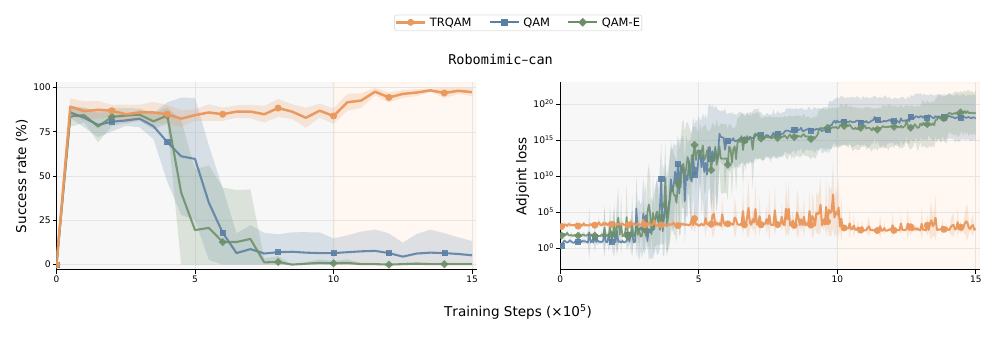}
    \vspace{-0.15in}
    \caption{\textbf{Empirical fragility of fixed-temperature Adjoint Matching.} On 
    \texttt{Robomimic-can}, the adjoint-matching loss in QAM and QAM-E grows above 
    $10^{20}$ even with gradient clipping (min--max across seeds), driving 
    task success from over $80\%$ to near zero ($\pm 1$ standard deviation), while TRQAM remains 
    stable. This collapse persists across most hyperparameter settings we tested on both \texttt{Robomimic-lift} and \texttt{Robomimic-can}. Appendix \cref{fig:full_robomimic_lift_curves_grid_4by2,fig:full_robomimic_can_curves_grid_4by2} show the full sweep.}
    \label{fig:robomimic_divergence}
\end{figure}
Since the total-variation bound grows exponentially in $\beta\varepsilon$, a large $\beta$ amplifies critic errors into large policy errors. Lowering $\beta$ avoids this amplification, but it also tilts $\pi_{\mathrm{base}}$ less toward high-value actions, so the policy improves little even when the critic is accurate. The lemma thus shows that no single fixed $\beta$ can both exploit a reliable critic and protect against an unreliable one. We observe this empirically on Robomimic~\citep{robomimic2021}, where fixed-temperature QAM and QAM-E show a diverging adjoint loss and a collapse in task success even with gradient clipping (\cref{fig:robomimic_divergence}).

The same tension arises for $\lambda$ in our SOC setting, where $\lambda$ controls how strongly the policy is held close to $\pi_{\mathrm{base}}$. Like a large $\beta$, a small $\lambda$ lets the policy deviate far from $\pi_{\mathrm{base}}$, while a large $\lambda$ keeps the policy close to $\pi_{\mathrm{base}}$. Hence no single fixed $\lambda$ resolves this tension either, which motivates adapting $\lambda$ throughout training. Doing so requires relating $\lambda$ to $\beta$, and thereby to the critic errors of \cref{lem:exp_amp}, which we establish next.

\subsection{How \texorpdfstring{$\lambda$}{lambda} limits critic errors through the KL}
\label{sec:method_path_kl}
We build this relation in steps, starting from how $\lambda$ enters the sampling process. In the controlled SDE in \cref{eq:prelim_controlled_sde}, the total diffusion $\sqrt{\lambda}\,\sigma(\tau)$ is fixed by the memoryless schedule, so $\lambda$ enters only through $\sigma(\tau) \propto 1/\sqrt{\lambda}$. Since $\sigma(\tau)$ multiplies the control $u$ in the drift of the controlled SDE, a larger $\lambda$ keeps the controlled SDE closer to the base SDE in \cref{eq:prelim_base_sde}. We formalize this with Girsanov's theorem, showing that the quadratic control cost equals the path-space KL between the fine-tuned and base trajectory distributions, with $\lambda$ as the inverse coefficient.

\begin{thm}[SOC control cost $=$ path-space KL]
\label{thm:path_kl}
Let $\mathbb{P}^u$ and $\mathbb{P}^{\mathrm{base}}$ denote the distributions over trajectories induced by the controlled and base SDEs in \cref{eq:prelim_controlled_sde} and \cref{eq:prelim_base_sde}, respectively. Then, 
\begin{align*}
    D_{\mathrm{KL}}\bigl(\mathbb{P}^u \,\|\, \mathbb{P}^{\mathrm{base}}\bigr)
    \;=\;
    \mathbb{E}_{\mathbf X \sim \mathbb{P}^u}\!\bigg[\,
        \frac{1}{2\lambda} \int_0^1 \|u(X_\tau, \tau)\|^2\, d\tau
    \,\bigg]\text{.}
\end{align*}
\end{thm}
\begin{proof}
See \cref{app:girsanov} for the full proof.
\end{proof}

\cref{thm:path_kl} thus rewrites the SOC objective in \cref{eq:prelim_soc} as $\lambda\, D_{\mathrm{KL}}(\mathbb{P}^u \,\|\, \mathbb{P}^{\mathrm{base}}) + \E[g(X_1^u)]$, so $\lambda$ is exactly the weight of the path-space KL within the SOC objective itself. This path-space KL is itself a natural trust region for flow policies, since it constrains the whole generation trajectory rather than only the terminal action, in line with trajectory-level formulations~\citep{ren2024diffusionpolicypolicyoptimization, zhang2026reinflowfinetuningflowmatching}. The inverse temperature $\beta$, however, acts on the terminal action distribution $\pi_\theta(\cdot \mid s)$. To connect the two, the following proposition shows that the path-space KL upper-bounds the terminal KL $D_{\mathrm{KL}}(\pi_\theta \,\|\, \pi_{\mathrm{base}})$.

\begin{prop}[Terminal KL upper-bounded by path-space KL]
\label{prop:dpi}
Let $\mathbb{P}^u$ and $\mathbb{P}^{\mathrm{base}}$ denote the distributions over trajectories induced by the controlled and base SDEs in \cref{eq:prelim_controlled_sde} and \cref{eq:prelim_base_sde}, respectively, and let $\pi_\theta(\cdot \mid s)$ and $\pi_{\mathrm{base}}(\cdot \mid s)$ denote the corresponding terminal action distributions at $\tau = 1$. Then,
\begin{align*}
    D_{\mathrm{KL}}\bigl(\pi_\theta(\cdot \mid s) \,\|\, \pi_{\mathrm{base}}(\cdot \mid s)\bigr)
    \;\leq\;
    D_{\mathrm{KL}}\bigl(\mathbb{P}^u \,\|\, \mathbb{P}^{\mathrm{base}}\bigr)\text{.}
\end{align*}
\end{prop}
\begin{proof}
See \cref{app:dpi} for the full proof.
\end{proof}

\cref{prop:dpi} bounds this terminal KL, but does not say how it depends on $\beta$.
To make this dependence explicit, we use a property of exponential tilting.
It is separate from \cref{lem:exp_amp}, which concerns the KL between two tilted policies under critic error.

\begin{prop}[Tilted-policy KL is non-decreasing in $\beta$]
\label{prop:tilt_mono}
Fix a state $s \in \cS$ and a bounded critic $Q: \cA \to \R$. For $\beta > 0$, define the exponentially tilted policy and its log-partition function
\begin{align*}
    \pi_\beta(a \mid s) = \frac{\pi_{\mathrm{base}}(a \mid s)\, e^{\beta Q(a)}}{Z(\beta)}\text{,}
    \qquad
    \psi(\beta) = \log Z(\beta)\text{,}
\end{align*}
where $Z(\beta) = \E_{a \sim \pi_{\mathrm{base}}}[e^{\beta Q(a)}]$. Then the following identities hold:
\begin{align*}
    D_{\mathrm{KL}}(\pi_\beta \,\|\, \pi_{\mathrm{base}}) = \beta\psi'(\beta) - \psi(\beta)\text{,}
    \qquad
    \frac{d}{d\beta} D_{\mathrm{KL}}(\pi_\beta \,\|\, \pi_{\mathrm{base}}) = \beta \Var_{\pi_\beta}[Q] \geq 0\text{.}
\end{align*}
In particular, $D_{\mathrm{KL}}(\pi_\beta \,\|\, \pi_{\mathrm{base}})$ is non-decreasing in $\beta$.
\end{prop}
\begin{proof}
See \cref{app:tilt_mono} for the full proof.
\end{proof}

Under Adjoint Matching (AM), the terminal distribution of the optimal controlled process is this tilted policy~\citep{domingoenrich2025adjointmatchingfinetuningflow}, so training drives $\pi_\theta$ toward it.
The terminal KL of \cref{prop:dpi} therefore satisfies $D_{\mathrm{KL}}(\pi_\theta \,\|\, \pi_{\mathrm{base}}) \approx D_{\mathrm{KL}}(\pi_\beta \,\|\, \pi_{\mathrm{base}})$, which grows with $\beta$.

Together, these results relate $\lambda$ and $\beta$ through the terminal KL, as summarized below.
\begin{remarkbox}{Remark: Connection between trust-region parameter $\lambda$ and inverse temperature $\beta$}
\[
\begin{aligned}
    \lambda\uparrow
    \;&\underset{\text{Thm.~\ref{thm:path_kl}}}{\Longrightarrow}\;
    \underbrace{D_{\mathrm{KL}}(\mathbb{P}^u \| \mathbb{P}^{\mathrm{base}})}_{\text{path-space KL}}\downarrow
    \;\;\underset{\text{Prop.~\ref{prop:dpi}}}{\Longrightarrow}\;
    \underbrace{D_{\mathrm{KL}}(\pi_\theta \| \pi_{\mathrm{base}})}_{\text{terminal KL}}\downarrow \\[2pt]
    \beta\uparrow
    \;&\Longrightarrow\;
    \underbrace{D_{\mathrm{KL}}(\pi_\theta \| \pi_{\mathrm{base}})}_{\substack{\text{terminal KL}\\\text{(AM, Prop.~\ref{prop:tilt_mono})}}}\uparrow
    \quad\text{and}\quad
    \underbrace{\tfrac{1}{2}\bigl(e^{2\beta\varepsilon}-1\bigr)}_{\substack{\text{critic-error amplification}\\\text{(Lem.~\ref{lem:exp_amp})}}}\uparrow
\end{aligned}
\]
\end{remarkbox}
In short, $\lambda$ and $\beta$ move the fine-tuned policy $\pi_\theta$ in opposite directions relative to $\pi_{\mathrm{base}}$. A larger $\lambda$ keeps $\pi_\theta$ closer to $\pi_{\mathrm{base}}$, while a larger $\beta$ pushes it further away and also amplifies critic errors (\cref{lem:exp_amp}). This confirms the claim of \cref{sec:method_instability} that a fixed $\lambda$ runs into the same problem as a fixed $\beta$. What $\lambda$ offers in addition is an exact link to the path-space KL (\cref{thm:path_kl}). This KL therefore gives a measurable budget on how far the policy moves from $\pi_{\mathrm{base}}$. Rather than fixing $\lambda$, we treat it as the dual variable of this budget and adapt it throughout training so that the path-space KL stays within a prescribed $\varepsilon_{\mathrm{KL}}$. The next section makes this concrete with a tractable KL estimator and a dual update for $\lambda$.

\subsection{Adaptive \texorpdfstring{$\lambda$}{lambda} via projected dual descent}
\label{sec:method_dual_update}
To keep the realized path-space KL within a target bound 
$\varepsilon_{\mathrm{KL}}$, we need a tractable KL estimator and a
principled rule for updating $\lambda$ based on it. We use the fact that the discretized memoryless OT
sampler with $T$ steps of size $h = 1/T$ on the grid $\tau_k = kh$ and diffusion schedule
$\nu(\tau) = \sqrt{2(1-\tau)/\tau}$ is a Markov chain whose Gaussian
transitions share the same covariance, so per-step KL divergences
admit a closed form. Summing these per-step KLs along a trajectory
approximately recovers the path-space KL, which we estimate via Monte Carlo over
sampled trajectories (derivation in \cref{app:pathkl}),
\begin{align}
    \widehat{D}_n
    = \mathbb{E}_{\mathbf X \sim \mathbb{P}^u}\!\bigg[\sum_{k=0}^{T-1}
    \frac{2h}{\nu(\tau_k)^2}
    \bigl\| v^{\mathrm{ft}}_\theta(X_{\tau_k}, \tau_k)
           - v^{\mathrm{base}}(X_{\tau_k}, \tau_k) \bigr\|^2\bigg]\text{.}
    \label{eq:method_kl_estimator}
\end{align}
This estimate comes at no extra cost, since the velocity difference $v^{\mathrm{ft}}_\theta - v^{\mathrm{base}}$ is already computed for the Adjoint Matching loss in \cref{eq:prelim_adj_loss}. To reduce variance, we smooth $\widehat{D}_n$ with an exponential 
moving average: $\overline{D}_n \leftarrow (1-\rho_D)\, \overline{D}_{n-1} 
+ \rho_D\, \widehat{D}_n$. Given this estimator, we adapt $\lambda$ by 
interpreting it as the dual variable of a KL-constrained improvement 
problem $\max_u \mathbb{E}[Q^\pi(X_1^u)]$ subject to 
$D_{\mathrm{KL}}(\mathbb{P}^u \| \mathbb{P}^{\mathrm{base}}) \leq 
\varepsilon_{\mathrm{KL}}$, and apply projected dual descent with a 
fixed step size $\eta_\lambda > 0$ (derivation in 
\cref{app:dual}),
\begin{align}
    \lambda_{n+1}
    \;\leftarrow\;
    \max\bigl\{0,\; \lambda_n + \eta_\lambda\,(\overline{D}_n - 
    \varepsilon_{\mathrm{KL}})\bigr\}\text{.}
    \label{eq:method_lambda_update}
\end{align}
Intuitively, when the realized KL exceeds the bound, $\lambda$ rises and the controlled SDE becomes more conservative. When it falls below, $\lambda$ decreases and allows more aggressive improvement.

\cref{alg:qam,alg:trqam} compare QAM and TRQAM side by side. TRQAM keeps every step of QAM and adds three changes, highlighted in blue. It uses the current $\lambda_n$ in the Adjoint Matching loss through $\sigma(\tau)$, estimates the path-space KL from the velocity difference already computed in the Adjoint Matching loss, and updates $\lambda$ by dual descent. \cref{app:algorithm} gives the full algorithm.
\begin{figure}[t]
\centering
\begin{minipage}[t]{0.49\linewidth}
\begin{algorithm}[H]
\caption{QAM}
\label{alg:qam}
\begin{algorithmic}[1]
\small
\Require $v^{\mathrm{base}}$, $Q^\pi_\phi$, training steps $N$
\State $v^{\mathrm{ft}}_\theta \leftarrow v^{\mathrm{base}}$
\For{$n = 0, \ldots, N-1$}
    \State TD update of $Q^\pi_\phi$
    \State Sample a trajectory by $v^{\mathrm{ft}}_\theta$
    \State Solve the lean adjoint ODE in \cref{eq:prelim_adjoint}
    \State Update $\theta$ by \cref{eq:prelim_adj_loss}
    \State \phantom{Estimate the path-space KL by Equation~(6)}
    \State \phantom{Dual update of $\lambda_n$ by Equation~(7)}
\EndFor
\end{algorithmic}
\end{algorithm}
\end{minipage}
\hfill
\begin{minipage}[t]{0.49\linewidth}
\begin{algorithm}[H]
\caption{\textcolor{GoogleBlue}{TRQAM (ours)}}
\label{alg:trqam}
\begin{algorithmic}[1]
\small
\Require $v^{\mathrm{base}}$, $Q^\pi_\phi$, \textcolor{GoogleBlue}{KL budget $\varepsilon_{\mathrm{KL}}$}, training steps $N$
\State $v^{\mathrm{ft}}_\theta \leftarrow v^{\mathrm{base}}$, \textcolor{GoogleBlue}{$\lambda \leftarrow \lambda_0$}
\For{$n = 0, \ldots, N-1$}
    \State TD update of $Q^\pi_\phi$
    \State Sample a trajectory by $v^{\mathrm{ft}}_\theta$
    \State Solve the lean adjoint ODE in \cref{eq:prelim_adjoint}
    \State Update $\theta$ by \cref{eq:prelim_adj_loss} \textcolor{GoogleBlue}{with $\lambda_n$}
    \State \textcolor{GoogleBlue}{Estimate the path-space KL by \cref{eq:method_kl_estimator}}
    \State \textcolor{GoogleBlue}{Dual update of $\lambda_n$ by \cref{eq:method_lambda_update}}
\EndFor
\end{algorithmic}
\end{algorithm}
\end{minipage}
\caption{Simplified versions of the QAM and TRQAM algorithms. The parts TRQAM adds are marked in \textcolor{GoogleBlue}{blue}, and \cref{app:algorithm} gives the full algorithm.}
\label{fig:alg_comparison}
\vspace{-0.1in}
\end{figure}

\subsection{Internal vs.\ external KL regularization}
\label{sec:method_internal_external}
The dual update in \cref{eq:method_lambda_update} can be combined with the SOC objective in two ways,
\begin{align*}
\underbrace{\min_{\theta}\; \mathcal{L}_{\text{Adj-Match}}(\theta) + 
\lambda \cdot \overline{D}_n(\theta)}_{\text{External: } \lambda \text{ 
is a regularization weight}}
\quad\text{vs.}\quad
\underbrace{\min_{\theta}\; \mathcal{L}_{\text{Adj-Match}}(\theta) 
\;\;\text{with } \lambda \text{ in \cref{eq:prelim_controlled_sde}}}_{\text{Internal 
(TRQAM): } \lambda \text{ appears in the SOC sampling process}}\text{.}
\end{align*}
In the internal form (right), $\lambda$ acts only on the control term $\sigma(\tau)\,u$ of the controlled SDE, as shown in \cref{sec:method_path_kl}. Increasing $\lambda$ therefore pulls the controlled SDE toward the base SDE, and \cref{thm:path_kl} ties the path-space KL to $\lambda$ by an exact identity. The dual update thus enforces the trust region through the sampling process itself. However, the external form (left) uses $\lambda$ only to weight a loss penalty, which must compete with the Adjoint Matching loss during optimization.  Under strong critic guidance, the critic term can dominate this penalty, and the realized KL drifts far from $\varepsilon_{\mathrm{KL}}$. \cref{tab:internal-vs-external} in \cref{app:int-vs-ext} summarizes these differences, and \cref{sec:exp_mechanism} validates them empirically.
\section{Experiments}
\label{sec:experiments}
We evaluate TRQAM on off-policy fine-tuning of pretrained flow policies in the offline-to-online setting. We use OGBench~\citep{ogbench_park2025} (50 tasks) for main comparison, and Robomimic~\citep{robomimic2021} for ablation and mechanism studies. All methods share the same pretrained flow policy and training schedule.

\paragraph{Setup.}

OGBench~\citep{ogbench_park2025} is an offline goal-conditioned RL benchmark spanning 10 domains, from which we evaluate on 50 tasks. While OGBench is originally designed 
for offline goal-conditioned RL, we use its reward-based single-task variants. Robomimic~\citep{robomimic2021} is a demonstration-based manipulation benchmark used to test the stability.
For all manipulation tasks (\emph{e.g.}, OGBench's \texttt{scene}, \texttt{cube}, \texttt{puzzle} 
domains and all Robomimic tasks), we use action-chunked policies with chunk size 5~\citep{li2025reinforcementlearningactionchunking}.

On OGBench, we compare our method against six off-policy 
fine-tuning baselines, namely FQL~\citep{fql_park2025}, 
CGQL-Linex~\citep{dhariwal2021diffusion}, 
DSRL~\citep{Wagenmaker2025DSRL}, IFQL~\citep{hansen2023idql}, 
QAM, and QAM-E~\citep{qam}. On Robomimic, we focus on 
the adjoint-matching variants (\emph{i.e.}, QAM and QAM-E), and DSRL as a non-adjoint reference. We refer to \cref{appendix:baselines} for details on each baseline.
Since trust-region 
fine-tuning regulates deviation from a pretrained prior, this prior must itself encode meaningful behavior. Therefore, all methods are pretrained for 300K steps with behavior cloning, and are then fine-tuned for $10^6$ offline steps followed by $0.5 \times 10^6$ online steps.
We report average success rate (\%) over 8 seeds, and the detailed hyperparameters are in \cref{app:hyperparams}.

\paragraph{Abbreviations.}
We evaluate on 10 OGBench domains and abbreviate their names in tables and figures for compactness:
\texttt{puzzle-4x4} (\texttt{p44}), \texttt{cube-double} (\texttt{c2}), \texttt{cube-triple} (\texttt{c3}), \texttt{cube-quadruple} (\texttt{c4}), \texttt{scene}, \texttt{humanoidmaze-medium} (\texttt{hm}), \texttt{humanoidmaze-large} (\texttt{hl}), 
\texttt{antmaze-large} (\texttt{al}), \texttt{antmaze-giant} (\texttt{ag}), and \texttt{puzzle-3x3} (\texttt{p33}).
\subsection{Main results on offline and offline-to-online RL}
\label{sec:exp_main}
\cref{tab:offline_table_abb} reports offline success rates after 
1M training steps. TRQAM achieves $68\%$ aggregate success across 50 
tasks, improving on QAM ($35\%$) by 33 points, on its strongest 
variant QAM-E ($45\%$) by 23 points, and on the strongest non-adjoint 
baseline DSRL ($46\%$) by 22 points, with the largest gains on 
long-horizon and combinatorial domains. This lead is sustained through 
the offline-to-online transition. The per-task curves in 
\cref{fig:per_task_curves_1,fig:per_task_curves_2} 
show that TRQAM remains the strongest method through 500K steps of 
online fine-tuning.

A natural explanation for TRQAM's gains is that it leverages the pretrained policy more effectively. \cref{fig:pretraining} 
tests this by running TRQAM, QAM-E, and QAM both from a pretrained 
flow policy (dashed) and from scratch (solid) under an identical 
offline-to-online protocol. The contrast is clear. TRQAM benefits 
substantially from the pretrained prior and reaches high success much 
earlier than its scratch counterpart. In contrast, QAM and QAM-E show little 
to no benefit from the same pretrained initialization, and their 
pretrained and scratch curves remain close throughout training.
\renewcommand{\arraystretch}{1.2} 
\setlength{\tabcolsep}{4pt}
\begin{table}[t]
    \centering
    \caption{\textbf{Offline RL on 50 OGBench~\cite{ogbench_park2025} tasks
    at 1M training steps (8 seeds).}
    Mean success rate (\%) with $\pm$1 standard deviation.
    Per-task results for all 50 tasks are in \cref{tab:full-results}.}
    \label{tab:offline_table_abb}
    \setlength{\abovecaptionskip}{4pt}
    \newcommand{\gval}[2]{{\color{gray!70}#1}\,{\color{gray!50}\scriptsize$\pm$#2}}
    \newcommand{\bmax}[2]{\underline{#1}\,{\color{gray!50}\scriptsize$\pm$#2}}
    
    \newcommand{\gavg}[1]{{\color{gray!70}#1}}
    \newcommand{\bavg}[1]{\underline{#1}}
    \newcommand{\ourmax}[2]{{\color{GoogleBlue}\textbf{\underline{#1}}}\,{\color{gray!50}\scriptsize$\pm$#2}}
    \newcommand{\ouravg}[1]{{\color{GoogleBlue}\textbf{\underline{#1}}}}
    
\resizebox{0.9\linewidth}{!}{
\begin{tabular}{>{\centering\arraybackslash}m{2.6cm} l ccccccccccc}
\toprule
 &  & \texttt{al} & \texttt{ag} & \texttt{hm} & \texttt{hl} & \texttt{scene} & \texttt{p33} & \texttt{p44} & \texttt{c2} & \texttt{c3} & \texttt{c4} & \texttt{all} \\
 &  & \tiny\texttt{5 tasks} & \tiny\texttt{5 tasks} & \tiny\texttt{5 tasks} & \tiny\texttt{5 tasks} & \tiny\texttt{5 tasks} & \tiny\texttt{5 tasks} & \tiny\texttt{5 tasks} & \tiny\texttt{5 tasks} & \tiny\texttt{5 tasks} & \tiny\texttt{5 tasks} & \tiny\texttt{50 tasks} \\
\midrule

\textsc{Backprop} & {\texttt{FQL}}
& \gval{38}{9} & \gval{2}{6} & \gval{74}{5} & \gval{2}{1} & \gval{70}{5} & \gval{25}{10} & \gval{9}{7} & \gval{44}{4} & \gval{7}{5} & \gval{9}{5} & \gavg{28} \\
\midrule

\textsc{Guidance} & {\texttt{CGQL-L}}
& \gval{48}{7} & \gval{7}{5} & \gval{57}{2} & \gval{6}{3} & \gval{58}{1} & \gval{0}{0} & \gval{0}{0} & \gval{55}{2} & \gval{0}{1} & \gval{1}{1} & \gavg{23} \\
\midrule

\textsc{Post Processing} & {\texttt{DSRL}}
& \gval{53}{2} & \gval{1}{1} & \gval{53}{10} & \gval{1}{1} & \bmax{80}{0} & \bmax{100}{0} & \gval{61}{8} & \gval{72}{4} & \gval{34}{6} & \gval{9}{3} & \gavg{46} \\
& {\texttt{IFQL}}
& \gval{29}{8} & \gval{12}{3} & \bmax{93}{2} & \gval{30}{7} & \gval{36}{1} & \gval{64}{4} & \gval{42}{4} & \gval{9}{2} & \gval{24}{7} & \gval{6}{3} & \gavg{35} \\
\midrule

\textsc{Adjoint Matching} & {\texttt{QAM}}
& \gval{62}{9} & \gval{29}{4} & \gval{64}{7} & \gval{4}{3} & \gval{64}{4} & \gval{15}{3} & \gval{1}{1} & \gval{71}{2} & \gval{19}{6} & \bmax{18}{3} & \gavg{35} \\
& {\texttt{QAM-E}}
& \gval{86}{3} & \gval{6}{8} & \gval{60}{6} & \gval{4}{5} & \gval{63}{6} & \gval{89}{4} & \gval{54}{8} & \gval{71}{3} & \gval{11}{4} & \gval{9}{3} & \gavg{45} \\
\midrule

\rowcolor{ourlightblue} 
\textsc{\textbf{Ours}} & {\textbf{\texttt{TRQAM}}}
& \ourmax{89}{4} & \ourmax{41}{4} & \gval{84}{3} & \ourmax{36}{4} & \ourmax{79}{1} & \ourmax{100}{0} & \ourmax{99}{1} & \ourmax{81}{3} & \ourmax{50}{5} & \ourmax{19}{5} & \ouravg{68} \\
\bottomrule
\end{tabular}
}
\end{table}

\subsection{What mechanism drives its gains?}
\label{sec:exp_mechanism}
To isolate what drives the asymmetry above, we compare three variants differing only in how $\lambda$ is regulated. QAM keeps $\lambda$ constant in our framework. QAM + External KL adapts $\lambda$ through a conventional KL regularization loss. TRQAM adapts $\lambda$ inside the SOC sampling process. These variants isolate two design axes, namely \emph{adaptation} (constant KL vs.\ adaptive KL) and \emph{internalization} (external loss penalty vs.\ internalized SOC sampling process).

\paragraph{1. Adaptive KL outperforms constant KL.}
Both adaptive KL variants, TRQAM and QAM with external KL regularization, substantially outperform QAM on \texttt{cube-triple-task1} and \texttt{humanoidmaze-medium-task1} (\cref{fig:mechanism}). The same asymmetry appears on Robomimic, where QAM exhibits the diverging adjoint loss and collapsing task success of \cref{fig:robomimic_divergence} while both adaptive variants remain stable (\cref{app:robomimic}). These results are consistent with \cref{lem:exp_amp}'s exponential amplification of critic errors under fixed temperature.
\begin{figure}[t]
  \centering
  \includegraphics[width=0.9\linewidth]{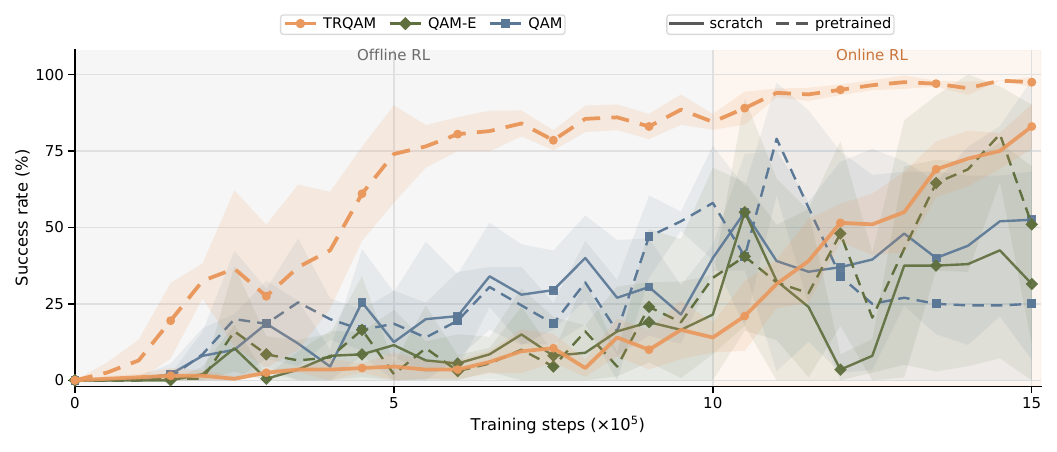}
  \vspace{-0.15in}
  \caption{\textbf{Pretraining alone is not sufficient.} On 
    \texttt{humanoidmaze-medium-task1}, each algorithm is run from a 
    pretrained flow policy (dashed) and from scratch (solid). TRQAM benefits substantially from the 
    pretrained prior, while QAM and QAM-E show little to no benefit from the 
    same pretrained initialization, with their pretrained and scratch curves 
    remaining close throughout training. Shaded regions denote standard deviation across seeds.}
  \label{fig:pretraining}
  \vspace{-0.15in}
\end{figure}
\begin{figure}[t]
  \centering\small
  \includegraphics[width=0.99\linewidth]{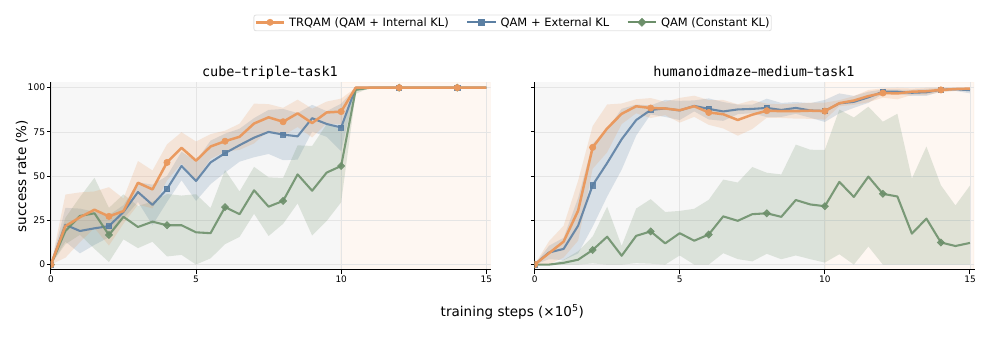}
  \vspace{-0.17in}
  \caption{\textbf{Adaptation is necessary.} Both adaptive variants 
    (TRQAM and QAM + External KL) outperform QAM (constant $\lambda$) on 
    \texttt{cube-triple-task1} and \texttt{humanoidmaze-medium-task1}, 
    consistent with \cref{lem:exp_amp}. A fixed temperature can amplify 
    critic errors into policy deviations, while adaptation can mitigate this. 
    Shaded regions denote standard deviation across seeds.}
  \label{fig:mechanism}
  \vspace{-0.15in}
\end{figure}
\begin{figure}[t]
  \centering
  \includegraphics[width=0.98\linewidth]{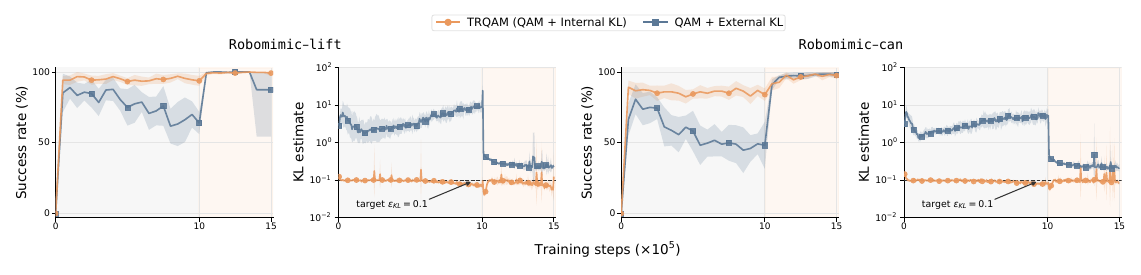}
  \vspace{-0.1in}
  \caption{\textbf{Internalization is necessary.} On \texttt{Robomimic-lift} \& \texttt{Robomimic-can} 
    with $\varepsilon_{\mathrm{KL}}=0.1$, TRQAM tightly tracks the target KL 
    bound throughout training, while QAM + External KL lets the realized KL 
    drift well above it, with corresponding success rate degradation. By 
    \cref{thm:path_kl}, only the internal parameterization (TRQAM) 
   makes $\lambda$ the exact weight of the path-space KL in the SOC objective, while the external 
    loss penalty (QAM + External KL) can be overridden by strong critic guidance. 
    Shaded regions denote standard deviation across seeds.}
  \label{fig:internalization}
  \vspace{-0.17in}
\end{figure}
\begin{figure}[t]
  \centering
  \includegraphics[width=0.98\linewidth]{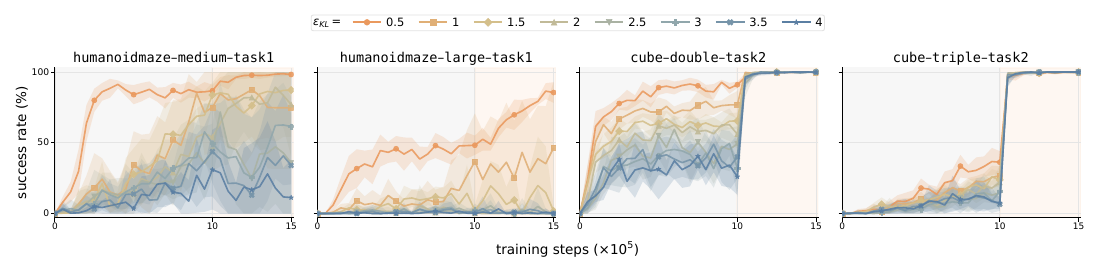}
  \vspace{-0.1in}
\caption{\textbf{Sensitivity Analysis.} Success-rate curves on four OGBench
tasks under varying KL budgets. Two patterns emerge. First, success rate 
changes smoothly with $\varepsilon_{\mathrm{KL}}$ on every task, making 
the budget a predictable knob. Second, tight budgets are best on all four tasks. Domains with much larger state spaces, such as \texttt{puzzle-4x4}, instead benefit from larger budgets (\cref{fig:ogbench_klb_sensitivity_grid}).}
  \label{fig:ogbench_klb_sensitivity_4tasks}
\end{figure}

\paragraph{2. Internal KL outperforms external KL regularization.}
Among the two adaptive variants, only TRQAM enforces the prescribed KL budget. On \texttt{Robomimic-lift} and \texttt{Robomimic-can} 
with $\varepsilon_{\mathrm{KL}}=0.1$, TRQAM tightly tracks the bound throughout training, while QAM + External KL lets the 
realized KL drift above $\varepsilon_{\mathrm{KL}}$ with corresponding success-rate degradation 
(\cref{fig:internalization}). By \cref{thm:path_kl}, 
only the internal parameterization makes $\lambda$ the exact weight of the path-space KL in the SOC objective, whereas an external penalty enters 
only as an additive loss term that strong critic guidance can override. The same pattern holds across all three Robomimic tasks 
and six KL budgets (\cref{fig:full_robomimic_lift_kl_compare_grid,fig:full_robomimic_can_kl_compare_grid,fig:full_robomimic_square_kl_compare_grid}).

\subsection{Sensitivity Analysis}
\label{sec:exp_sensitivity}
Since $\varepsilon_{\mathrm{KL}}$ directly sets how far TRQAM can move from the pretrained policy, it is the main hyperparameter of our method. To measure its effect, we sweep $\varepsilon_{\mathrm{KL}}$ from $0.5$ to 
$4$ in steps of $0.5$ on four representative OGBench tasks 
(\texttt{humanoidmaze-medium-task1}, 
\texttt{humanoidmaze-large-task1}, \texttt{cube-double-task2}, 
\texttt{cube-triple-task2}), with the same training setup as 
\cref{sec:exp_main}. 
\cref{fig:ogbench_klb_sensitivity_4tasks} shows two patterns. 
First, success rate changes smoothly with $\varepsilon_{\mathrm{KL}}$ 
on every task. Second, tight budgets are best across all four 
tasks. The similar pattern holds across most of the 10 OGBench 
domains in \cref{app:sensitivity}, with \texttt{puzzle-4x4} 
as the exception where larger budgets monotonically improve 
performance, in keeping with its notably larger state space. 
Because TRQAM tightly enforces the chosen budget 
(\cref{sec:exp_mechanism}), tuning $\varepsilon_{\mathrm{KL}}$ 
to task structure produces predictable, controlled behavior, 
so adapting TRQAM to a new environment comes down to choosing 
$\varepsilon_{\mathrm{KL}}$ based on task structure.
\section{Related Work}
\label{sec:related}
\paragraph{Offline-to-online RL.}
Pretraining a policy and value function on offline data and then 
fine-tuning online is a standard recipe for sample-efficient 
RL~\citep{ball2023efficientonlinereinforcementlearning,hester2017deepqlearningdemonstrations,hu2024imitationbootstrappedreinforcementlearning,kostrikov2021offline,kumar2020conservativeqlearningofflinereinforcement,lei2024unio4unifyingonlineoffline,luo2025serlsoftwaresuitesampleefficient,nair2020awac,nakamoto2024calqlcalibratedofflinerl,rajeswaran2018learningcomplexdexterousmanipulation,vecerik2018leveragingdemonstrationsdeepreinforcement}. 
A central challenge is the distribution shift at the transition, 
which destabilizes the value function and induces catastrophic 
forgetting~\citep{ball2023efficientonlinereinforcementlearning,nakamoto2024calqlcalibratedofflinerl,wolczyk2024fine}, 
motivating trust-region-style constraints during fine-tuning.

\paragraph{Fine-tuning flow and diffusion policies.}
Flow matching and diffusion policies~\citep{chi2024diffusionpolicy,lipman2022flow} 
parameterize multi-modal action distributions and are increasingly 
pretrained at scale~\citep{bjorck2025gr00t,black2026pi0visionlanguageactionflowmodel, intelligence2025pi06vlalearnsexperience, black2025pi05}. 
A growing body of work trains such policies with 
RL~\citep{chen2026pitextttrlonlinerlfinetuning,dong2025expo,hansen2023idql,liu2025flowgrpotrainingflowmatching,oberai2026reversalqlearning,fql_park2025,ren2024diffusionpolicypolicyoptimization,Wagenmaker2025DSRL,wang2023diffusionpoliciesexpressivepolicy,xiao2025selfimproving,xu2026rltokenbootstrappingonline,zhang2026reinflowfinetuningflowmatching,sacflow, zhou2026testtimegradientguidanceflow}, 
each navigating a tradeoff between policy expressivity, 
computational cost, and training stability.

\paragraph{KL trust regions in RL.}
KL regularization stabilizes policy updates as a hard constraint 
or soft 
penalty~\citep{bdolmaleki2018maximumposterioripolicyoptimisation,peng2019advantageweightedregressionsimplescalable,schulman2015trust,schulman2017proximal,wu2019behaviorregularizedofflinereinforcement}, 
sometimes with strength adapted via dual updates 
(e.g., SAC~\citep{haarnoja2018soft} for entropy). These approaches 
enforce the constraint via auxiliary losses external to the policy.
\section{Conclusion}
\label{sec:conclusion}
\vspace{-0.05in}
We introduce \emph{Trust Region Q Adjoint Matching} (TRQAM), a stable off-policy fine-tuning method for pretrained flow policies. TRQAM adapts a trust-region parameter $\lambda$ inside the SOC sampling process. Letting $\lambda$ scale the control term in the drift makes $\lambda$ the exact weight of the path-space KL in the SOC objective (\cref{thm:path_kl}), so dual descent on $\lambda$ enforces the target bound within the SOC formulation rather than through a loss-level penalty. Across 50 OGBench tasks, TRQAM improves the strongest baseline by 22 points, with the largest gains on long-horizon and combinatorial domains. On Robomimic, it remains stable where fixed-temperature Adjoint Matching collapses. PPO became a standard tool for on-policy RL largely because it made policy updates stable~\citep{schulman2017proximal}. We hope TRQAM plays a similar role for off-policy fine-tuning of pretrained flow policies, where a prescribed KL budget keeps each update close to a trusted prior. As pretrained flow policies grow into large vision-language-action models, we believe such stability will matter even more. We hope this work encourages future research to treat stability as a central goal when fine-tuning these models with RL.

\section*{Acknowledgments}
This work was supported by Institute for Information \& communications Technology Planning \& Evaluation (IITP) grant funded by the Korea government (MSIT) (RS-2022-II220953, Self-directed AI Agents with Problem-solving Capability; RS-2019-II190075, Artificial Intelligence Graduate School Program (KAIST)). We thank RLWRLD Inc. for providing compute resources to conduct experiments in this work.

YH thanks the anonymous reviewers for their constructive feedback, and Sihyun Yu, Taeyoung Kim, Seungku Kim, Daewon Choi, Woomin Song, Junbeom Kim, Beomjun Kim, Myungkyu Koo, Suhyeok Jang, Huiwon Jang, Sanghyeok Lee, John Won, and Jimin Lee for their helpful feedback and insightful discussions.

\bibliographystyle{plainnat}
\bibliography{reference}

@String(ICML = {International Conference on Machine Learning})

@String(ECCV = {European Conference on Computer Vision})

@String(ICLR = {International Conference on Learning Representations})

@String(NeurIPS = {Advances in Neural Information Processing Systems})

@String(JMLR = {Journal of Machine Learning Research})

@String(CORL = {Conference on Robot Learning})

@String(RSS = {Robotics: Science and Systems})

@String(ICRA = {{IEEE} International Conference on Robotics and Automation})

@String(AAAI = {{AAAI} Conference on Artificial Intelligence})

@inproceedings{domingoenrich2023stochastic,
    author = {Carles Domingo-Enrich and Jiequn Han and Brandon Amos and Joan Bruna and Ricky Tian Qi Chen},
    title = {Stochastic Optimal Control Matching},
    booktitle = NeurIPS,
    year = {2024}
}

@inproceedings{domingoenrich2025adjointmatchingfinetuningflow,
    author = {Carles Domingo-Enrich and Michal Drozdzal and Brian Karrer and Ricky Tian Qi Chen},
    title = {Adjoint Matching: Fine-tuning Flow and Diffusion Generative Models with Memoryless Stochastic Optimal Control},
    booktitle = ICLR,
    year = {2025}
}

@article{nusken2023solving,
  author = {N{\"u}sken, Nikolas and Richter, Lorenz},
  title = {Solving high-dimensional {H}amilton--{J}acobi--{B}ellman {PDEs} using neural networks: perspectives from the theory of controlled diffusions and measures on path space},
  journal = {Partial differential equations and applications},
  volume = {2},
  pages = {1--48},
  year = {2021},
  publisher = {Springer}
}

@inproceedings{qam,
    author = {Qiyang Li and Sergey Levine},
    title  = {{Q-learning} with Adjoint Matching},
    booktitle = ICLR,
    year = {2026}
}

@inproceedings{fql_park2025,
  title = {Flow {Q-Learning}},
  author = {Seohong Park and Qiyang Li and Sergey Levine},
  booktitle = ICML,
  year = {2025},
}

@inproceedings{sacflow,
    author = {Yixian Zhang and Shu'ang Yu and Tonghe Zhang and Mo Guang and Haojia Hui and Kaiwen Long and Yu Wang and Chao Yu and Wenbo Ding},
    title = {{SAC} Flow: Sample-Efficient Reinforcement Learning of Flow-Based Policies via Velocity-Reparameterized Sequential Modeling}, 
    booktitle = ICLR,
    year = {2026}
}

@article{intelligence2025pi06vlalearnsexperience,
    author = {{Physical Intelligence} and Ali Amin and Raichelle Aniceto and Ashwin Balakrishna and Kevin Black and Ken Conley and Grace Connors and James Darpinian and Karan Dhabalia and Jared DiCarlo and Danny Driess and Michael Equi and Adnan Esmail and Yunhao Fang and Chelsea Finn and Catherine Glossop and Thomas Godden and Ivan Goryachev and Lachy Groom and Hunter Hancock and Karol Hausman and Gashon Hussein and Brian Ichter and Szymon Jakubczak and Rowan Jen and Tim Jones and Ben Katz and Liyiming Ke and Chandra Kuchi and Marinda Lamb and Devin LeBlanc and Sergey Levine and Adrian Li-Bell and Yao Lu and Vishnu Mano and Mohith Mothukuri and Suraj Nair and Karl Pertsch and Allen Z. Ren and Charvi Sharma and Lucy Xiaoyang Shi and Laura Smith and Jost Tobias Springenberg and Kyle Stachowicz and Will Stoeckle and Alex Swerdlow and James Tanner and Marcel Torne and Quan Vuong and Anna Walling and Haohuan Wang and Blake Williams and Sukwon Yoo and Lili Yu and Ury Zhilinsky and Zhiyuan Zhou},
    title = {$\pi^{*}_{0.6}$: a {VLA} That Learns From Experience},
    journal = {arXiv preprint arXiv:2511.14759},
    year = {2025}
}

@inproceedings{Wagenmaker2025DSRL,
  author    = {Andrew Wagenmaker and Mitsuhiko Nakamoto and Yunchu Zhang and Seohong Park and Waleed Yagoub and Anusha Nagabandi and Abhishek Gupta and Sergey Levine},
  title     = {Steering Your Diffusion Policy with Latent Space Reinforcement Learning},
  booktitle = CORL,
  year      = {2025},
}

@inproceedings{dong2025expo,
    author  = {Perry Dong and Qiyang Li and Dorsa Sadigh and Chelsea Finn},
    title   = {{EXPO}: Stable Reinforcement Learning with Expressive Policies},
    booktitle = ICLR,
    year = {2026}
}

@inproceedings{chi2024diffusionpolicy,
	author = {Cheng Chi and Zhenjia Xu and Siyuan Feng and Eric Cousineau and Yilun Du and Benjamin Burchfiel and Russ Tedrake and Shuran Song},
	title = {Diffusion Policy: Visuomotor Policy Learning via Action Diffusion},
	booktitle = RSS,
	year = {2023},
}

@inproceedings{fujimoto2018addressingfunctionapproximationerror,
    author = {Scott Fujimoto and Herke van Hoof and David Meger},
    title = {Addressing Function Approximation Error in Actor-Critic Methods},
    booktitle = ICML,
    year = {2018}
}

@inproceedings{kumar2020conservativeqlearningofflinereinforcement,
    author = {Aviral Kumar and Aurick Zhou and George Tucker and Sergey Levine},
    title = {Conservative {Q-Learning} for Offline Reinforcement Learning},
    booktitle = NeurIPS,
    year = {2020}
}

@inproceedings{kostrikov2021offline,
    author = {Ilya Kostrikov and Ashvin Nair and Sergey Levine},
    title = {Offline Reinforcement Learning with Implicit {Q-Learning}},
    booktitle = ICLR,
    year = {2022}
}

@inproceedings{schulman2015trust,
    author = {John Schulman and Sergey Levine and Philipp Moritz and Michael I. Jordan and Pieter Abbeel},
    title = {Trust region policy optimization},
    booktitle = ICML,
    year = {2015}
}

@article{schulman2017proximal,
    title = {Proximal policy optimization algorithms},
    author = {Schulman, John and Wolski, Filip and Dhariwal, Prafulla and Radford, Alec and Klimov, Oleg},
    journal = {arXiv preprint arXiv:1707.06347},
    year = {2017}
}

@inproceedings{haarnoja2018soft,
    author = {Haarnoja, Tuomas and Zhou, Aurick and Abbeel, Pieter and Levine, Sergey},
    title = {Soft actor-critic: Off-policy maximum entropy deep reinforcement learning with a stochastic actor},
    booktitle = ICML,
    year = {2018}
}

@book{sutton2018reinforcement,
  title = {Reinforcement learning: An introduction},
  author = {Sutton, Richard S and Barto, Andrew G},
  year = {2018},
  publisher = {MIT press}
}

@inproceedings{kim2026deas,
    title = {{DEAS}: {DEtached} value learning with Action Sequence for Scalable Offline {RL}},
    author = {Changyeon Kim and Haeone Lee and Younggyo Seo and Kimin Lee and Yuke Zhu},
    booktitle = ICLR,
    year = {2026},
}

@inproceedings{peng2019advantageweightedregressionsimplescalable,
    author={Xue Bin Peng and Aviral Kumar and Grace Zhang and Sergey Levine},
    title={Advantage-Weighted Regression: Simple and Scalable Off-Policy Reinforcement Learning}, 
    booktitle = ICLR,
    year = {2021}
}

@inproceedings{wu2019behaviorregularizedofflinereinforcement,
    author={Yifan Wu and George Tucker and Ofir Nachum},
    title={Behavior Regularized Offline Reinforcement Learning}, 
    booktitle = ICLR,
    year = {2020}
}

@inproceedings{ren2024diffusionpolicypolicyoptimization,
    author={Allen Z. Ren and Justin Lidard and Lars L. Ankile and Anthony Simeonov and Pulkit Agrawal and Anirudha Majumdar and Benjamin Burchfiel and Hongkai Dai and Max Simchowitz},
    title={Diffusion Policy Policy Optimization}, 
    booktitle = ICLR,
    year = {2025}
}

@inproceedings{wang2023diffusionpoliciesexpressivepolicy,
    author={Zhendong Wang and Jonathan J Hunt and Mingyuan Zhou},
    title={Diffusion Policies as an Expressive Policy Class for Offline Reinforcement Learning},
    booktitle = ICLR,
    year = 2023
}

@inproceedings{zhang2026reinflowfinetuningflowmatching,
    author={Tonghe Zhang and Chao Yu and Sichang Su and Yu Wang},
    title={{ReinFlow}: Fine-tuning Flow Matching Policy with Online Reinforcement Learning}, 
    booktitle = NeurIPS,
    year = {2025}
}

@article{chen2026pitextttrlonlinerlfinetuning,
    author={Kang Chen and Zhihao Liu and Tonghe Zhang and Zhen Guo and Si Xu and Hao Lin and Hongzhi Zang and Xiang Li and Quanlu Zhang and Zhaofei Yu and Guoliang Fan and Tiejun Huang and Yu Wang and Chao Yu},
    title={$\pi_\texttt{RL}$: Online {RL} Fine-tuning for Flow-based Vision-Language-Action Models}, 
    journal = {arXiv preprint arXiv:2510.25889}, 
    year = {2026}
}

@inproceedings{nair2020awac,
    author={Nair, Ashvin and Gupta, Abhishek and Dalal, Murtaza and Levine, Sergey},
    title={{AWAC}: Accelerating online reinforcement learning with offline datasets},
    booktitle = ICLR,
    year = {2021}
}

@inproceedings{ball2023efficientonlinereinforcementlearning,
    author={Philip J. Ball and Laura Smith and Ilya Kostrikov and Sergey Levine},
    title={Efficient Online Reinforcement Learning with Offline Data}, 
    booktitle = ICML,
    year = {2023}
}

@inproceedings{nakamoto2024calqlcalibratedofflinerl,
    author={Mitsuhiko Nakamoto and Yuexiang Zhai and Anikait Singh and Max Sobol Mark and Yi Ma and Chelsea Finn and Aviral Kumar and Sergey Levine},
    title={{Cal-QL}: Calibrated Offline {RL} Pre-Training for Efficient Online Fine-Tuning}, 
    booktitle = NeurIPS,
    year = {2023}
}

@inproceedings{hu2024imitationbootstrappedreinforcementlearning,
    author={Hengyuan Hu and Suvir Mirchandani and Dorsa Sadigh},
    title={Imitation Bootstrapped Reinforcement Learning},
    booktitle = ICLR,
    year = {2024}
}

@inproceedings{luo2025serlsoftwaresuitesampleefficient,
    author={Jianlan Luo and Zheyuan Hu and Charles Xu and You Liang Tan and Jacob Berg and Archit Sharma and Stefan Schaal and Chelsea Finn and Abhishek Gupta and Sergey Levine},
    title={{SERL}: A Software Suite for Sample-Efficient Robotic Reinforcement Learning}, 
    booktitle = ICRA,
    year = {2024}
}

@inproceedings{rajeswaran2018learningcomplexdexterousmanipulation,
    author={Aravind Rajeswaran and Vikash Kumar and Abhishek Gupta and Giulia Vezzani and John Schulman and Emanuel Todorov and Sergey Levine},
    title={Learning Complex Dexterous Manipulation with Deep Reinforcement Learning and Demonstrations}, 
    booktitle = RSS,
    year = {2018}
}

@article{vecerik2018leveragingdemonstrationsdeepreinforcement,
    title={Leveraging Demonstrations for Deep Reinforcement Learning on Robotics Problems with Sparse Rewards}, 
    author={Mel Vecerik and Todd Hester and Jonathan Scholz and Fumin Wang and Olivier Pietquin and Bilal Piot and Nicolas Heess and Thomas Rothörl and Thomas Lampe and Martin Riedmiller},
    journal = {arXiv preprint arXiv:1707.08817}, 
    year={2018}
}

@inproceedings{hester2017deepqlearningdemonstrations,
    author={Todd Hester and Matej Vecerik and Olivier Pietquin and Marc Lanctot and Tom Schaul and Bilal Piot and Dan Horgan and John Quan and Andrew Sendonaris and Gabriel Dulac-Arnold and Ian Osband and John Agapiou and Joel Z. Leibo and Audrunas Gruslys},
    title={Deep {Q-learning} from Demonstrations}, 
    booktitle = AAAI,
    year = {2018}
}

@inproceedings{lei2024unio4unifyingonlineoffline,
    author={Kun Lei and Zhengmao He and Chenhao Lu and Kaizhe Hu and Yang Gao and Huazhe Xu},
    title={{Uni-O4}: Unifying Online and Offline Deep Reinforcement Learning with Multi-Step On-Policy Optimization}, 
    booktitle = ICLR,
    year = {2024}
}

@inproceedings{wolczyk2024fine,
  title={Fine-tuning Reinforcement Learning Models is Secretly a Forgetting Mitigation Problem},
  author={Wo{\l}czyk, Maciej and Cupia{\l}, Bart{\l}omiej and Ostaszewski, Mateusz and Bortkiewicz, Micha{\l} and Zaj{\k{a}}c, Micha{\l} and Pascanu, Razvan and Kuci{\'n}ski, {\L}ukasz and Mi{\l}o{\'s}, Piotr},
  booktitle=ICML,
  year={2024}
}

@article{zhou2026testtimegradientguidanceflow,
    title={Test-Time Gradient Guidance of Flow Policies in Reinforcement Learning}, 
    author={Zhiyuan Zhou and Andy Peng and Charles Xu and Qiyang Li and Tobias Springenberg and Kevin Frans and Sergey Levine},
    journal = {arXiv preprint arXiv:2606.11087}, 
    year={2026}
}

@article{oberai2026reversalqlearning,
    title={Reversal {Q-Learning}}, 
    author={Aditya Oberai and Seohong Park and Sergey Levine},
    journal = {arXiv preprint arXiv:2606.17551},
    year={2026}
}

@article{thilges2026scalablemaximumentropyreinforcement,
    author={Serge Thilges and Onur Celik and Denis Blessing and Emiliyan Gospodinov and Gerhard Neumann},
    title={Scalable Maximum Entropy Reinforcement Learning for Diffusion Policies via Adjoint Matching},
    journal = {arXiv preprint arXiv:2606.22630},
    year = {2026}
}

@article{lipman2024flowmatchingguidecode,
    title={Flow Matching Guide and Code}, 
    author={Yaron Lipman and Marton Havasi and Peter Holderrieth and Neta Shaul and Matt Le and Brian Karrer and Ricky Tian Qi Chen and David Lopez-Paz and Heli Ben-Hamu and Itai Gat},
    journal = {arXiv preprint arXiv:2412.06264},
    year={2024}
}

@inproceedings{lipman2022flow,
    author = {Lipman, Yaron and Chen, Ricky Tian Qi Chen and Ben-Hamu, Heli and Nickel, Maximilian and Le, Matt},
    title = {Flow matching for generative modeling},
    booktitle = ICLR,
    year = {2023}
}

@inproceedings{albergo2023stochastic,
    author = {Albergo, Michael S and Boffi, Nicholas M and Vanden-Eijnden, Eric},
    title = {Stochastic interpolants: A unifying framework for flows and diffusions},
    booktitle = JMLR,
    year = {2025}
}

@inproceedings{xiao2025selfimproving,
    author = {Wenli Xiao and Haotian Lin and Andy Peng and Haoru Xue and Tairan He and Yuqi Xie and Fengyuan Hu and Jimmy Wu and Zhengyi Luo and Jim Fan and Guanya Shi and Yuke Zhu},
    title = {Self-Improving Vision-Language-Action Models with Data Generation via Residual {RL}},
    booktitle = ICLR,
    year = {2026}
}

@inproceedings{dhariwal2021diffusion,
    author = {Dhariwal, Prafulla and Nichol, Alexander},
    title = {Diffusion models beat {GANs} on image synthesis},
    booktitle = NeurIPS,
    year = {2021}
}

@article{hansen2023idql,
    author = {Hansen-Estruch, Philippe and Kostrikov, Ilya and Janner, Michael and Kuba, Jakub Grudzien and Levine, Sergey},
    title = {{IDQL}: Implicit {Q-Learning} as an Actor-Critic Method with Diffusion Policies},
    journal = {arXiv preprint arXiv:2304.10573},
    year = {2023}
}

@inproceedings{myers2025offline,
    author = {Vivek Myers and Bill Zheng and Benjamin Eysenbach and Sergey Levine},
    title = {Offline Goal-conditioned Reinforcement Learning with Quasimetric Representations},
    booktitle = NeurIPS,
    year = {2025}
}

@incollection{parsian2002estimation,
  title = {Estimation under {LINEX} loss function},
  author = {Parsian, Ahmad and Kirmani, SNUA},
  booktitle = {Handbook of applied econometrics and statistical inference},
  pages = {75--98},
  year = {2002},
  publisher = {CRC Press}
}

@inproceedings{ma2024sit,
    author = {Ma, Nanye and Goldstein, Mark and Albergo, Michael S and Boffi, Nicholas M and Vanden-Eijnden, Eric and Xie, Saining},
    title = {{SiT}: Exploring flow and diffusion-based generative models with scalable interpolant transformers},
    booktitle = ECCV,
    year = {2024}
}

@inproceedings{bdolmaleki2018maximumposterioripolicyoptimisation,
    author={Abbas Abdolmaleki and Jost Tobias Springenberg and Yuval Tassa and Remi Munos and Nicolas Heess and Martin Riedmiller},
    title={Maximum a Posteriori Policy Optimisation}, 
    booktitle = ICLR,
    year = {2018}
}

@inproceedings{liu2022flowstraightfastlearning,
    author={Xingchao Liu and Chengyue Gong and Qiang Liu},
    title={Flow Straight and Fast: Learning to Generate and Transfer Data with Rectified Flow},
    booktitle = ICLR,
    year = {2023}
}

@inproceedings{doo2026qflowstableexpressivereinforcement,
  title={{Q-Flow}: Stable and Expressive Reinforcement Learning with Flow-Based Policy},
  author={JaeHyeok Doo and Byeongguk Jeon and Seonghyeon Ye and Kimin Lee and Minjoon Seo},
  booktitle=ICML,
  year={2026}
}

@article{bjorck2025gr00t,
    title = {{GR00T} {N1}: An Open Foundation Model for Generalist Humanoid Robots},
    author = {{NVIDIA} and Johan Bjorck and Fernando Casta{\~n}eda and Nikita Cherniadev and Xingye Da and Runyu Ding and Jim Fan and Yu Fang and Dieter Fox and Fengyuan Hu and Spencer Huang and Joel Jang and Zhenyu Jiang and Jan Kautz and Kaushil Kundalia and Lawrence Lao and Zhiqi Li and Zongyu Lin and Kevin Lin and Guilin Liu and Edith Llontop and Loic Magne and Ajay Mandlekar and Avnish Narayan and Soroush Nasiriany and Scott Reed and You Liang Tan and Guanzhi Wang and Zu Wang and Jing Wang and Qi Wang and Jiannan Xiang and Yuqi Xie and Yinzhen Xu and Zhenjia Xu and Seonghyeon Ye and Zhiding Yu and Ao Zhang and Hao Zhang and Yizhou Zhao and Ruijie Zheng and Yuke Zhu},
    journal = {arXiv preprint arXiv:2503.14734},
    year = "2025"
}

@inproceedings{black2025pi05,
    author={{Physical Intelligence} and Kevin Black and Noah Brown and James Darpinian and Karan Dhabalia and Danny Driess and Adnan Esmail and Michael Equi and Chelsea Finn and Niccolo Fusai and Manuel Y. Galliker and Dibya Ghosh and Lachy Groom and Karol Hausman and Brian Ichter and Szymon Jakubczak and Tim Jones and Liyiming Ke and Devin LeBlanc and Sergey Levine and Adrian Li-Bell and Mohith Mothukuri and Suraj Nair and Karl Pertsch and Allen Z. Ren and Lucy Xiaoyang Shi and Laura Smith and Jost Tobias Springenberg and Kyle Stachowicz and James Tanner and Quan Vuong and Homer Walke and Anna Walling and Haohuan Wang and Lili Yu and Ury Zhilinsky},
    title = {$\pi_{0.5}$: a Vision-Language-Action Model with Open-World Generalization},
    booktitle = CORL,
    year = {2025}
}

@inproceedings{li2025reinforcementlearningactionchunking,
    author = {Qiyang Li and Zhiyuan Zhou and Sergey Levine},
    title = {Reinforcement Learning with Action Chunking},
    booktitle = NeurIPS,
    year = {2025}
}

@inproceedings{black2026pi0visionlanguageactionflowmodel,
    author={Kevin Black and Noah Brown and Danny Driess and Adnan Esmail and Michael Equi and Chelsea Finn and Niccolo Fusai and Lachy Groom and Karol Hausman and Brian Ichter and Szymon Jakubczak and Tim Jones and Liyiming Ke and Sergey Levine and Adrian Li-Bell and Mohith Mothukuri and Suraj Nair and Karl Pertsch and Lucy Xiaoyang Shi and James Tanner and Quan Vuong and Anna Walling and Haohuan Wang and Ury Zhilinsky},
    title={$\pi_0$: A Vision-Language-Action Flow Model for General Robot Control}, 
    booktitle = RSS,
    year = {2025}
}

@article{xu2026rltokenbootstrappingonline,
    title={{RL} Token: Bootstrapping Online {RL} with Vision-Language-Action Models}, 
    author={Charles Xu and Jost Tobias Springenberg and Michael Equi and Ali Amin and Adnan Esmail and Sergey Levine and Liyiming Ke},
    journal = {arXiv preprint arXiv:2604.23073}, 
    year={2026}
}

@inproceedings{liu2025flowgrpotrainingflowmatching,
    author={Jie Liu and Gongye Liu and Jiajun Liang and Yangguang Li and Jiaheng Liu and Xintao Wang and Pengfei Wan and Di Zhang and Wanli Ouyang},
    title={{Flow-GRPO}: Training Flow Matching Models via Online {RL}}, 
    booktitle = NeurIPS,
    year = {2025}
}

@article{shukor2025smolvla,
      title={{SmolVLA}: A Vision-Language-Action Model for Affordable and Efficient Robotics}, 
      author={Mustafa Shukor and Dana Aubakirova and Francesco Capuano and Pepijn Kooijmans and Steven Palma and Adil Zouitine and Michel Aractingi and Caroline Pascal and Martino Russi and Andres Marafioti and Simon Alibert and Matthieu Cord and Thomas Wolf and Remi Cadene},
      year={2025},
      journal = {arXiv preprint arXiv:2506.01844}
}

@inproceedings{
zheng2026xvla,
title={X-{VLA}: Soft-Prompted Transformer as Scalable Cross-Embodiment Vision-Language-Action Model},
author={Jinliang Zheng and Jianxiong Li and Zhihao Wang and Dongxiu Liu and Xirui Kang and Yuchun Feng and Yinan Zheng and Jiayin Zou and Yilun Chen and Jia Zeng and Tai Wang and Ya-Qin Zhang and Jingjing Liu and Xianyuan Zhan},
booktitle=ICLR,
year={2026}
}

@inproceedings{
  reuss2025flower,
  title={{FLOWER}: Democratizing Generalist Robot Policies with Efficient Vision-Language-Flow Models},
  author={Moritz Reuss and Hongyi Zhou and Marcel R{\"u}hle and {\"O}mer Erdin{\c{c}} Ya{\u{g}}murlu and Fabian Otto and Rudolf Lioutikov},
  booktitle=CORL,
  year={2025}
}

@article{zhai2025walloss,
      title={Igniting {VLMs} toward the Embodied Space}, 
      author={Andy Zhai and Brae Liu and Bruno Fang and Chalse Cai and Ellie Ma and Ethan Yin and Hao Wang and Hugo Zhou and James Wang and Lights Shi and Lucy Liang and Make Wang and Qian Wang and Roy Gan and Ryan Yu and Shalfun Li and Starrick Liu and Sylas Chen and Vincent Chen and Zach Xu},
      year={2025},
      journal = {arXiv preprint arXiv:2509.11766}
}

@article{hung2025nora15,
      title={{NORA-1.5}: A Vision-Language-Action Model Trained using World Model- and Action-based Preference Rewards}, 
      author={Chia-Yu Hung and Navonil Majumder and Haoyuan Deng and Liu Renhang and Yankang Ang and Amir Zadeh and Chuan Li and Dorien Herremans and Ziwei Wang and Soujanya Poria},
      year={2025},
      journal = {arXiv preprint arXiv:2511.14659}
}

@article{jiang2025galaxea,
      title={{Galaxea} Open-World Dataset and {G0} Dual-System {VLA} Model}, 
      author={Tao Jiang and Tianyuan Yuan and Yicheng Liu and Chenhao Lu and Jianning Cui and Xiao Liu and Shuiqi Cheng and Jiyang Gao and Huazhe Xu and Hang Zhao},
      year={2025},
      journal = {arXiv preprint arXiv:2509.00576}
}

@article{kim2026rldx1technicalreport,
    author={Dongyoung Kim and Huiwon Jang and Myungkyu Koo and Suhyeok Jang and Taeyoung Kim and Beomjun Kim and Byungjun Yoon and Changsung Jang and Daewon Choi and Dongsu Han and others},
    title={{RLDX-1} Technical Report}, 
    journal = {arXiv preprint arXiv:2605.03269},
    year = {2026}
}

@inproceedings{ogbench_park2025,
    title = {{OGBench}: Benchmarking Offline Goal-Conditioned {RL}},
    author = {Park, Seohong and Frans, Kevin and Eysenbach, Benjamin and Levine, Sergey},
    booktitle = ICLR,
    year = {2025},
}

@inproceedings{robomimic2021,
    author = {Mandlekar, Ajay and Xu, Danfei and Wong, Josiah and Nasiriany, Soroush and Wang, Chen and Kulkarni, Rohun and Fei-Fei, Li and Savarese, Silvio and Zhu, Yuke and Mart\'in-Mart\'in, Roberto},
    title = {What Matters in Learning from Offline Human Demonstrations for Robot Manipulation},
    booktitle = CORL,
    year = {2021}
}

@book{polyanskiy2025information,
    title={Information Theory: From Coding to Learning},
    author={Polyanskiy, Yury and Wu, Yihong},
    publisher={Cambridge University Press},
    year={2025}
}

@book{oksendal2003stochastic,
  title={Stochastic differential equations},
  author={{\O}ksendal, Bernt},
  year={2003},
  publisher={Springer}
}

\clearpage
\appendix
\crefalias{section}{appendix}
\crefalias{subsection}{appendix}
\newpage

\section*{Contents}
\startcontents[appendices]
\printcontents[appendices]{}{1}{\setcounter{tocdepth}{2}}

\newpage

\section{Algorithm}
\label{app:algorithm}
\begin{algorithm}[H]
\caption{Trust Region Q Adjoint Matching (TRQAM) for fine-tuning Flow Matching policies.
\textcolor{GoogleBlue}{Blue} marks TRQAM additions over QAM~\citep{qam}.}
\label{alg_full:trqam}
\begin{algorithmic}
\State \textbf{Input:} replay buffer $\mathcal{D}$;
$v^{\mathrm{base}}$: pretrained (behavior) velocity field;
$v^{\mathrm{ft}}_\theta$: fine-tuned velocity field;
$Q_\phi$: critic function;
step size $h$;
\textcolor{GoogleBlue}{KL budget $\varepsilon_{\mathrm{KL}}$;
dual stepsize $\eta_\lambda$;
EMA coefficient $\rho_D$};
fine-tuning iterations $N$.
\State \textbf{Initialize:}
$v^{\mathrm{ft}}_\theta \gets v^{\mathrm{base}}$ with parameters $\theta$;
\textcolor{GoogleBlue}{$\lambda_0 > 0$;
$\overline{D}_0 \gets 0$}.
\State \textbf{Memoryless SDE:}
\textcolor{GoogleBlue}{$\sqrt{\lambda_n}\,\sigma_n(\tau) = \nu(\tau)$
for all $n$, where $\nu(\tau) := \sqrt{2(1-\tau)/\tau}$}
\vspace{2pt}
\For{$n \in \{0, \dots, N-1\}$}
    \State Sample a batch $\mathcal{B} = \{(s_i, a_i, r_i, s'_i)\}$ from $\mathcal{D}$
    \vspace{4pt}
    \State \textbf{Critic update:}
    Optimize $\phi$ w.r.t.
    \begin{equation*}
        \mathcal{L}(\phi)
        = \frac{1}{|\mathcal{B}|} \sum_{(s,a,r,s') \in \mathcal{B}}
          \left[
            Q_\phi(s, a) - r - \gamma\, Q_{\bar{\phi}}(s',\, a' \sim \pi_\theta(\cdot \mid s'))
          \right]^2
        \quad \triangleright~\text{TD backup}
    \end{equation*}
    \vspace{2pt}
    \State \textbf{Policy update:}
    For each state $s \in \mathcal{B}$, sample a trajectory
    $\boldsymbol{X} = (X_\tau)_{\tau \in \{0, h, \dots, 1\}}$ via the memoryless Euler scheme:
    \begin{equation*}
        X_{\tau + h}
        = X_\tau
        + h\!\left(2 v^{\mathrm{ft}}_\theta(s, X_\tau, \tau)
                  - \frac{1}{\tau} X_\tau\right)
        + \sqrt{h}\,
          \underbrace{\textcolor{GoogleBlue}{\sqrt{\lambda_n}}\,\sigma_n(\tau)}_{=\,\nu(\tau)}\,
          \varepsilon_\tau,
        \quad
        \varepsilon_\tau, X_0 \sim \mathcal{N}(0, I).
    \end{equation*}
    \State Compute the critic's action gradient:
    $\tilde{a}_1 \gets -\nabla_{X_1} Q_{\bar{\phi}}(s, X_1)$.
    \State Solve the lean adjoint ODE backwards:
    \begin{equation*}
        \tilde a_{\tau - h}
        = \tilde a_\tau
        + h\,\tilde a_\tau^{\!\top}
          \nabla_{X_\tau}\!\left(2 v^{\mathrm{base}}(s, X_\tau, \tau)
                                - \frac{1}{\tau} X_\tau\right).
    \end{equation*}
    \State Stop gradient: $X_\tau \gets \mathtt{stopgrad}(X_\tau)$,
    $\tilde a_\tau \gets \mathtt{stopgrad}(\tilde a_\tau)$.
    \State Optimize $\theta$ w.r.t. the Adjoint Matching objective:
    \begin{equation*}
        \mathcal{L}_{\mathrm{Adj\text{-}Match}}(\theta)
        = \frac{1}{|\mathcal{B}|}\!\sum_{s \in \mathcal{B}}\sum_{\tau}
          \left\|
              \tfrac{2}{\textcolor{GoogleBlue}{\sigma_n(\tau)}}
              \big(v^{\mathrm{ft}}_\theta(s, X_\tau, \tau)
                   - v^{\mathrm{base}}(s, X_\tau, \tau)\big)
              + \textcolor{GoogleBlue}{\sigma_n(\tau)}\,\tilde a_\tau
          \right\|^2
    \end{equation*}
    \vspace{2pt}
    \State \textcolor{GoogleBlue}{\textbf{Trust region update:}}
    \State \textcolor{GoogleBlue}{Estimate path-space KL surrogate:}
    \begin{equation*}
        \textcolor{GoogleBlue}{
            \widehat{D}_n
            = \frac{1}{|\mathcal{B}|}\!\sum_{s \in \mathcal{B}}\sum_{\tau}
              \tfrac{2h}{\nu(\tau)^2}
              \big\|v^{\mathrm{ft}}_\theta(s, X_\tau, \tau) - v^{\mathrm{base}}(s, X_\tau, \tau)\big\|^2
        }
    \end{equation*}
    \State \textcolor{GoogleBlue}{EMA smoothing:
    $\overline{D}_n \gets (1-\rho_D)\,\overline{D}_{n-1} + \rho_D\,\widehat{D}_n$.}
    \State \textcolor{GoogleBlue}{Dual descent:
    $\lambda_{n+1} \gets \max\!\big\{0,\ \lambda_n + \eta_\lambda(\overline{D}_n - \varepsilon_{\mathrm{KL}})\big\}$.}
\EndFor
\State \textbf{Output:} fine-tuned velocity field $v^{\mathrm{ft}}_\theta$, critic $Q_\phi$.
\end{algorithmic}
\end{algorithm}

\section{Baselines}
\label{appendix:baselines}
We compare against six baselines spanning distinct fine-tuning paradigms 
for flow-matching policies. Throughout, $(s, a, r, s')$ are sampled 
uniformly from the replay buffer $D$ without re-weighting. The buffer $D$ contains 
the offline dataset during offline training and is augmented with online 
rollouts during online fine-tuning.

\paragraph{FQL~\citep{fql_park2025}.}
FQL distills a multi-step flow policy into a one-step policy to avoid 
backpropagation through time. Conditioned on state $s$, the behavior-cloning 
rollout is given by the ODE
\begin{align*}
    dX_\tau= v_\theta(X_\tau, \tau; s)d\tau, \quad X_0 \sim \mathcal{N}(0, I), \quad \tau \in [0, 1]\text{,}
\end{align*}
and we define $\mathrm{ODE}(v_\theta, s, X_0) := X_1$ as its terminal value. 
The one-step policy $\pi_\omega(s, X_0)$ is trained jointly with $v_\theta$ 
to maximize the critic while staying close to this rollout by minimizing
\begin{align*}
    \mathcal L_{\mathrm{onestep}}(\omega) = \mathbb{E}_{X_0 \sim \mathcal{N}}\!\Big[
    \underbrace{-\, Q(s, \pi_\omega(s, X_0))}_{\text{RL maximization}}
    \;+\;
    \underbrace{\alpha\, \|\pi_\omega(s, X_0) - \mathrm{ODE}(v_\theta, s, X_0)\|_2^2}_{\text{BC distillation}}
    \Big]\text{,}
\end{align*}
where $\alpha$ controls how closely $\pi_\omega$ stays to the BC rollout. 
The environment policy is $\pi_\omega$.

\paragraph{CGQL-Linex.}
CGQL-Linex is a baseline introduced in~\citet{qam}, combining a BC velocity 
field with classifier guidance~\citep{dhariwal2021diffusion} from a 
Q-function. It trains an auxiliary intermediate critic 
$Q_\psi(s, X_\tau, \tau)$ on intermediate noisy actions $X_\tau$, used to 
construct a guidance velocity that steers sampling toward the 
entropy-regularized optimal policy 
$\pi^\star(\cdot \mid s) \propto e^{\beta Q_\phi(s, \cdot)}$. The guidance 
velocity is
\begin{align*}
    \hat{v}_\psi(X_\tau, \tau; s) := \frac{(1 - \tau)\beta\,\nabla_{X_\tau} Q_\psi(s, X_\tau, \tau) + X_\tau}{\tau}\text{,}
\end{align*}
where $X_\tau = (1-\tau) X_0 + \tau X_1$ with $X_0 \sim \mathcal{N}(0, I)$. 
The final sampling velocity is $v = v^{\mathrm{base}} + w\, \hat{v}_\psi$, 
where $w$ modulates guidance strength. The intermediate critic $Q_\psi$ is 
trained by minimizing a Linex regression loss~\citep{parsian2002estimation, myers2025offline}
\begin{align*}
    \mathcal L_{\mathrm{Linex}}(\psi) = \mathbb{E}_{\tau, X_0}\!\left[
    \exp\bigl(\beta(Q_\phi(s, X_1) - Q_\psi(s, X_\tau, \tau))\bigr) + \beta Q_\psi(s, X_\tau, \tau)\right]\text{,}
\end{align*}
while the standard critic $Q_\phi$ is trained via standard TD with target 
actions sampled from the full velocity $v$, by minimizing
\begin{align*}
    \mathcal L_{\mathrm{TD}}(\phi) = \mathbb{E}\!\left[
    \bigl(Q_\phi(s, a) - r - \gamma\, Q_{\bar\phi}(s', \mathrm{ODE}(v, s', X_0))\bigr)^2
    \right]\text{.}
\end{align*}
We follow~\citet{qam} in using a Huber-style stabilization of the Linex 
loss to prevent exponential blow-up.

\paragraph{DSRL~\citep{Wagenmaker2025DSRL}.}
DSRL performs RL directly in the noise space of a frozen flow policy. 
It trains a one-step Gaussian noise-space policy $\pi_\omega(X_0 \mid s)$ 
via SAC, maximizing a noise-space critic $Q_\psi(s, X_0)$ that regresses 
to the original action-space critic by minimizing
\begin{align*}
    L(\psi) = \mathbb{E}_{X_0 \sim \mathcal{N}}\!\left[
    \bigl(Q_\psi(s, X_0) - Q_\phi(s, \mathrm{ODE}(v_{\bar\theta}, s, X_0))\bigr)^2
    \right]\text{.}
\end{align*}
At inference, actions are obtained by sampling $X_0 \sim \pi_\omega(\cdot \mid s)$ 
and pushing it through the flow policy, which gives $a = \mathrm{ODE}(v_\theta, s, X_0)$. 
Following~\citet{qam}, we modify the original DSRL to also fine-tune the 
BC velocity online (using a target network $v_{\bar\theta}$ for stability), 
which yields stronger offline-to-online performance.

\paragraph{IFQL.}
IFQL is the flow counterpart of implicit diffusion 
$Q$-learning~\citep{hansen2023idql}, considered as a baseline 
in~\citet{fql_park2025}. Value learning uses IQL-style expectile 
regression~\citep{kostrikov2021offline}. A value network $V_\xi$ is 
trained to fit an upper expectile of the critic, and the critic is 
bootstrapped through $V_\xi$, with the losses
\begin{align*}
    \mathcal L_V(\xi) &= \mathbb{E}\!\left[L_2^\kappa\bigl(Q_{\bar\phi}(s, a) - V_\xi(s)\bigr)\right]\text{,} \\
    \mathcal L_Q(\phi) &= \mathbb{E}\!\left[\bigl(r + \gamma V_\xi(s') - Q_\phi(s, a)\bigr)^2\right]\text{,}
\end{align*}
where $L_2^\kappa(u) = |\kappa - \mathbf{1}(u < 0)|\, u^2$ is the 
expectile loss with parameter $\kappa \in (0.5, 1)$. Policy extraction 
uses rejection sampling, where $N$ candidate actions are drawn from a BC flow 
policy by sampling $X_0^{(i)} \sim \mathcal{N}(0, I)$ and computing 
$a^{(i)} = \mathrm{ODE}(v_\theta, s, X_0^{(i)})$ for $i = 1, \dots, N$, 
and the action with the highest $Q_\phi$ value is selected.

\paragraph{QAM and QAM-E~\citep{qam}.}
QAM is the closest prior work to TRQAM. It solves a memoryless SOC problem 
analogous to \cref{eq:prelim_soc} but \emph{without} the parameter $\lambda$ in the controlled SDE (i.e., the special case $\lambda=1$ of 
our framework), with an inverse temperature $\beta$ in the terminal cost $g = -\beta\, Q_\phi(s, X_1)$. The fine-tuned velocity $v_\theta$ is 
trained against the BC velocity $v^{\mathrm{base}}$ through the Adjoint Matching loss
\begin{align*}
    L_{\mathrm{AM}}(\theta) = \mathbb{E}\!\left[\int_0^1 \left\|\frac{2\bigl(v_\theta(X_\tau, \tau; s) - v^{\mathrm{base}}(X_\tau, \tau; s)\bigr)}{\sigma(\tau)} + \sigma(\tau)\, \tilde{a}_\tau\right\|_2^2 d\tau\right]\text{,}
\end{align*}
where $\tilde{a}_\tau$ is the lean adjoint state with terminal condition 
$\tilde{a}_1 = -\beta\, \nabla_{X_1} Q_\phi(s, X_1)$, and $X_\tau$ follows 
the memoryless SDE in \cref{eq:prelim_controlled_sde} with $\lambda=1$. 
Element-wise gradient clipping is applied for numerical stability. QAM-E 
augments QAM with an additional residual edit policy 
$\pi_\omega(\Delta a \mid s, \tilde{a})$ that perturbs the QAM-generated 
action $\tilde{a}$ by at most $\sigma_a$ in $L_\infty$ distance (enforced 
by a tanh-squashed Gaussian), trained via entropy-regularized SAC with 
automatic entropy tuning. The  difference from 
TRQAM is that they keep $\lambda = 1$ with a fixed $\beta$, whereas TRQAM adapts $\lambda$ by projected dual 
descent.

\section{Experimental details}
\label{app:experiments}

\subsection{Domains and tasks}
\label{app:domains}
We evaluate on two benchmarks, OGBench~\citep{ogbench_park2025} and
Robomimic~\citep{robomimic2021}. 

OGBench is originally designed for offline goal-conditioned RL, and we use its reward-based single-task variants. From OGBench, we use $10$ domains spanning long-horizon navigation, multi-object manipulation, and combinatorial planning. Their abbreviations are \texttt{scene}, \texttt{puzzle-3x3} (\texttt{p33}), \texttt{puzzle-4x4} (\texttt{p44}), \texttt{cube-double} (\texttt{c2}), \texttt{cube-triple} (\texttt{c3}), \texttt{cube-quadruple} (\texttt{c4}), \texttt{humanoidmaze-medium} (\texttt{hm}), \texttt{humanoidmaze-large} (\texttt{hl}), \texttt{antmaze-large} (\texttt{al}), and \texttt{antmaze-giant} (\texttt{ag}).

Robomimic is a demonstration-based manipulation benchmark used as a stability stress-test. We use the \texttt{lift}, \texttt{can}, and \texttt{square} tasks.

The dataset size, episode length, and action dimension for each domain are reported in
\cref{tab:metadata}. For each method and task, we run $8$ random seeds.
Unless otherwise stated, tables report mean success rate $\pm$ standard
deviation across seeds, and plots show the mean with shaded regions
denoting standard deviation.

\paragraph{Dataset structure.}
Unlike teleoperation-style benchmarks where each demonstration directly 
solves the target task, OGBench's offline data is task-agnostic. The 
\texttt{navigate} datasets capture free maze exploration and \texttt{play} 
datasets capture unstructured object manipulation. The BC-pretrained 
policy thus serves as a behavioral prior rather than a task-specific 
solution.

\paragraph{Dataset sources.}
We use the official OGBench datasets~\citep{ogbench_park2025} for all
domains except where noted below. For \texttt{cube-triple-10M-*} and \texttt{puzzle-4x4-10M-*}, we
use a 10M-size subset of the official 100M release. The 100M release
is split into 100 files of 1M transitions each, and we take the
first 10 files sorted by name, following \citep{kim2026deas}. For
\texttt{antmaze-giant-10M-*}, OGBench does not release a
dataset at the size used in our experiments, so we generate it
ourselves using the official OGBench data-generation pipeline with
default settings, and release it at \url{https://huggingface.co/datasets/yonghoon96/antmaze-giant-navigate-10m-v0}. For Robomimic, we use the Multi-Human (MH) datasets, each consisting of
$300$ trajectories collected by six operators of varying proficiency
(two ``worse'', two ``okay'', and two ``better''), yielding diverse
mixed-quality demonstrations~\citep{robomimic2021}. The behavior-cloning checkpoints used to initialize all methods on OGBench are available at \url{https://huggingface.co/yonghoon96/trqam-bc-checkpoints}.

\subsection{Compute resources}
\label{app:compute}
Experiments ran on NVIDIA RTX 2080 Ti and A100-80GB GPUs, with a small share on RTX 3090, RTX A6000, and RTX 4090 cards. \cref{tab:compute} reports the median wallclock time per run. The OGBench and Robomimic experiments reported in this paper used approximately $38{,}000$ GPU-hours in total, excluding preliminary runs and the KL-budget sweeps used to select $\varepsilon_{\mathrm{KL}}$.

\begin{table}[ht]
\centering
\small
\caption{\textbf{Median wallclock time per run.} A100-80GB runs share one GPU among four runs.}
\label{tab:compute}
\begin{tabular}{@{}lcc@{}}
\toprule
 & RTX 2080 Ti & A100-80GB \\
\midrule
OGBench, 1M data & $\approx 7$\,h & $\approx 2$--$4$\,h \\
OGBench, 10M data & $\approx 19$\,h & $\approx 9$\,h \\
OGBench, 100M data & -- & $\approx 7$\,h \\
Robomimic & $\approx 10$--$12$\,h & $\approx 11$--$12$\,h \\
\bottomrule
\end{tabular}
\end{table}

\subsection{Hyperparameters}
\label{app:hyperparams}
Most methods share the common hyperparameters in \cref{tab:rl-hyperparams}, and
method-specific hyperparameters are listed in \cref{tab:hparam}. 

\paragraph{OGBench tuning.} For all baselines (FQL, DSRL, IFQL, CGQL-L, 
QAM, QAM-E), we adopt the per-domain values reported as optimal in 
QAM~\citep{qam}, which were selected through an extensive per-domain 
hyperparameter sweep covering 6--20 configurations per method across all 
ten domains, totaling approximately 32{,}000 GPU-hours. For TRQAM, we tune
$\varepsilon_{\mathrm{KL}}$ on two tuning tasks per domain with $2$ seeds 
per configuration (different from the evaluation seeds), sweeping over 
$\{0.5,\allowbreak 1.0,\allowbreak 1.5,\allowbreak 2.0,\allowbreak 2.5,\allowbreak 3.0,\allowbreak 3.5,\allowbreak 4.0\}$. Following~\citet{qam}, 
we use task 1 (the default task) and task 4 for locomotion domains, and 
task 2 (the default task) and task 4 for manipulation domains, as this 
combination better covers the characteristics of each domain than the 
default task alone. We select the configuration based on the combined 
offline-to-online learning curve and stability across the two tuning 
tasks. Selected values are reported in \cref{tab:hparam}. For \texttt{antmaze-giant-10M-*}, we use a time-varying schedule with 
$\varepsilon_{\mathrm{KL}} = 0.5$ during offline training and 
$3.0$ during online fine-tuning. The offline value is selected from the per-domain sweep, where it 
yielded the most stable offline training. However, retaining $0.5$ 
online led to slow improvement in success rate, reflecting the 
larger exploration demands of giant-scale navigation. We therefore relax the bound at 
the online transition, selecting $3.0$ from the same sweep range as 
a sweet spot between exploration and stability 
(see~\cref{app:time_varying} for a direct comparison of static vs.\ 
relaxed schedules). This schedule exploits TRQAM's capability to 
track time-varying $\varepsilon_{\mathrm{KL}}$ without retraining. All main-paper results are averaged over $8$ evaluation seeds.

\paragraph{Robomimic tuning.} Since Robomimic was not evaluated in 
the QAM paper~\citep{qam}, no per-domain hyperparameters are available, so we 
conducted an independent sweep for each method (DSRL, QAM, QAM-E, 
TRQAM) on this benchmark. As Robomimic serves as our stability 
testbed, we run the full sweep range with $8$ seeds per configuration 
to obtain reliable estimates of variance across hyperparameters. 
For QAM-E, which has two hyperparameters, we sweep the full Cartesian product ($4 \times 2 = 8$ configurations). 
Robomimic consists of human teleoperation data, for which staying 
closer to behavior cloning is generally beneficial. We therefore 
adjust the QAM~\citep{qam} sweep ranges toward configurations that 
more strongly anchor the policy to the behavior policy. The sweep 
ranges are listed in~\cref{tab:hparam-range}. 

\begin{table}[ht]
\centering
\small
\caption{\textbf{Domain metadata.}}
\label{tab:metadata}
\setlength{\tabcolsep}{6pt}
\renewcommand{\arraystretch}{0.95}
\begin{adjustbox}{max width=\linewidth}
\begin{tabular}{@{}lccc@{}}
\toprule
\textbf{Domain} & \textbf{Data} & \textbf{Horizon} & \textbf{Act. dim.} \\
\midrule
\texttt{cube-double-*}          & 1M   & 500  & 5  \\
\texttt{cube-triple-10M-*}      & 10M  & 1000 & 5  \\
\texttt{cube-quadruple-100M-*}  & 100M & 1000 & 5  \\
\texttt{antmaze-large-*}        & 1M   & 1000 & 8  \\
\texttt{antmaze-giant-10M-*}    & 10M  & 1000 & 8  \\
\texttt{humanoidmaze-medium-*}  & 1M   & 2000 & 21 \\
\texttt{humanoidmaze-large-*}   & 1M   & 2000 & 21 \\
\texttt{scene-*}                & 1M   & 750  & 5  \\
\texttt{puzzle-3x3-*}           & 1M   & 500  & 5  \\
\texttt{puzzle-4x4-10M-*}       & 10M  & 500  & 5  \\
\texttt{lift}                   & 31127  & 500  & 7  \\ 
\texttt{can}                    & 62756  & 500  & 7  \\ 
\texttt{square}                 & 80731  & 500  & 7  \\ 
\bottomrule
\end{tabular}
\end{adjustbox}
\end{table}
\begin{table}[!htbp]
\centering
\small
\caption{\textbf{Common hyperparameters.}}
\label{tab:rl-hyperparams}
\setlength{\tabcolsep}{5pt}
\renewcommand{\arraystretch}{0.95}
\begin{adjustbox}{max width=\linewidth}
\begin{tabular}{@{}lc@{}}
\toprule
\textbf{Parameter} & \textbf{Value} \\
\midrule
Batch size                              & 256 \\
Discount factor ($\gamma$)              & 0.995 (default), 0.999 (\texttt{humanoidmaze}) \\
Optimizer                               & Adam \\
Learning rate                           & $3 \times 10^{-4}$ \\
Target network update rate              & $5 \times 10^{-3}$ \\
Critic ensemble size ($K$)              & 10 \\
Critic pessimism coefficient ($\rho$)   & 0.5 (default), 0 (\texttt{humanoidmaze}) \\
UTD ratio                               & 1 \\
Number of flow steps ($T$)              & 10 \\
BC training steps                       & $0.3 \times 10^{6}$ \\
Offline RL steps                        & $10^{6}$ \\
Online RL steps                & $0.5 \times 10^{6}$ \\
Network width                           & 512 (default), 1024 (10M/100M data) \\
Network depth                           & 4 hidden layers \\
Gradient max-norm clipping              & False (default), 1 (\texttt{QAM}, \texttt{QAM-E}, \texttt{TRQAM}) \\
Actor layer norm                        & False (default), True (10M/100M data) \\
Critic layer norm                       & True \\
Initial trust-region parameter ($\lambda_0$) & 1 \\
Dual step size ($\eta_\lambda$) & 0.01 \\
Range of $\lambda$ & $[0.01, 100]$ \\
KL EMA coefficient ($\rho_D$) & 0.1 \\
\bottomrule
\end{tabular}
\end{adjustbox}
\vspace{-10pt}
\end{table}
\begin{table}[!htbp]
\centering
\small
\caption{\textbf{Domain-specific hyperparameters.}}
\label{tab:hparam}
\setlength{\tabcolsep}{3pt}
\renewcommand{\arraystretch}{1.03}
\begin{adjustbox}{max width=\linewidth}
\begin{tabular}{@{}lccccccc@{}}
\toprule
Domain
  & \texttt{FQL}
  & \texttt{DSRL}
  & \texttt{IFQL}
  & \texttt{CGQL-L}
  & \texttt{QAM}
  & \texttt{QAM-E}
  & \texttt{TRQAM} \\
  & $\alpha$
  & $\sigma_z$
  & $\kappa$
  & $(\vartheta,\varrho,\tau)$
  & $\beta$
  & $(\beta,\sigma_a)$
  & $\varepsilon_{\mathrm{KL}}$ \\
\midrule
\texttt{scene-*}                & 300 & 0.4 & 0.9  & $(10, 0.1, 0.1)$   & 1  & $(1, 0)$    & 0.5 \\
\texttt{puzzle-3x3-*}           & 300 & 1.0 & 0.95 & $(10, 0.001, 0.1)$ & 3  & $(1, 0.1)$  & 2.0 \\
\texttt{puzzle-4x4-10M-*}       & 1   & 1.0 & 0.9  & $(10, 0.001, 1)$   & 30 & $(0.1, 0.9)$& 4.0 \\
\texttt{cube-double-*}          & 300 & 1.0 & 0.9  & $(10, 0.001, 0.01)$& 1  & $(1, 0)$    & 0.5 \\
\texttt{cube-triple-10M-*}      & 30  & 1.4 & 0.95 & $(10, 0.001, 0.1)$ & 3  & $(3, 0.1)$  & 0.5 \\
\texttt{cube-quadruple-100M-*}  & 100 & 1.4 & 0.95 & $(10, 0.1, 0.01)$  & 1  & $(3, 0.1)$  & 1.0 \\
\texttt{antmaze-large-*}        & 3   & 0.8 & 0.9  & $(10, 0.001, 0.1)$ & 10 & $(1, 0.1)$  & 1.0 \\
\texttt{antmaze-giant-10M-*}    & 3   & 1.2 & 0.8  & $(10, 0.001, 0.1)$ & 3  & $(10, 0.1)$ & $0.5\!\to\!3.0$ \\
\texttt{humanoidmaze-medium-*}  & 30  & 0.6 & 0.7  & $(10, 0.1, 0.1)$   & 3  & $(3, 0.1)$  & 0.5 \\
\texttt{humanoidmaze-large-*}   & 30  & 0.8 & 0.8  & $(10, 0.1, 0.1)$   & 3  & $(3, 0.1)$  & 0.5 \\
\texttt{lift}                   & -   & 0.8 & -    & -   & 1  & $(1, 0.03)$  & 0.1 \\
\texttt{can}                    & -  & 0.6 & -     & -   & 0.3  & $(1, 0.03)$  & 0.1 \\
\texttt{square}                 & -  & 0.8 & -     & -   & 0.1  & $(1, 0.03)$  & 0.5 \\
\bottomrule
\end{tabular}
\end{adjustbox}
\end{table}
\begin{table}[!htbp]
    \centering
    \caption{\textbf{Robomimic hyperparameter sweep ranges.}}
    \label{tab:hparam-range}
    \begin{adjustbox}{max width=\linewidth}
    \begin{tabular}{@{}ccc@{}}
    \toprule
    \textbf{Method} & \textbf{Hyperparameter(s)} & \textbf{Sweep Range} \\
    \midrule
        \texttt{DSRL}    & $\sigma_z$                  & $\{0.03, 0.1, 0.2, 0.4, 0.6, 0.8\}$ \\
        \texttt{QAM}     & $\beta$                      & $\{0.01, 0.03, 0.1, 0.3, 1, 3\}$ \\
        \texttt{QAM-E}   & $(\beta, \sigma_a)$          & $(\{0.03, 0.1, 0.3, 1\},\; \{0.03, 0.1\})$ \\
        \texttt{TRQAM}   & $\varepsilon_{\mathrm{KL}}$ & $\{0.01, 0.03, 0.1, 0.5, 1.0, 1.5\}$ \\
    \bottomrule
    \end{tabular}
    \end{adjustbox}
\end{table}
\newpage
\section{Proofs}
\label{app:proofs}
This section provides full proofs for the four theoretical results stated in
\cref{sec:method} of the main text, namely
\cref{thm:path_kl} (path-space KL identity),
\cref{prop:dpi} (Terminal KL upper-bounded by path-space KL), \cref{prop:tilt_mono} (tilted-policy KL is non-decreasing in $\beta$), and
\cref{lem:exp_amp} (exponential amplification of critic errors).

We work throughout on a filtered probability space $(\Omega, \mathcal{F}, (\mathcal{F}_\tau)_{\tau \in [0,1]}, \mathbb{P})$ 
equipped with a natural filtration and an initial state distribution $X_0 \sim p_0$. 
We adopt standard regularity assumptions for controlled diffusions, as in \citet{nusken2023solving}. Specifically, the coefficients $b$ and $\sigma$ are sufficiently smooth, $b$ has at most linear growth, 
$\sigma\sigma^\top$ is uniformly positive definite, and the admissible control set $\mathcal{U}$ 
consists of progressively measurable controls with at most linear growth. Crucially, to ensure 
strong duality in our KL-budgeted improvement problem, we assume that the set of 
achievable path measures $\{ \mathbb{P}^u \mid u \in \mathcal{U} \}$ is convex. 
Moreover, we assume that $u \in \mathcal{U}$ satisfies Novikov's condition,
\begin{align*}
    \mathbb{E}_{\mathbf X \sim\mathbb{P}^{\mathrm{base}}}\!\left[
        \exp\!\left( \frac{1}{2\lambda} \int_0^1 \|u(X_\tau, \tau)\|^2 \, d\tau \right)
    \right] < \infty\text{,}
\end{align*}
which justifies the application of Girsanov's theorem in the proof of
\cref{thm:path_kl}.

\subsection{Proof of \cref{lem:exp_amp} (exponential amplification of critic errors)}
\label{app:exp_amp}
\setcounter{lem}{0}
\begin{lem}[Exponential amplification of critic errors]
Fix a state $s \in \mathcal{S}$. Let $Q: \mathcal{A} \to \mathbb{R}$ be the true critic and $\widetilde{Q}: \mathcal{A} \to \mathbb{R}$ a learned critic with
$\|Q - \widetilde{Q}\|_\infty \le \varepsilon$. Assume
$\pi_{\mathrm{base}}(\cdot \mid s) > 0$ a.e.\ and define the policies obtained by exponentially tilting $\pi_{\mathrm{base}}$ with each critic,
\[
\pi_Q(a \mid s) = \frac{\pi_{\mathrm{base}}(a \mid s)\,e^{\beta Q(a)}}{Z_Q},
\qquad
\pi_{\widetilde{Q}}(a \mid s) = \frac{\pi_{\mathrm{base}}(a \mid s)\,e^{\beta \widetilde{Q}(a)}}{Z_{\widetilde{Q}}}\text{,}
\]
with normalizers $Z_Q, Z_{\widetilde{Q}} \in (0, \infty)$. Then
\[
\TV(\pi_Q, \pi_{\widetilde{Q}}) \le \frac{1}{2}\bigl(e^{2\beta\varepsilon} - 1\bigr)\text{,}
\qquad
\KL(\pi_Q \,\|\, \pi_{\widetilde{Q}}) \le 2\beta\varepsilon\text{.}
\]
\end{lem}
\begin{proof}
For notational simplicity, we suppress $s$ throughout the proof. From $\|Q - \widetilde{Q}\|_\infty \le \varepsilon$,
multiplying $-\varepsilon \le Q(a) - \widetilde{Q}(a) \le \varepsilon$ by $\beta$ and
exponentiating yields, for all $a$,
\begin{equation}
e^{-\beta\varepsilon} \le e^{\beta(Q(a) - \widetilde{Q}(a))} \le e^{\beta\varepsilon}.
\label{eq:app_pointwise_exp}
\end{equation}
Multiplying \cref{eq:app_pointwise_exp} by
$\pi_{\mathrm{base}}(a)\,e^{\beta\widetilde{Q}(a)} \ge 0$ and integrating gives
$e^{-\beta\varepsilon} Z_{\widetilde{Q}} \le Z_Q \le e^{\beta\varepsilon} Z_{\widetilde{Q}}$.
Combining this with \cref{eq:app_pointwise_exp} in the likelihood ratio
\[
\frac{\pi_Q(a)}{\pi_{\widetilde{Q}}(a)}
= e^{\beta(Q(a) - \widetilde{Q}(a))} \cdot \frac{Z_{\widetilde{Q}}}{Z_Q}
\]
shows that both factors lie in $[e^{-\beta\varepsilon}, e^{\beta\varepsilon}]$, so
\begin{equation}
e^{-2\beta\varepsilon}
\;\le\; \frac{\pi_Q(a)}{\pi_{\widetilde{Q}}(a)} \;\le\;
e^{2\beta\varepsilon} \quad \text{a.e.}
\label{eq:app_ratio_bound}
\end{equation}

The TV bound follows from \cref{eq:app_ratio_bound}, which gives
$|\pi_Q(a) - \pi_{\widetilde{Q}}(a)| \le (e^{2\beta\varepsilon} - 1)\,\pi_{\widetilde{Q}}(a)$
a.e., and integrating with $\int \pi_{\widetilde{Q}} = 1$ gives
$\TV(\pi_Q, \pi_{\widetilde{Q}}) = \tfrac{1}{2}\|\pi_Q - \pi_{\widetilde{Q}}\|_1
\le \tfrac{1}{2}(e^{2\beta\varepsilon} - 1)$. The KL bound also follows
from \cref{eq:app_ratio_bound}, which gives $\log(\pi_Q / \pi_{\widetilde{Q}}) \le 2\beta\varepsilon$
a.e., so
\[
\KL(\pi_Q \,\|\, \pi_{\widetilde{Q}})
= \int \pi_Q(a) \log \frac{\pi_Q(a)}{\pi_{\widetilde{Q}}(a)}\,da
\le 2\beta\varepsilon \int \pi_Q(a)\,da = 2\beta\varepsilon\text{.}
\]
\end{proof}
\subsection{Proof of \cref{thm:path_kl} (path-space KL identity)}
\label{app:girsanov}
\setcounter{thm}{0}
\begin{thm}[Path-space KL identity with explicit $\lambda$ dependence]
Let $\mathbb{P}^{\mathrm{base}}$ and $\mathbb{P}^u$ denote the path
measures on $C([0,1]; \mathbb{R}^d)$ induced by the base and controlled SDEs, respectively,
\begin{align*}
    dX_\tau^{\mathrm{base}}
    &= b(X_\tau^{\mathrm{base}},\tau)\,d\tau
       + \sqrt{\lambda}\,\sigma(\tau)\,dB_\tau^{\mathrm{base}},\\
    dX_\tau^u
    &= \bigl(b(X_\tau^u,\tau) + \sigma(\tau)\,u(X_\tau^u,\tau)\bigr)\,d\tau
       + \sqrt{\lambda}\,\sigma(\tau)\,dB_\tau^u\text{,} 
\end{align*}
with common initial distribution $X_0 \sim p_0$. Then
\begin{align*}
    D_{\mathrm{KL}}\bigl(
        \mathbb{P}^u(\mathbf{X} \mid X_0)
        \,\|\,
        \mathbb{P}^{\mathrm{base}}(\mathbf{X} \mid X_0)
    \bigr)
    =
    \mathbb{E}_{\mathbf{X} \sim \mathbb{P}^u}\!\left[
        \frac{1}{2\lambda}\int_0^1 \|u(X_\tau^u, \tau)\|^2\,d\tau
    \right]\text{.}
\end{align*}
\end{thm}
\begin{proof}
Let $(\mathcal{F}_\tau)_{\tau \in [0,1]}$ denote the natural filtration
on the canonical path space $C([0,1]; \mathbb{R}^d)$. By Girsanov's
theorem~\citep[Theorem 8.6.6]{oksendal2003stochastic}, the
Radon--Nikodym derivative of $\mathbb{P}^u$ with respect to
$\mathbb{P}^{\mathrm{base}}$ restricted to $\mathcal{F}_\tau$ is
\begin{equation}
Z_\tau := \left.\frac{d\mathbb{P}^u}{d\mathbb{P}^{\mathrm{base}}}\right|_{\mathcal{F}_\tau}
= \exp\!\left(
\frac{1}{\sqrt{\lambda}}\int_0^\tau u(X_s^u, s)^\top \, dB_s^{\mathrm{base}}
- \frac{1}{2\lambda}\int_0^\tau \|u(X_s^u, s)\|^2 \, ds
\right)\text{,}
\label{eq:app_girsanov_density}
\end{equation}
where $B^{\mathrm{base}}$ is a standard Brownian motion under
$\mathbb{P}^{\mathrm{base}}$. By the definition of KL divergence on the
full path space (terminal $\sigma$-algebra $\mathcal{F}_1$),
\[
    D_{\mathrm{KL}}\!\left(\mathbb{P}^u \,\|\, \mathbb{P}^{\mathrm{base}}\right)
    = \mathbb{E}_{\mathbf{X} \sim \mathbb{P}^u}[\log Z_1]\text{.}
\]

Next, define
\[
    B_\tau^u := B_\tau^{\mathrm{base}} - \frac{1}{\sqrt{\lambda}}\int_0^\tau u(X_s^u, s) \, ds\text{.}
\]
By Girsanov's theorem, $B^u$ is a standard Brownian motion under
$\mathbb{P}^u$. In differential form,
$dB_s^{\mathrm{base}} = dB_s^u + \tfrac{1}{\sqrt{\lambda}}\,u(X_s^u, s)\,ds$.
Substituting this into \cref{eq:app_girsanov_density} at $\tau = 1$ gives
\begin{align*}
\log Z_1
&= \frac{1}{\sqrt{\lambda}}\int_0^1 u(X_s^u, s)^\top
   \!\left(dB_s^u + \frac{1}{\sqrt{\lambda}}\,u(X_s^u, s)\,ds\right)
   - \frac{1}{2\lambda}\int_0^1 \|u(X_s^u, s)\|^2 \, ds \\
&= \frac{1}{\sqrt{\lambda}}\int_0^1 u(X_s^u, s)^\top \, dB_s^u
   + \frac{1}{2\lambda}\int_0^1 \|u(X_s^u, s)\|^2 \, ds\text{.}
\end{align*}
Taking expectation under $\mathbb{P}^u$ and using that the It\^o
integral $\int_0^1 u(X_s^u, s)^\top \, dB_s^u$ is a zero-mean
martingale yields
\[
    D_{\mathrm{KL}}\!\left(\mathbb{P}^u \,\|\, \mathbb{P}^{\mathrm{base}}\right)
    = \mathbb{E}_{\mathbf{X} \sim \mathbb{P}^u}\!\left[
        \frac{1}{2\lambda}\int_0^1 \|u(X_s^u, s)\|^2 \, ds
    \right]\text{,}
\]
which is the claimed identity.
\end{proof}

\paragraph{Comparison to prior SOC parameterizations.}
The SOC parameterization used in prior
work~\citep{domingoenrich2025adjointmatchingfinetuningflow} takes
the form
\[
dX_\tau^u
    = \bigl(b(X_\tau^u,\tau) + \sigma(\tau)\,u(X_\tau^u,\tau)\bigr)\,d\tau
       + \sigma(\tau)\,dB_\tau^u\text{,}
\]
where $\sigma(\tau)$ is the full diffusion and no $\lambda$ appears. Girsanov's theorem then yields
\[
D_{\mathrm{KL}}(\mathbb{P}^u \| \mathbb{P}^{\mathrm{base}})
= \mathbb{E}_{\mathbf{X} \sim \mathbb{P}^u}\!\left[
        \frac{1}{2}\int_0^1 \|u(X_s^u, s)\|^2 \, ds
    \right]\text{.}
\]
In this form, the path-space KL equals the SOC quadratic cost, and no parameter inside the SDE modulates the strength of this regularization. Adapting the trust-region strength would then require an external coefficient on the SOC objective. Our parameterization instead lets $\lambda$ scale the control term in the drift, which makes $\lambda$ the exact weight of the path-space KL in the SOC objective. This is what allows $\lambda$ to serve as an adaptive dual variable for KL-budgeted improvement (\cref{sec:method}).

\subsection{Proof of \cref{prop:dpi} (terminal KL upper-bounded by path-space KL)}
\label{app:dpi}
\setcounter{prop}{0}
\begin{prop}[Terminal KL upper-bounded by path-space KL]
Let $\mathbb{P}^{\mathrm{base}}$ and $\mathbb{P}^u$ denote the path
measures on $C([0,1]; \mathbb{R}^d)$ induced by the base and controlled SDEs, respectively,
\begin{align*}
    dX_\tau^{\mathrm{base}}
    &= b(X_\tau^{\mathrm{base}},\tau)\,d\tau
       + \sqrt{\lambda}\,\sigma(\tau)\,dB_\tau^{\mathrm{base}}\text{,}\\
    dX_\tau^u
    &= \bigl(b(X_\tau^u,\tau) + \sigma(\tau)\,u(X_\tau^u,\tau)\bigr)\,d\tau
       + \sqrt{\lambda}\,\sigma(\tau)\,dB_\tau^u\text{,} 
\end{align*}
with common initial distribution $X_0 \sim p_0$. Let $\pi_\theta(\cdot \mid s)$ and $\pi_{\mathrm{base}}(\cdot \mid s)$ denote the corresponding terminal action distributions at $\tau = 1$, respectively. Then
\begin{align*}
    D_{\mathrm{KL}}\bigl(\pi_\theta(\cdot \mid s) \,\|\, \pi_{\mathrm{base}}(\cdot \mid s)\bigr)
    \;\leq\;
    D_{\mathrm{KL}}\bigl(\mathbb{P}^u(\mathbf{X} \mid X_0) \,\|\, \mathbb{P}^{\mathrm{base}}(\mathbf{X} \mid X_0)\bigr)\text{,}
\end{align*}
\end{prop}
\begin{proof}
Let $\Pi : C([0,1];\mathbb R^d) \to \mathbb R^d$ denote the deterministic
terminal projection $\Pi(\mathbf X) = X_1$. By construction, the terminal
action distributions are pushforwards of the path measures, namely
$\pi_\theta(\cdot \mid s) = \Pi_\#\,\mathbb P^u$ and
$\pi_{\mathrm{base}}(\cdot \mid s) = \Pi_\#\,\mathbb P^{\mathrm{base}}$.
Applying the data-processing inequality for KL divergence under 
maps~\citep[Corollary.~2.18]{polyanskiy2025information} to $\Pi$ yields
\[
D_{\mathrm{KL}}\!\left(\Pi_\#\,\mathbb P^u\,\|\,\Pi_\#\,\mathbb P^{\mathrm{base}}\right)
\;\le\;
D_{\mathrm{KL}}\!\left(\mathbb P^u\,\|\,\mathbb P^{\mathrm{base}}\right)\text{,}
\]
which proves the claim.     
\end{proof}

\subsection{Proof of \cref{prop:tilt_mono} (tilted-policy KL is non-decreasing in \texorpdfstring{$\beta$}{beta})}
\label{app:tilt_mono}
\setcounter{prop}{1}
\begin{prop}[Tilted-policy KL is non-decreasing in $\beta$]
Fix a state $s \in \cS$ and a bounded critic $Q: \cA \to \R$. For $\beta > 0$, define the exponentially tilted policy and its log-partition function
\begin{align*}
    \pi_\beta(a \mid s) = \frac{\pi_{\mathrm{base}}(a \mid s)\, e^{\beta Q(a)}}{Z(\beta)}\text{,}
    \qquad
    \psi(\beta) = \log Z(\beta)\text{,}
\end{align*}
where $Z(\beta) = \E_{a \sim \pi_{\mathrm{base}}}[e^{\beta Q(a)}]$. Then the following identities hold
\begin{align*}
    D_{\mathrm{KL}}(\pi_\beta \,\|\, \pi_{\mathrm{base}}) = \beta\psi'(\beta) - \psi(\beta)\text{,}
    \qquad
    \frac{d}{d\beta} D_{\mathrm{KL}}(\pi_\beta \,\|\, \pi_{\mathrm{base}}) = \beta \Var_{\pi_\beta}[Q] \geq 0\text{.}
\end{align*}
In particular, $D_{\mathrm{KL}}(\pi_\beta \,\|\, \pi_{\mathrm{base}})$ is non-decreasing in $\beta$.
\end{prop}
\begin{proof}
Since $Q$ is bounded, $Z(\beta)$ is finite and positive for every $\beta$, and differentiation under the expectation is justified.
Differentiating $\psi$ gives the first two moments of $Q$ under $\pi_\beta$,
\begin{align*}
    \psi'(\beta) = \frac{\E_{\pi_{\mathrm{base}}}[Q\, e^{\beta Q}]}{Z(\beta)} = \E_{\pi_\beta}[Q]\text{,}
    \qquad
    \psi''(\beta) = \frac{\E_{\pi_{\mathrm{base}}}[Q^2 e^{\beta Q}]}{Z(\beta)} - \psi'(\beta)^2 = \Var_{\pi_\beta}[Q]\text{.}
\end{align*}
The log-density ratio is $\log \bigl(\pi_\beta(a \mid s) / \pi_{\mathrm{base}}(a \mid s)\bigr) = \beta Q(a) - \psi(\beta)$. Taking its expectation under $\pi_\beta$ yields
\begin{align*}
    D_{\mathrm{KL}}(\pi_\beta \,\|\, \pi_{\mathrm{base}})
    = \beta\, \E_{\pi_\beta}[Q] - \psi(\beta)
    = \beta\psi'(\beta) - \psi(\beta)\text{.}
\end{align*}
Differentiating in $\beta$ gives
\begin{align*}
    \frac{d}{d\beta} D_{\mathrm{KL}}(\pi_\beta \,\|\, \pi_{\mathrm{base}})
    = \psi'(\beta) + \beta\psi''(\beta) - \psi'(\beta)
    = \beta \Var_{\pi_\beta}[Q] \geq 0\text{.}
\end{align*}
Hence the divergence is non-decreasing in $\beta$.
\end{proof}

\section{Path-space KL surrogate under OT memoryless discretization}
\label{app:pathkl}

This section derives the closed-form path-space KL estimator
$\widehat{D}_n$ used in \cref{alg_full:trqam}. Recall from
\cref{thm:path_kl} that the path-space KL between $\mathbb{P}^u$
and $\mathbb{P}^{\mathrm{base}}$ admits a closed-form expression in
terms of the control $u$, and from \cref{prop:dpi} that this
path-space KL upper-bounds the terminal-policy KL we ultimately wish to
constrain. We thus seek a tractable estimator of the path-space KL under
the discretized OT memoryless sampler.

\paragraph{Gaussian KL with shared covariance.}
For two Gaussians $\mathcal{N}(\mu_1, \Sigma)$ and $\mathcal{N}(\mu_0, \Sigma)$
with shared covariance $\Sigma \succ 0$, specializing the multivariate
Gaussian KL formula~\citep[Equation. 2.8]{polyanskiy2025information} makes
the log-determinant and trace terms vanish. In other words,
\begin{align}
\KL\!\left(\mathcal{N}(\mu_1, \Sigma) \,\|\, \mathcal{N}(\mu_0, \Sigma)\right)
=
\tfrac{1}{2}\,(\mu_1 - \mu_0)^\top \Sigma^{-1} (\mu_1 - \mu_0)\text{.}
\label{eq:app_gaussian_step_kl}
\end{align}

\paragraph{OT memoryless Euler step KL.}
Under the OT memoryless Euler scheme with step size $h$ and schedule
$\nu(\tau) = \sqrt{2(1-\tau)/\tau}$, the fine-tuned and base transitions
at time $\tau_k$ given $X_{\tau_k} = x$ are both Gaussian with covariance
$\Sigma_k = h\,\nu(\tau_k)^2\,I$ and means
\[
\mu_\theta(x, \tau_k) = x + h\!\left(2 v^{\mathrm{ft}}_\theta(x, \tau_k) - \tfrac{1}{\tau_k}\,x\right)\text{,}
\qquad
\mu_{\mathrm{base}}(x, \tau_k) = x + h\!\left(2 v^{\mathrm{base}}(x, \tau_k) - \tfrac{1}{\tau_k}\,x\right)\text{.}
\]
The means differ only through the velocity fields, since
$\mu_\theta(x, \tau_k) - \mu_{\mathrm{base}}(x, \tau_k) = 2h\,(v^{\mathrm{ft}}_\theta(x, \tau_k) - v^{\mathrm{base}}(x, \tau_k))$.
Substituting into \cref{eq:app_gaussian_step_kl} with
$\Sigma_k^{-1} = \tfrac{1}{h\,\nu(\tau_k)^2}\,I$ gives the per-step KL
\begin{align}
\KL\!\left(p_\theta(\cdot \mid x) \,\|\, p_{\mathrm{base}}(\cdot \mid x)\right)
=
\frac{2h}{\nu(\tau_k)^2}\,\|v^{\mathrm{ft}}_\theta(x, \tau_k) - v^{\mathrm{base}}(x, \tau_k)\|^2\text{.}
\label{eq:app_ot_step_kl}
\end{align}

\paragraph{Chain rule and Monte Carlo estimator.}
Let $\mathbb{P}^u$ and $\mathbb{P}^{\mathrm{base}}$ now denote the
discrete-time path measures of the Markov chains
$\boldsymbol{X} = (X_\tau)_{\tau \in \{0, h, \dots,\allowbreak 1\}}$ with shared
initial distribution $p_0$ and the Gaussian transition kernels above.
Under mild regularity conditions, the Markov KL chain rule gives
\begin{align*}
\KL(\mathbb{P}^u \,\|\, \mathbb{P}^{\mathrm{base}})
= \sum_{\tau \in \{0, h, \dots, 1-h\}}
\mathbb{E}_{X_\tau \sim \mathbb{P}^u}\!\left[
\KL\!\left(p_\theta(\cdot \mid X_\tau) \,\|\, p_{\mathrm{base}}(\cdot \mid X_\tau)\right)
\right]\text{.}
\end{align*}
Substituting \cref{eq:app_ot_step_kl} and estimating the expectation
by a Monte Carlo average over the trajectories in batch
$\mathcal{B}$ yields
\begin{align}
\widehat{D}
=
\frac{1}{|\mathcal{B}|}\sum_{\boldsymbol{X} \in \mathcal{B}} \sum_{\tau \in \{0, h, \dots, 1-h\}}
\frac{2h}{\nu(\tau)^2}\,
\|v^{\mathrm{ft}}_\theta(X_\tau, \tau) - v^{\mathrm{base}}(X_\tau, \tau)\|^2\text{,}
\label{eq:app_pathkl_mc}
\end{align}
which is the estimator used in \cref{alg_full:trqam} (optionally
smoothed by EMA to reduce variance).

\section{Primal--dual derivation of KL-budgeted improvement}
\label{app:dual}

This section derives the projected dual update on $\lambda$ used in
\cref{alg_full:trqam}.

\paragraph{Primal problem and Slater's condition.}
We consider the path-space KL-budgeted improvement problem
\begin{align*}
\max_{u \in \mathcal{U}}\quad & \mathbb{E}_{X \sim \mathbb{P}^u}[Q^\pi(X_1)]
\\
\text{s.t.}\quad & \KL(\mathbb{P}^u \,\|\, \mathbb{P}^{\mathrm{base}}) \le \varepsilon_{\mathrm{KL}}\text{.}
\end{align*}
Viewed in the path measure space, the objective is linear in $\mathbb{P}^u$ and
$\KL(\cdot \,\|\, \mathbb{P}^{\mathrm{base}})$ is convex. The trivial
control $u \equiv 0$ recovers $\mathbb{P}^u = \mathbb{P}^{\mathrm{base}}$
and is strictly feasible whenever $\varepsilon_{\mathrm{KL}} > 0$, so
Slater's condition holds and strong duality applies.

\paragraph{Lagrangian and dual function.}
Let $\lambda \ge 0$ be the dual variable for the KL constraint. The
Lagrangian is
\[
\mathcal{L}(u, \lambda)
= \mathbb{E}_{X \sim \mathbb{P}^u}[Q^\pi(X_1)]
+ \lambda \bigl(\varepsilon_{\mathrm{KL}} - \KL(\mathbb{P}^u \,\|\, \mathbb{P}^{\mathrm{base}}) \bigr)\text{,}
\]
and the dual function is $\Gamma(\lambda) := \sup_u \mathcal{L}(u, \lambda)$.
By strong duality, minimizing $\Gamma$ over $\lambda \ge 0$ solves the
primal. A standard subgradient of $\Gamma$ at $\lambda$ is
\begin{align*}
s(\lambda) = \varepsilon_{\mathrm{KL}} - \KL(\mathbb{P}^{u_\lambda} \,\|\, \mathbb{P}^{\mathrm{base}})\text{,}
\end{align*}
where $u_\lambda := \arg\sup_u \mathcal{L}(u, \lambda)$ denotes the
inner maximizer at the current dual variable. The subgradient
inequality follows immediately. For any $\lambda' \ge 0$,
\[
\Gamma(\lambda') \ge \mathcal{L}(u_\lambda, \lambda')
= \Gamma(\lambda) + (\lambda' - \lambda)\bigl(\varepsilon_{\mathrm{KL}} - \KL(\mathbb{P}^{u_\lambda} \,\|\, \mathbb{P}^{\mathrm{base}})\bigr)\text{.}
\]

\paragraph{Projected dual descent.}
Projected subgradient descent on $\min_{\lambda \ge 0} \Gamma(\lambda)$ would
require evaluating the KL at the inner maximizer $u_\lambda$, which is
not available in closed form. In \cref{alg_full:trqam}, we instead
evaluate the path-space KL at the control $u_n$ at iteration $n$ via
the Monte Carlo surrogate $\widehat{D}_n$
from \cref{eq:app_pathkl_mc}. The dual update
\begin{align*}
\lambda_{n+1} \leftarrow \max\bigl\{0,\ \lambda_n + \eta_\lambda(\overline{D}_n - \varepsilon_{\mathrm{KL}})\bigr\}
\end{align*}
preserves the key sign property. When the realized KL exceeds the budget, $\lambda$ rises and the controlled SDE becomes more conservative. When it falls below the budget, $\lambda$ decreases and the policy can move further from $\pi_{\mathrm{base}}$ within the same budget.

\section{Detailed comparison of internal and external KL regularization}
\label{app:int-vs-ext}
\cref{sec:method_internal_external} contrasts two ways of pairing the 
dual update in \cref{eq:method_lambda_update} with a KL constraint, namely 
appending $\overline{D}_n$ as an auxiliary loss term (\emph{external}), 
or letting $\lambda$ scale the control term in the drift inside the SOC 
sampling process (\emph{internal}, TRQAM). The two formulations share the 
same dual update rule and the same KL estimator, yet they differ in 
how the constraint is enforced. This appendix expands on that 
contrast and explains why the difference is structural rather than 
cosmetic.

The key asymmetry is in what $\lambda$ represents. In the external formulation, $\lambda$ is a scalar weight on an auxiliary KL regularization term in the loss. It balances this 
regularization against the critic signal, but does not appear in 
the controlled SDE itself. The path-space KL between 
$\mathbb{P}^u$ and $\mathbb{P}^{\mathrm{base}}$ is therefore independent 
of $\lambda$, and is determined entirely by the learned control $u$. 
Adapting $\lambda$ only changes the relative weight between the 
adjoint-matching loss $\mathcal{L}_{\text{Adj-Match}}$ and the KL 
regularization $\overline{D}_n$ during gradient descent. Whether the realized KL ends up close 
to $\varepsilon_{\mathrm{KL}}$ depends on the optimizer balancing 
these competing signals. Under strong critic guidance, the KL term 
can be effectively overridden, and the realized KL drifts above the target KL bound $\varepsilon_{\mathrm{KL}}$ (\cref{fig:full_robomimic_lift_kl_compare_grid,fig:full_robomimic_can_kl_compare_grid}).

In the internal formulation, by contrast, $\lambda$ is a parameter of 
the controlled SDE itself. \cref{thm:path_kl} establishes that 
the path-space KL is exactly $\mathbb{E}[\tfrac{1}{2\lambda}\int \|u\|^2]$, 
so $\lambda$ appears as the inverse coefficient of the KL itself. 
Increasing $\lambda$ shrinks the path-space KL directly through the control term in the drift, rather than indirectly through the critic gradient. 
The dual update therefore reshapes the trajectory distribution 
structurally rather than competing with critic guidance. 
\cref{tab:internal-vs-external} summarizes these distinctions.

\begin{table}[ht]
\centering
\small
\caption{\textbf{Internal vs.\ external KL regularization.} Although both
approaches can use the same dual update rule on $\lambda$, the role of
$\lambda$ differs structurally. Internalizing $\lambda$ inside the SDE makes $\lambda$ the exact weight of the path-space KL in the SOC objective (\cref{thm:path_kl}), turning the target KL bound
$\varepsilon_{\mathrm{KL}}$ into a structural constraint rather than a
soft penalty.}
\label{tab:internal-vs-external}
\renewcommand{\arraystretch}{1.4}
\begin{tabular}{@{}>{\raggedright\arraybackslash}p{0.22\linewidth} >{\raggedright\arraybackslash}p{0.36\linewidth} >{\raggedright\arraybackslash}p{0.36\linewidth}@{}}
\toprule
& \textbf{External KL} (auxiliary regularization) & \textbf{Internal KL} (TRQAM) \\
\midrule
Controlled SDE
& $dX^u_\tau = (b + \sigma u)\,d\tau + \sigma(\tau)\,dB_\tau$
& $dX^u_\tau = (b + \sigma u)\,d\tau + \sqrt{\lambda}\,\sigma(\tau)\,dB_\tau$ \\
Path-space KL
& $\mathbb{E}\!\left[\tfrac{1}{2}\!\int_0^1\!\|u\|^2\,d\tau\right]$ \newline ($\lambda$-independent)
& $\mathbb{E}\!\left[\tfrac{1}{2\lambda}\!\int_0^1\!\|u\|^2\,d\tau\right]$ \newline (\cref{thm:path_kl}) \\
Role of $\lambda$
& Regularization weight on auxiliary loss
& Intrinsic SDE parameter \\
Training loss
& $\mathcal{L}_{\text{Adj-Match}}(\theta) + \lambda\,\overline{D}_n(\theta)$
& $\mathcal{L}_{\text{Adj-Match}}(\theta)$ \\
Effect of increasing $\lambda$
& Competes with critic guidance at loss level
& Reshapes the SDE at sampling level \\
Enforcement mechanism
& Soft (gradient competition)
& Structural (via the sampling process) \\
Realized KL vs.\ target bound
& Exceed the target bound (\cref{fig:full_robomimic_lift_kl_compare_grid,fig:full_robomimic_can_kl_compare_grid})
& Tracks target bound tightly \\
\bottomrule
\end{tabular}
\vspace{4pt}
\end{table}

\section{Additional experiments}
\label{app:additional_exp}

\subsection{KL-budget tracking}
\label{app:kl_tracking}

\begin{figure}[t]
    \centering
    \includegraphics[width=0.7\linewidth]{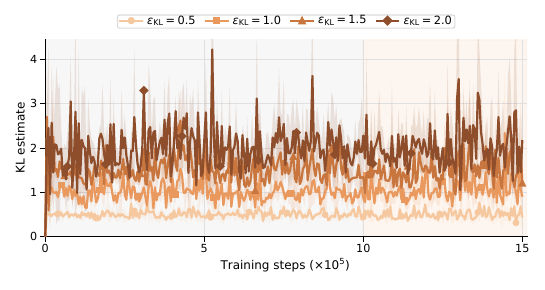}
    \caption{\textbf{TRQAM tracks the prescribed KL budget across 
    offline and online training.} We vary the target KL bound 
    $\varepsilon_{\mathrm{KL}} \in \{0.5, 1.0, 1.5, 2.0\}$ and plot 
    the realized path-space KL throughout training. Larger budgets 
    produce correspondingly larger policy deviation, and the 
    monotonic ordering is preserved across the offline-to-online 
    transition.}
    \label{fig:kl_bound}
\end{figure}

TRQAM faithfully enforces the prescribed KL budget across a range 
of target values. \cref{fig:kl_bound} plots the realized 
path-space KL under four budgets 
$\varepsilon_{\mathrm{KL}} \in \{0.5,\allowbreak 1.0,\allowbreak 1.5,\allowbreak 2.0\}$ throughout 
offline and online training. The realized KL tracks each prescribed 
target with a clear monotonic ordering, and this ordering is 
preserved across the offline-to-online transition, indicating that 
the dual update remains effective under distribution shift. The KL 
constraint in TRQAM is therefore both controllable and stable 
across training regimes.

\begin{figure}[htbp]
\centering
\includegraphics[width=\linewidth]{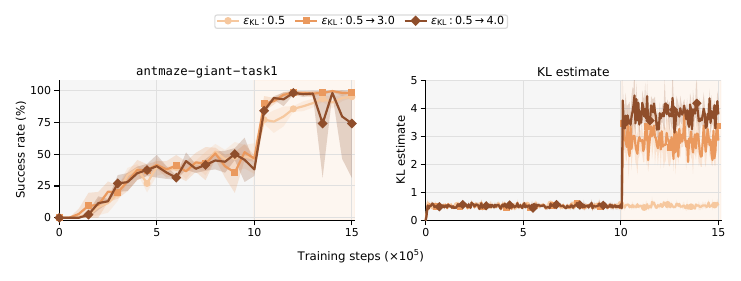}
\caption{\textbf{TRQAM tracks time-varying $\varepsilon_{\mathrm{KL}}$ 
schedules across the offline-to-online transition.} We illustrate 
this on \texttt{antmaze-giant}, the domain in our benchmark with the 
largest state space and accordingly the highest exploration demands 
during online fine-tuning. We keep $\varepsilon_{\mathrm{KL}} = 0.5$ 
during offline training and optionally switch to a larger bound at 
the online transition (vertical dashed line). In the right panel, the 
realized KL adapts to each new target almost immediately, showing 
that the dual update on $\lambda$ handles schedule changes without 
instability. In the left panel, the static schedule (0.5) 
improves more slowly online than the moderate schedule 
(0.5 → 3.0), which reaches $\sim$98\% success rapidly and 
outperforms the more aggressive (0.5 → 4.0) alternative.}
\label{fig:online_kl_schedule}
\end{figure}
\paragraph{Tracking accuracy.}
\label{app:kl_tracking_stats}
To quantify how closely TRQAM tracks the budget, we record the realized path-space KL throughout training on one representative task per OGBench domain (\texttt{task1} for \texttt{antmaze} and \texttt{humanoidmaze}, \texttt{task2} for \texttt{scene} and \texttt{cube}, and \texttt{task4} for \texttt{puzzle}), with a common budget $\varepsilon_{\mathrm{KL}} = 0.5$ and 2 seeds. We measure the relative deviation $|\overline{D}_n - \varepsilon_{\mathrm{KL}}| / \varepsilon_{\mathrm{KL}}$ at every training step. \cref{tab:kl_tracking_stats} shows that the realized KL stays centered at the budget on all domains, with a mean relative deviation of $5.0\%$ and a 95th-percentile deviation of $12.7\%$.

\begin{table}[ht]
\centering
\small
\caption{\textbf{KL-budget tracking accuracy} at $\varepsilon_{\mathrm{KL}} = 0.5$ (one representative task per domain, 2 seeds).}
\label{tab:kl_tracking_stats}
\begin{adjustbox}{max width=\linewidth}
\begin{tabular}{@{}lccccccccccc@{}}
\toprule
 & \texttt{al} & \texttt{ag} & \texttt{hm} & \texttt{hl} & \texttt{scene} & \texttt{p33} & \texttt{p44} & \texttt{c2} & \texttt{c3} & \texttt{c4} & Overall \\
\midrule
Path-space KL (mean) & $0.501$ & $0.501$ & $0.499$ & $0.500$ & $0.498$ & $0.500$ & $0.501$ & $0.498$ & $0.499$ & $0.501$ & -- \\
Relative deviation (mean) & $5.2\%$ & $6.0\%$ & $4.3\%$ & $4.9\%$ & $4.3\%$ & $4.6\%$ & $4.3\%$ & $3.7\%$ & $6.0\%$ & $6.4\%$ & $5.0\%$ \\
Relative deviation (95th pct.) & $11.8\%$ & $15.0\%$ & $10.3\%$ & $12.2\%$ & $10.5\%$ & $11.1\%$ & $10.6\%$ & $8.8\%$ & $17.0\%$ & $15.8\%$ & $12.7\%$ \\
\bottomrule
\end{tabular}
\end{adjustbox}
\end{table}

\subsection{Time-varying KL budget}
\label{app:time_varying}
The KL bound $\varepsilon_{\mathrm{KL}}$ is not necessarily a 
fixed quantity throughout training. In offline-to-online RL, the 
appropriate trust region may differ between phases, since online 
fine-tuning often requires more exploration than the offline phase 
permits. TRQAM accommodates this naturally because $\lambda$ is 
internalized in the SOC sampling process rather than entering as a soft 
loss penalty. As a result, the dual update on $\lambda$ remains stable under 
time-varying $\varepsilon_{\mathrm{KL}}$, allowing the bound to be 
re-set at any point without retraining or instability.

We illustrate this on \texttt{antmaze-giant}, the domain in our 
benchmark with the largest state space and accordingly the highest 
exploration demands during online fine-tuning. 
\Cref{fig:online_kl_schedule} sweeps three schedules on 
\texttt{antmaze-giant-task1}, the tuning task for our per-domain 
$\varepsilon_{\mathrm{KL}}$ choice. These are the static schedule 
(0.5, matching the offline value) and two relaxations at 
the online transition (0.5 → 3.0, 0.5 → 4.0). 
Two observations follow.

First, the realized KL adapts to each new target almost immediately 
after the switch, confirming that TRQAM tracks time-varying bounds 
without loss of stability.

Second, the static schedule (0.5) improves substantially more 
slowly online than the moderate schedule ($0.5 \to 3.0$). On 
domains with large state spaces such as \texttt{antmaze-giant}, 
online fine-tuning relies on exploring beyond the offline support, 
and a tight KL budget restricts this exploration and slows 
adaptation to the new state distribution. The moderate schedule 
reaches $\sim$98\% success rapidly, while the more aggressive 
$0.5 \to 4.0$ introduces substantial variance, identifying $3.0$ 
as a sweet spot.

\subsection{Sensitivity analysis}
\label{app:sensitivity}
We extend the four-task sweep of \cref{sec:exp_sensitivity} 
to all ten OGBench domains, sweeping 
$\varepsilon_{\mathrm{KL}} \in \{0.5,\allowbreak 1.0,\allowbreak 1.5,\allowbreak 2.0,\allowbreak 2.5,\allowbreak 3.0,\allowbreak 3.5,\allowbreak 4.0\}$ 
on one representative task per domain with all other 
hyperparameters fixed. Results across $8$ seeds appear in 
\cref{fig:ogbench_klb_sensitivity_grid}, showing two 
patterns. Tight budgets are best on \texttt{humanoidmaze-medium}, 
\texttt{humanoidmaze-large}, \texttt{cube-double}, 
\texttt{cube-triple}, and \texttt{cube-quadruple}. Larger budgets 
monotonically improve performance on \texttt{puzzle-4x4}, in 
keeping with its notably larger state space.

Smaller $\varepsilon_{\mathrm{KL}}$ also produces slower online 
adaptation across the sweep, consistent with the dual update 
tightening the trust region. We observe the same effect on 
\texttt{antmaze-giant} (\cref{app:time_varying}), where the 
static $0.5$ schedule slows online adaptation in a domain whose 
large state space demands broader exploration.

\begin{figure}[p]
    \centering
    \setlength{\abovecaptionskip}{4pt}
    \setlength{\belowcaptionskip}{0pt}
    \includegraphics[width=\linewidth,height=0.85\textheight,keepaspectratio]{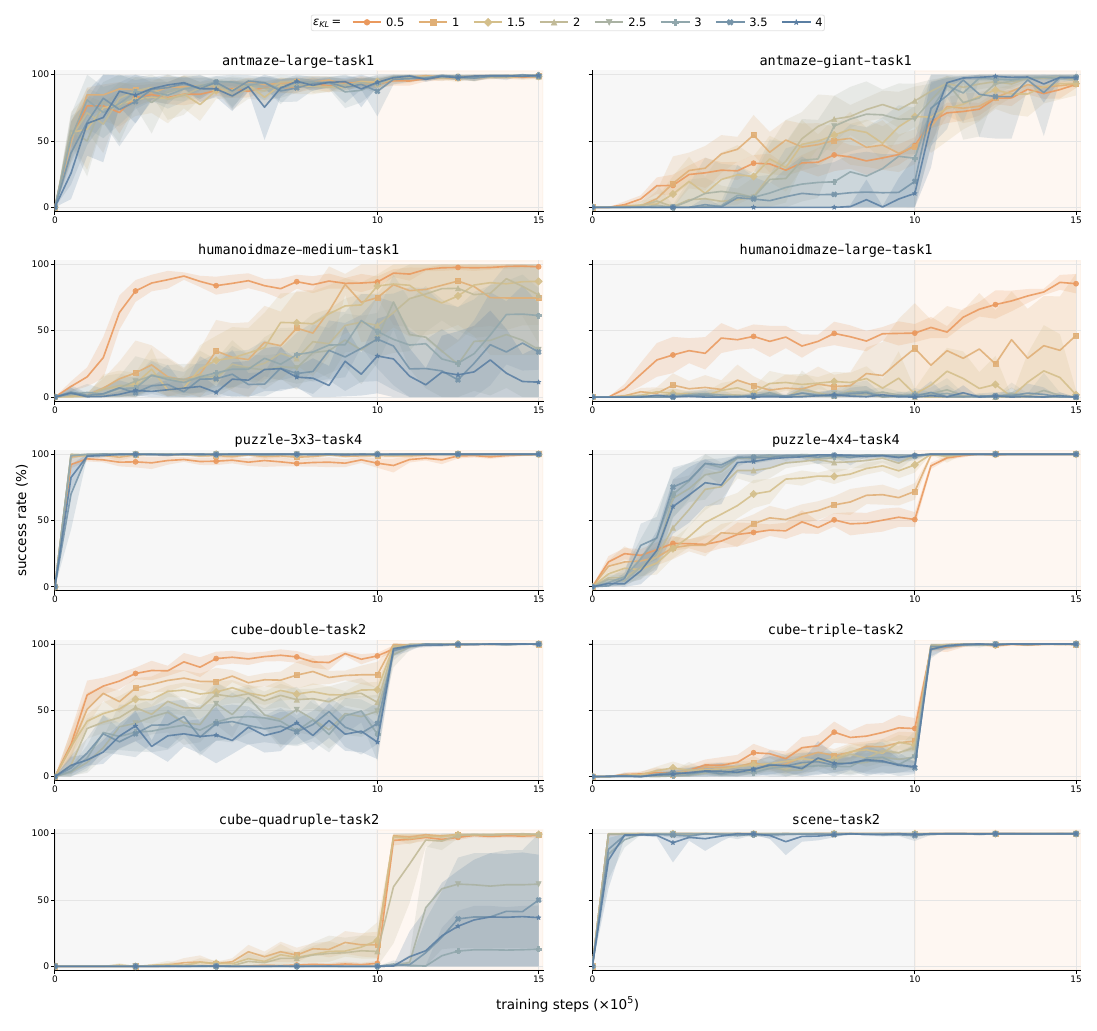}
    \caption{\textbf{Hyperparameter $\varepsilon_{\mathrm{KL}}$ sweep on OGBench~\citep{ogbench_park2025} (8 seeds).}
    We sweep $\varepsilon_{\mathrm{KL}} \in \{0.5,\,\allowbreak 1.0,\,\allowbreak 1.5,\,\allowbreak 2.0,\,\allowbreak 2.5,\,\allowbreak 3.0,\,\allowbreak 3.5,\,\allowbreak 4.0\}$
    on one representative task per domain, with all other hyperparameters fixed
    to the main-comparison setting. Each panel reports success rate over training
    steps, with the offline-to-online transition at $10^{6}$ steps. Shaded regions
    denote $\pm 1$ standard deviation across seeds. Evaluated tasks are listed in 
    the panel titles. Two patterns emerge. Tight budgets are best on 
    \texttt{humanoidmaze-medium}, \texttt{humanoidmaze-large}, \texttt{cube-double}, 
    \texttt{cube-triple}, and \texttt{cube-quadruple}, while larger budgets 
    monotonically improve performance on \texttt{puzzle-4x4}, consistent with 
    its notably larger state space. Smaller $\varepsilon_{\mathrm{KL}}$ also 
    produces slower online adaptation across the sweep, consistent with the dual 
    update tightening the trust region.}
    \label{fig:ogbench_klb_sensitivity_grid}
\end{figure}

\subsection{A single global KL budget}
\label{app:global_budget}
The main results use the per-domain budgets of \cref{tab:hparam}. To measure how much this tuning matters, we rerun TRQAM with a single global budget $\varepsilon_{\mathrm{KL}} = 0.5$ on every OGBench task whose tuned budget differs from $0.5$, with 8 seeds. \cref{tab:global_budget} reports the aggregate success rate over all 50 tasks at 1M steps. The global budget reaches $63\%$, close to the $68\%$ of per-domain tuning and well above the strongest baseline at $46\%$. Excluding \texttt{puzzle-4x4}, where the score drops from $99$ to $62$, the global budget stays within about one point of the tuned result. This exception matches \cref{app:sensitivity}, where \texttt{puzzle-4x4} benefits monotonically from larger budgets because of its larger state space.

\begin{table}[ht]
\centering
\small
\caption{\textbf{Per-domain versus a single global KL budget} (aggregate success rate over all 50 OGBench tasks, 8 seeds).}
\label{tab:global_budget}
\begin{adjustbox}{max width=\linewidth}
\begin{tabular}{@{}lccccccccccc@{}}
\toprule
Budget & \texttt{al} & \texttt{ag} & \texttt{hm} & \texttt{hl} & \texttt{scene} & \texttt{p33} & \texttt{p44} & \texttt{c2} & \texttt{c3} & \texttt{c4} & Mean \\
\midrule
Per-domain & $89 \pm 4$ & $41 \pm 4$ & $84 \pm 3$ & $36 \pm 4$ & $79 \pm 1$ & $100 \pm 0$ & $99 \pm 1$ & $81 \pm 3$ & $50 \pm 5$ & $19 \pm 5$ & $68$ \\
Global $\varepsilon_{\mathrm{KL}} = 0.5$ & $89 \pm 4$ & $41 \pm 4$ & $84 \pm 3$ & $36 \pm 4$ & $79 \pm 1$ & $97 \pm 3$ & $62 \pm 8$ & $81 \pm 3$ & $50 \pm 5$ & $14 \pm 3$ & $63$ \\
\bottomrule
\end{tabular}
\end{adjustbox}
\end{table}

\subsection{Gap between the path-space KL and the terminal KL}
\label{app:pathkl_gap}
\cref{prop:dpi} bounds the terminal KL by the path-space KL. By the chain rule of KL divergence, the gap is exactly~\citep{thilges2026scalablemaximumentropyreinforcement}
\begin{align*}
    D_{\mathrm{KL}}(\mathbb{P}^u \,\|\, \mathbb{P}^{\mathrm{base}})
    = D_{\mathrm{KL}}(\pi_\theta \,\|\, \pi_{\mathrm{base}})
    + \E_{X_1 \sim \pi_\theta}\bigl[D_{\mathrm{KL}}\bigl(\mathbb{P}^u(\cdot \mid X_1) \,\|\, \mathbb{P}^{\mathrm{base}}(\cdot \mid X_1)\bigr)\bigr]\text{,}
\end{align*}
which measures how differently the two processes reach the same action. To check how loose the bound is in practice, we measure both quantities on a Gaussian-mixture target in dimensions $d \in \{2,\allowbreak 8,\allowbreak 16\}$ under eight budgets. We estimate the terminal KL by the change of variables with a full Jacobian trace over 3000 samples~\citep{lipman2024flowmatchingguidecode}.
 As \cref{tab:pathkl_gap} shows, the terminal KL is never negligible. In all 24 configurations, it stays at least about $40\%$ of the path-space KL, so the bound in \cref{prop:dpi} remains informative in practice.

\begin{table}[ht]
\centering
\small
\caption{\textbf{Realized path-space KL / terminal KL} on a Gaussian-mixture target.}
\label{tab:pathkl_gap}
\begin{adjustbox}{max width=\linewidth}
\begin{tabular}{@{}lcccccccc@{}}
\toprule
$d \backslash \varepsilon_{\mathrm{KL}}$ & $0.5$ & $1.0$ & $1.5$ & $2.0$ & $2.5$ & $3.0$ & $3.5$ & $4.0$ \\
\midrule
2 & $0.51/0.48$ & $0.99/0.94$ & $1.47/0.84$ & $2.03/1.71$ & $2.44/2.40$ & $2.45/1.63$ & $2.41/1.83$ & $3.52/1.76$ \\
8 & $0.80/0.35$ & $1.02/0.58$ & $1.47/1.03$ & $1.74/1.73$ & $2.40/2.35$ & $2.85/2.17$ & $3.31/2.89$ & $3.97/2.99$ \\
16 & $1.45/0.92$ & $1.38/0.84$ & $1.59/0.70$ & $1.99/0.90$ & $2.44/1.76$ & $2.92/1.94$ & $3.36/1.69$ & $4.04/3.41$ \\
\bottomrule
\end{tabular}
\end{adjustbox}
\end{table}

\subsection{Controlling the deviation with \texorpdfstring{$\lambda$}{lambda} versus \texorpdfstring{$\beta$}{beta}}
\label{app:lambda_vs_beta}
Both $\beta$ in QAM and $\lambda$ in TRQAM move the policy away from $\pi_{\mathrm{base}}$ (\cref{sec:method_path_kl}). They differ in predictability. To compare them, we measure the realized path-space KL and the terminal KL on three 2D testbeds from \citet{doo2026qflowstableexpressivereinforcement} (8-gaussians, swiss-roll, and spiral), estimating the terminal KL by the change of variables over 9000 samples~\citep{lipman2024flowmatchingguidecode}. Under QAM (\cref{tab:lambda_vs_beta}, top), the KL for a given $\beta$ varies substantially across datasets and responds unevenly. It stays nearly flat from $\beta = 0.1$ to $1$ and then rises sharply between $\beta = 3$ and $30$. Setting $\beta$ therefore gives no indication of how far the policy will move. Under TRQAM (bottom), the realized path-space KL instead follows the specified budget at a similar level on all three datasets. Setting $\varepsilon_{\mathrm{KL}}$ thus predicts how far the policy will move.

\begin{table}[ht]
\centering
\small
\caption{\textbf{Realized path-space KL / terminal KL} under QAM with inverse temperature $\beta$ (top) and TRQAM with budget $\varepsilon_{\mathrm{KL}}$ (bottom).}
\label{tab:lambda_vs_beta}
\begin{adjustbox}{max width=\linewidth}
\begin{tabular}{@{}lcccccccc@{}}
\toprule
QAM & $\beta = 0.1$ & $0.3$ & $1$ & $3$ & $10$ & $30$ & & \\
\midrule
8-gaussians & $0.16/0.02$ & $0.17/0.02$ & $0.25/0.10$ & $1.15/0.99$ & $3.38/3.33$ & $4.59/3.48$ & & \\
swiss-roll & $0.15/0.04$ & $0.15/0.03$ & $0.20/0.04$ & $0.76/0.48$ & $2.95/1.46$ & $7.92/4.04$ & & \\
spiral & $0.58/0.24$ & $0.58/0.23$ & $0.59/0.25$ & $0.77/0.38$ & $1.97/1.34$ & $4.27/2.11$ & & \\
\midrule
TRQAM & $\varepsilon_{\mathrm{KL}} = 0.5$ & $1.0$ & $1.5$ & $2.0$ & $2.5$ & $3.0$ & $3.5$ & $4.0$ \\
\midrule
8-gaussians & $0.47/0.45$ & $0.97/0.90$ & $1.50/0.86$ & $1.95/1.26$ & $2.38/2.25$ & $2.89/0.87$ & $2.59/1.64$ & $3.56/3.03$ \\
swiss-roll & $0.50/0.27$ & $1.01/0.46$ & $1.52/1.02$ & $2.06/1.63$ & $2.64/2.02$ & $2.39/0.77$ & $2.72/0.63$ & $3.16/1.36$ \\
spiral & $0.74/0.34$ & $1.07/0.49$ & $1.49/1.45$ & $1.99/1.03$ & $2.38/1.64$ & $2.56/1.26$ & $3.15/2.18$ & $3.92/2.41$ \\
\bottomrule
\end{tabular}
\end{adjustbox}
\end{table}

\subsection{Robustness to the dual-update hyperparameters}
\label{app:dual_sensitivity}
TRQAM uses a single setting of the dual step size $\eta_\lambda$ and the EMA coefficient $\rho_D$ across all tasks (\cref{tab:rl-hyperparams}). To test how much this choice matters, we vary each of them on one representative task per OGBench domain with 8 seeds. \cref{tab:dual_sensitivity} reports the results. $\rho_D$ has little effect, giving the same mean of $70.2$ under heavy smoothing ($\rho_D = 0.01$) and no smoothing ($\rho_D = 1$). For $\eta_\lambda$, performance stays close to the default at $0.001$ and drops only at $0.1$. Because $\lambda$ enters the SOC formulation directly, a large dual step makes $\lambda$ react greedily to each KL estimate, so a small step size is the natural choice.

\begin{table}[H]
\centering
\small
\caption{\textbf{Sensitivity to the dual-update hyperparameters} (success rate on one representative task per domain, 8 seeds).}
\label{tab:dual_sensitivity}
\begin{adjustbox}{max width=\linewidth}
\begin{tabular}{@{}lccccccccccc@{}}
\toprule
Setting & \texttt{al} & \texttt{ag} & \texttt{hm} & \texttt{hl} & \texttt{scene} & \texttt{p33} & \texttt{p44} & \texttt{c2} & \texttt{c3} & \texttt{c4} & Mean \\
\midrule
Default ($\eta_\lambda = 0.01$, $\rho_D = 0.1$) & $95 \pm 4$ & $48 \pm 5$ & $87 \pm 5$ & $45 \pm 8$ & $100 \pm 0$ & $100 \pm 1$ & $100 \pm 1$ & $92 \pm 3$ & $42 \pm 9$ & $18 \pm 11$ & $72.7$ \\
$\eta_\lambda = 0.001$ & $94 \pm 4$ & $37 \pm 10$ & $87 \pm 5$ & $50 \pm 12$ & $100 \pm 0$ & $100 \pm 0$ & $100 \pm 1$ & $88 \pm 5$ & $28 \pm 8$ & $14 \pm 9$ & $69.8$ \\
$\eta_\lambda = 0.1$ & $84 \pm 10$ & $4 \pm 4$ & $86 \pm 5$ & $41 \pm 11$ & $100 \pm 0$ & $100 \pm 1$ & $91 \pm 7$ & $87 \pm 5$ & $26 \pm 12$ & $1 \pm 1$ & $62.0$ \\
$\rho_D = 0.01$ & $94 \pm 5$ & $43 \pm 6$ & $86 \pm 4$ & $47 \pm 8$ & $100 \pm 0$ & $100 \pm 0$ & $99 \pm 2$ & $90 \pm 8$ & $29 \pm 8$ & $14 \pm 8$ & $70.2$ \\
$\rho_D = 1$ & $95 \pm 3$ & $45 \pm 11$ & $87 \pm 8$ & $42 \pm 6$ & $100 \pm 0$ & $100 \pm 1$ & $99 \pm 2$ & $84 \pm 4$ & $32 \pm 8$ & $18 \pm 5$ & $70.2$ \\
\bottomrule
\end{tabular}
\end{adjustbox}
\end{table}

\subsection{Number of critics}
\label{app:num_critics}
All methods use an ensemble of 10 critics (\cref{tab:rl-hyperparams}). Fewer critics lead to larger critic errors, which \cref{lem:exp_amp} predicts should hurt QAM more than TRQAM. To test this, we compare 2 and 10 critics on one representative task per OGBench domain with 8 seeds. \cref{tab:num_critics} shows that reducing the ensemble lowers TRQAM by only 4 points ($73 \to 69$) but QAM by 7 points ($37 \to 30$), so the gap between them widens from 36 to 39 points. Even with only 2 critics, TRQAM stays more than 30 points above QAM with 10 critics. This is consistent with our claim that the trust region limits the effect of critic errors.

\begin{table}[H]
\centering
\small
\caption{\textbf{Effect of the number of critics} (success rate on one representative task per domain, 8 seeds).}
\label{tab:num_critics}
\begin{adjustbox}{max width=\linewidth}
\begin{tabular}{@{}lccccccccccc@{}}
\toprule
Method & \texttt{al} & \texttt{ag} & \texttt{hm} & \texttt{hl} & \texttt{scene} & \texttt{p33} & \texttt{p44} & \texttt{c2} & \texttt{c3} & \texttt{c4} & Mean \\
\midrule
QAM (10 critics) & $96 \pm 3$ & $37 \pm 9$ & $33 \pm 34$ & $6 \pm 10$ & $100 \pm 0$ & $0 \pm 0$ & $0 \pm 1$ & $81 \pm 6$ & $15 \pm 11$ & $2 \pm 4$ & $37$ \\
QAM (2 critics) & $74 \pm 37$ & $41 \pm 10$ & $13 \pm 24$ & $5 \pm 10$ & $100 \pm 1$ & $0 \pm 1$ & $0 \pm 0$ & $59 \pm 12$ & $5 \pm 5$ & $1 \pm 2$ & $30$ \\
TRQAM (10 critics) & $95 \pm 4$ & $48 \pm 5$ & $87 \pm 5$ & $45 \pm 8$ & $100 \pm 0$ & $100 \pm 1$ & $100 \pm 1$ & $92 \pm 3$ & $42 \pm 9$ & $18 \pm 11$ & $73$ \\
TRQAM (2 critics) & $90 \pm 4$ & $52 \pm 11$ & $89 \pm 4$ & $44 \pm 10$ & $100 \pm 1$ & $100 \pm 0$ & $100 \pm 1$ & $78 \pm 7$ & $24 \pm 11$ & $15 \pm 11$ & $69$ \\
\bottomrule
\end{tabular}
\end{adjustbox}
\end{table}

\subsection{Computational cost}
\label{app:vjp_cost}
Estimating the path-space KL reuses the velocity difference $v^{\mathrm{ft}}_\theta - v^{\mathrm{base}}$ already computed for the Adjoint Matching loss, so the trust region adds $O(1)$ time and memory per step over QAM. The main cost shared with QAM is the vector-Jacobian product (VJP) in the lean adjoint ODE. \cref{tab:vjp_actor,tab:vjp_actiondim} report its share of per-step training time. The share grows with the actor size, rising from $9.2\%$ to $21.1\%$ over a $240\times$ increase in parameters (\cref{tab:vjp_actor}). It does not grow without bound, but at about a fifth of each training step for larger actors, the VJP cost becomes hard to ignore as the policy grows (\cref{app:limitations}). The action dimension has little effect (\cref{tab:vjp_actiondim}).

\begin{table}[ht]
\centering
\small
\setlength{\parskip}{0pt}
\captionsetup{justification=centering, skip=6pt}
\begin{minipage}[t]{0.42\linewidth}
\centering
\caption{\textbf{VJP share of per-step training time versus actor size} (action dimension 14).}
\label{tab:vjp_actor}
\begin{tabular}{@{}lcc@{}}
\toprule
Actor size & Time & FLOPs \\
\midrule
$0.84$M & $9.2\%$ & $13.1\%$ \\
$12.78$M & $19.4\%$ & $22.8\%$ \\
$202.12$M & $21.1\%$ & $24.0\%$ \\
\bottomrule
\end{tabular}
\end{minipage}\hfill
\begin{minipage}[t]{0.52\linewidth}
\centering
\caption{\textbf{VJP share of per-step training time versus action dimension} for a small and a large actor.}
\label{tab:vjp_actiondim}
\begin{tabular}{@{}lcc@{}}
\toprule
Action dim. & $0.84$M actor & $50.7$M actor \\
\midrule
5 & $10.2\%$ & $19.3\%$ \\
21 & $9.8\%$ & $18.8\%$ \\
128 & $10.1\%$ & $22.6\%$ \\
1024 & $10.9\%$ & $19.8\%$ \\
\bottomrule
\end{tabular}
\end{minipage}
\end{table}

\subsection{Stability stress test on Robomimic}
\label{app:robomimic}
We complement our OGBench results with Robomimic~\citep{robomimic2021}, a standard manipulation benchmark on which we observe 
fixed-temperature Adjoint Matching to be unstable.

\paragraph{Fixed-temperature methods collapse across the sweep.}
Sweeping QAM over six values of 
$\beta \in \{0.01,\allowbreak 0.03,\allowbreak 0.1,\allowbreak 0.3,\allowbreak 1.0,\allowbreak 3.0\}$ and QAM-E over eight 
$(\beta, \sigma_a)$ configurations, the collapse pattern persists 
across most settings on \texttt{lift} and \texttt{can} 
(\cref{fig:full_robomimic_lift_curves_grid_4by2,fig:full_robomimic_can_curves_grid_4by2}), where the 
adjoint-matching loss explodes during offline training and success 
rate collapses to near zero. On \texttt{square}, the explosion is milder, but the adjoint loss of the fixed-temperature variants still grows steadily during offline training, while TRQAM keeps it bounded (\cref{fig:full_robomimic_square_curves_grid_4by2}). This is consistent with the structural 
fragility predicted by fixed-$\beta$ critic-error amplification 
(\cref{lem:exp_amp}).

\paragraph{External KL is insufficient, while TRQAM enforces the budget.}
External KL regularization partially mitigates the collapse, but the realized path-space KL 
substantially exceeds the prescribed bound across every budget, 
leaving the policy vulnerable to critic-error amplification. TRQAM, 
by contrast, tracks each budget tightly throughout offline training 
and across the offline-to-online transition, and remains stable 
across the full sweep 
$\varepsilon_{\mathrm{KL}} \in \{0.01,\allowbreak 0.03,\allowbreak 0.1,\allowbreak 0.5,\allowbreak 1,\allowbreak 1.5\}$ (\cref{fig:full_robomimic_lift_kl_compare_grid,fig:full_robomimic_can_kl_compare_grid,fig:full_robomimic_square_kl_compare_grid}).

\paragraph{Best configuration of each method.}
\cref{fig:full_robomimic_best_hp_compare} compares the best configuration of each method on all three tasks.

\section{Full OGBench results}
\label{app:full_results}
\cref{tab:full-results} reports the offline success rate of every method on all 50 OGBench tasks, and \cref{fig:per_task_curves_1,fig:per_task_curves_2} show the per-task offline-to-online learning curves.

\clearpage
\begin{figure}[H]
    \centering
    \setlength{\abovecaptionskip}{4pt}
    \setlength{\belowcaptionskip}{0pt}
    \includegraphics[width=\linewidth,height=0.40\textheight,keepaspectratio]{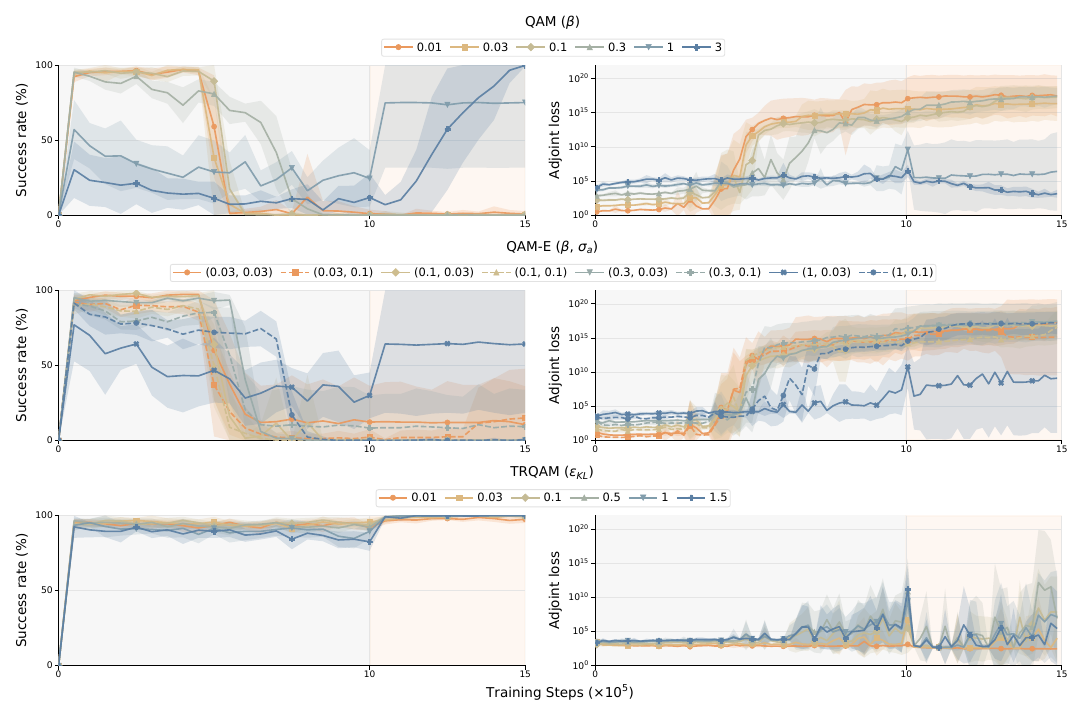}
    \caption{\textbf{Hyperparameter sweep on \texttt{Robomimic-lift}
(8 seeds).} Sweep ranges for QAM, QAM-E, and TRQAM follow \cref{tab:hparam-range}. The left
column reports success rate, and the right column reports adjoint-matching loss
on log scale. Fixed-temperature methods (QAM, QAM-E) exhibit adjoint
loss growth of 10 to 20 orders of magnitude across most settings, with
corresponding success rate collapse. TRQAM remains stable across all six
budgets. Shaded regions denote $\pm$1 standard deviation across seeds.}
    \label{fig:full_robomimic_lift_curves_grid_4by2}
\end{figure}
\begin{figure}[H]
    \centering
    \setlength{\abovecaptionskip}{4pt}
    \setlength{\belowcaptionskip}{0pt}
    \includegraphics[width=\linewidth,height=0.40\textheight,keepaspectratio]{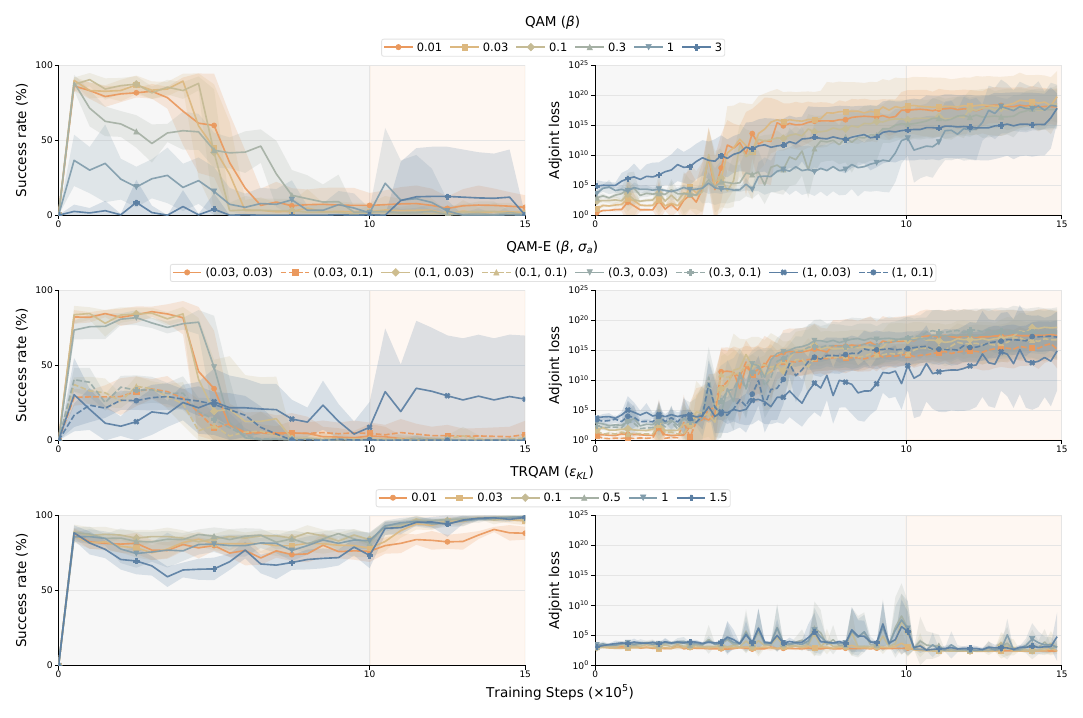}
    \caption{\textbf{Hyperparameter sweep on \texttt{Robomimic-can}
    (8 seeds).} Same setup as
    \cref{fig:full_robomimic_lift_curves_grid_4by2}. Fixed-temperature
    collapse is again hyperparameter-wide. QAM and QAM-E exhibit adjoint
    loss growth of 10 to 25 orders of magnitude across most settings with
    success rate collapse, while TRQAM remains stable across all six
    budgets. Shaded regions denote $\pm$1 standard deviation across seeds.}
    \label{fig:full_robomimic_can_curves_grid_4by2}
\end{figure}
\clearpage
\begin{figure}[H]
    \centering
    \setlength{\abovecaptionskip}{4pt}
    \setlength{\belowcaptionskip}{0pt}
    \includegraphics[width=\linewidth,height=0.40\textheight,keepaspectratio]{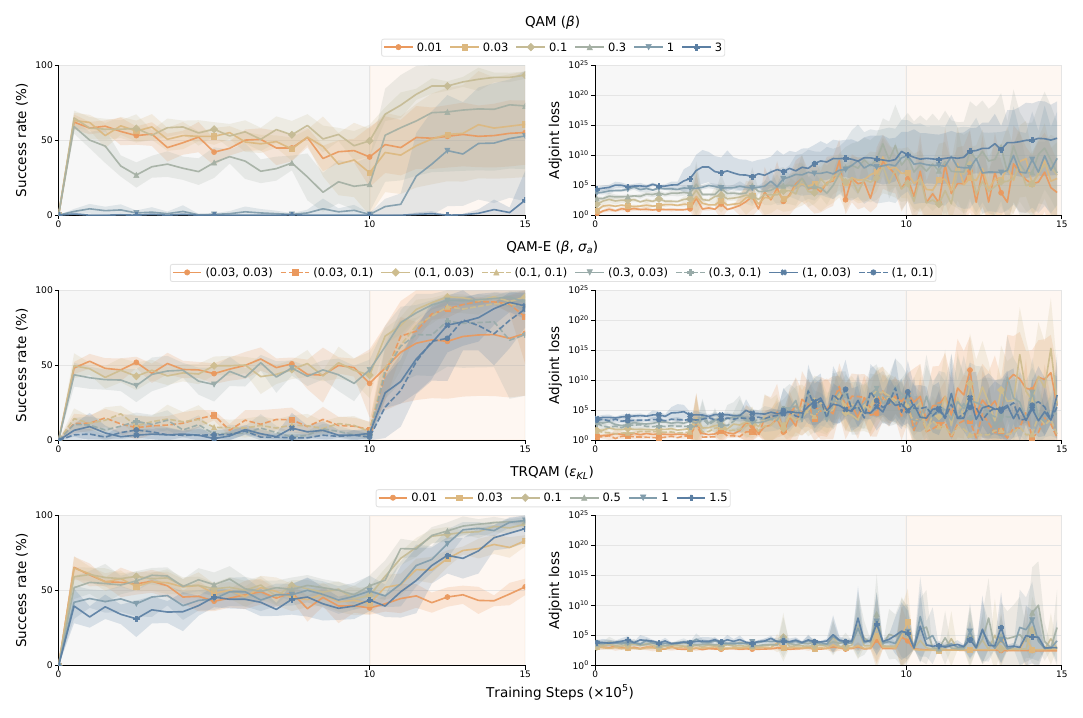}
    \caption{\textbf{Hyperparameter sweep on \texttt{Robomimic-square}
    (8 seeds).} Same setup as
    \cref{fig:full_robomimic_lift_curves_grid_4by2}. While the adjoint 
    loss explosion is less severe than on \texttt{lift} and \texttt{can}, 
    fixed-temperature variants still exhibit signs of instability across 
    the sweep, with adjoint loss steadily growing during offline training. 
    TRQAM, in contrast, maintains a bounded adjoint loss and matches or 
    exceeds the fixed-temperature variants across the sweep. Shaded 
    regions denote $\pm$1 standard deviation across seeds.}
    \label{fig:full_robomimic_square_curves_grid_4by2}
\end{figure}
\begin{figure}[H]
    \centering
    \setlength{\abovecaptionskip}{4pt}
    \setlength{\belowcaptionskip}{0pt}
    \includegraphics[width=\linewidth,height=0.85\textheight,keepaspectratio]{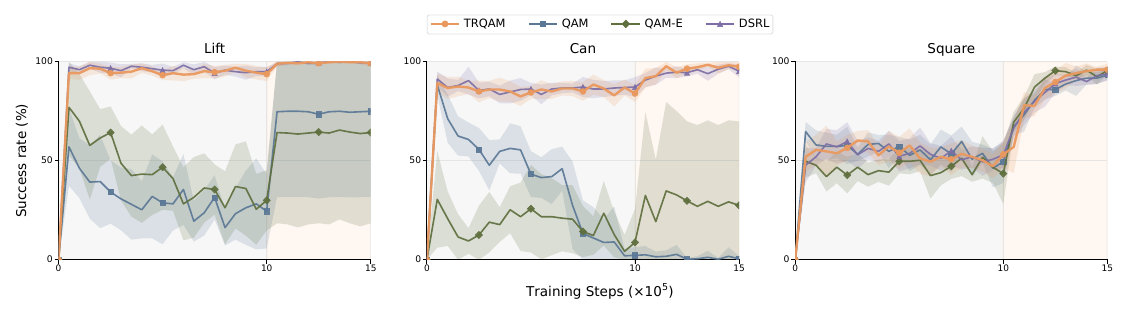}
    \caption{\textbf{Best per-method hyperparameters on Robomimic across all
    three tasks (8 seeds).} For each method (TRQAM, QAM, QAM-E, DSRL) we
    select the configuration with the strongest overall offline-to-online
    learning curve from the sweep in \cref{tab:hparam-range}. Shaded regions
    denote $\pm$1 standard deviation across seeds.}
    \label{fig:full_robomimic_best_hp_compare}
\end{figure}
\clearpage
\begin{figure}[p]
    \centering
    \setlength{\abovecaptionskip}{4pt}
    \setlength{\belowcaptionskip}{0pt}
    \includegraphics[width=\linewidth,height=0.85\textheight,keepaspectratio]{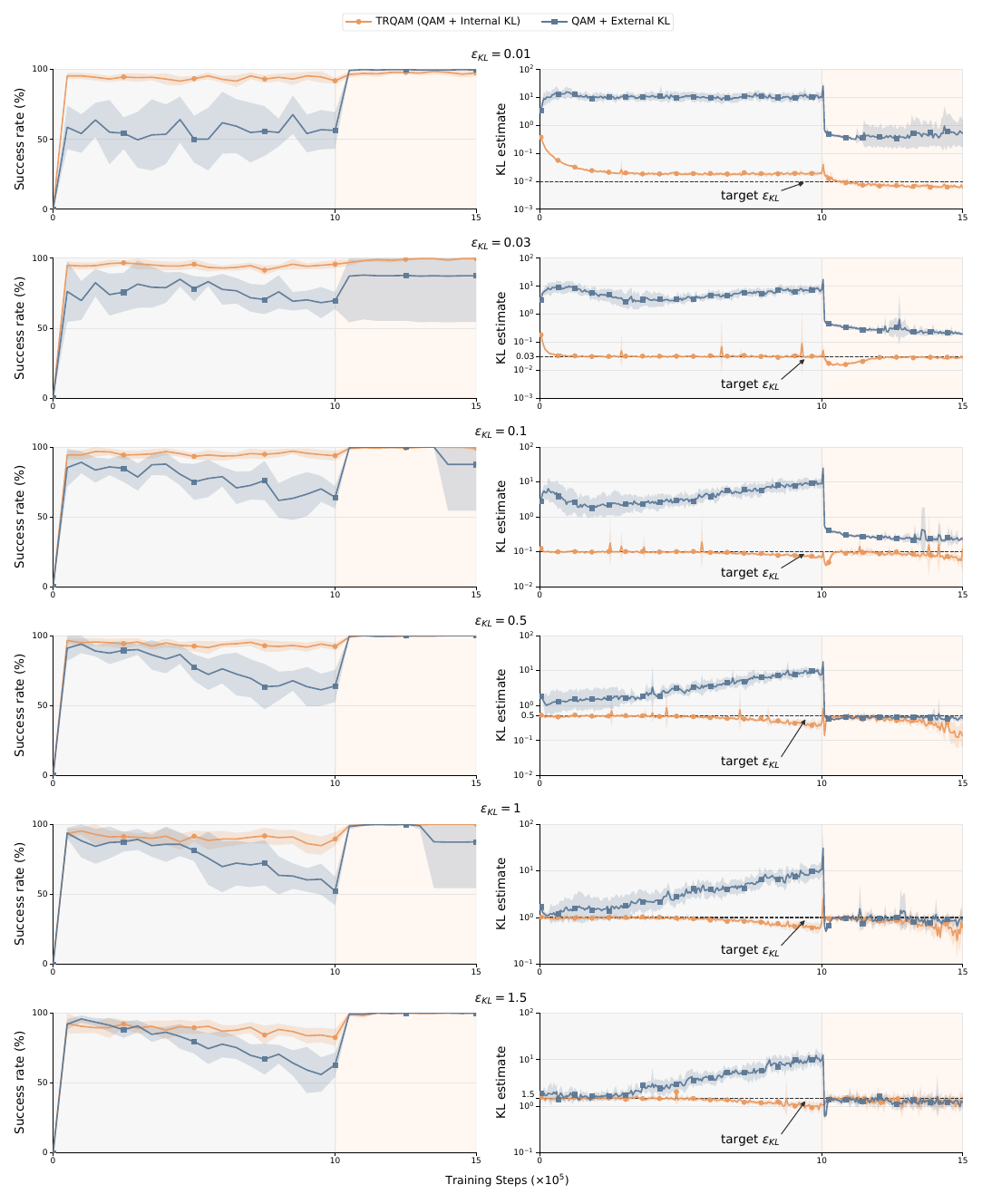}
    \caption{\textbf{Internal vs.\ external KL regularization on
    \texttt{Robomimic-lift} across all budgets (8 seeds).} Each row plots
    one target KL budget
    $\varepsilon_{\mathrm{KL}} \in \{0.01,\allowbreak 0.03,\allowbreak 0.1,\allowbreak 0.5,\allowbreak 1.0,\allowbreak 1.5\}$.
    The left column shows success rate over training steps, and the right column shows realized
    path-space KL with the target shown as a dashed line. TRQAM (orange) tracks
    each prescribed budget tightly across offline and online training,
    whereas QAM with external KL regularization (blue) lets the realized KL
    drift well over it, with corresponding success rate degradation. Shaded regions
    denote $\pm$1 standard deviation across seeds.}
    \label{fig:full_robomimic_lift_kl_compare_grid}
\end{figure}
\begin{figure}[p]
    \centering
    \setlength{\abovecaptionskip}{4pt}
    \setlength{\belowcaptionskip}{0pt}
    \includegraphics[width=\linewidth,height=0.85\textheight,keepaspectratio]{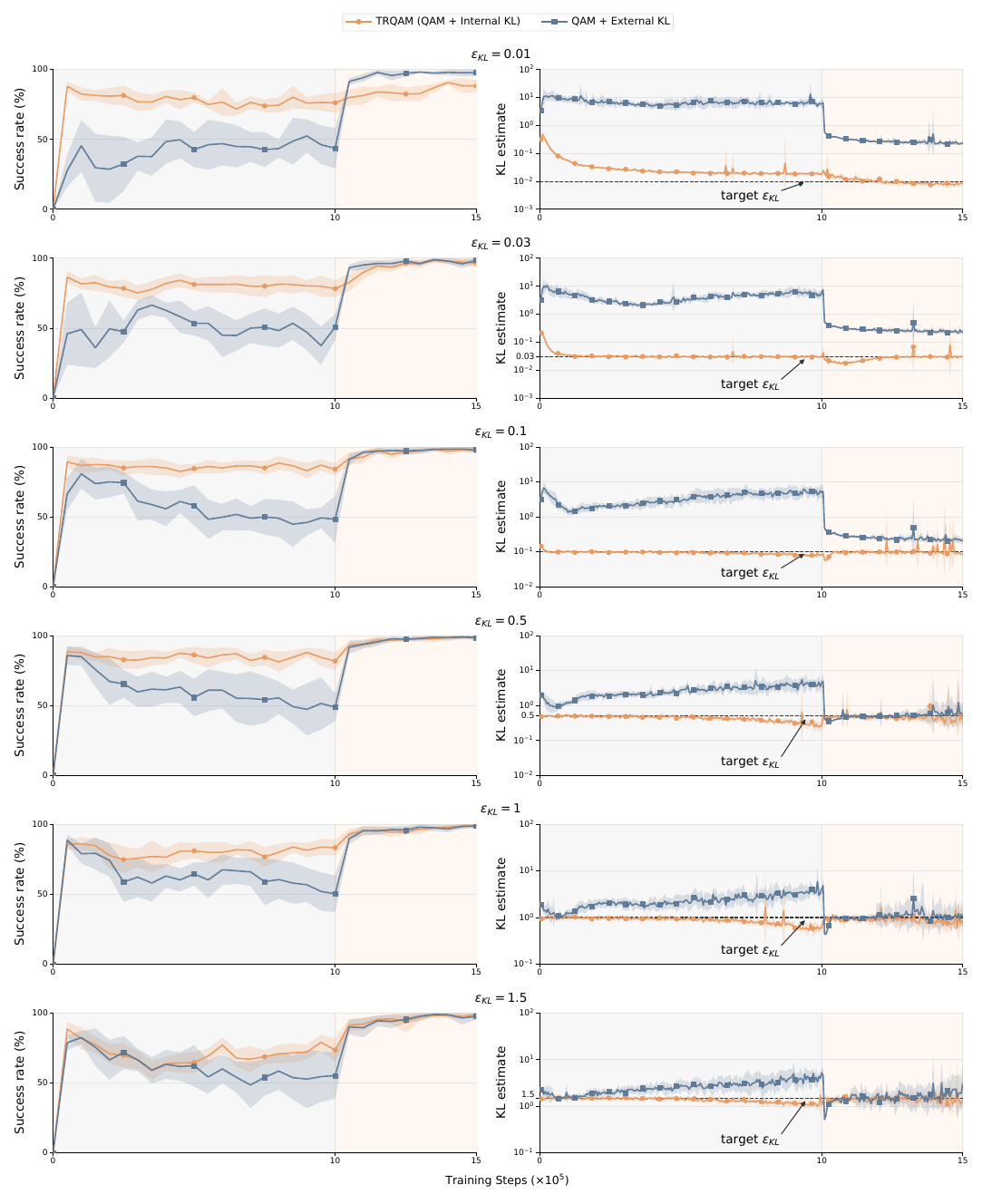}
    \caption{\textbf{Internal vs.\ external KL regularization on
    \texttt{Robomimic-can} across all budgets (8 seeds).} Each row plots
    one target KL budget
    $\varepsilon_{\mathrm{KL}} \in \{0.01,\allowbreak 0.03,\allowbreak 0.1,\allowbreak 0.5,\allowbreak 1.0,\allowbreak 1.5\}$.
    The left column shows success rate over training steps, and the right column shows realized
    path-space KL with the target shown as a dashed line. TRQAM (orange) tracks
    each prescribed budget tightly across offline and online training,
    whereas QAM with external KL regularization (blue) lets the realized KL
    drift well over it, with corresponding success rate degradation. Shaded regions
    denote $\pm$1 standard deviation across seeds.}
    \label{fig:full_robomimic_can_kl_compare_grid}
\end{figure}
\begin{figure}[p]
    \centering
    \setlength{\abovecaptionskip}{4pt}
    \setlength{\belowcaptionskip}{0pt}
    \includegraphics[width=\linewidth,height=0.85\textheight,keepaspectratio]{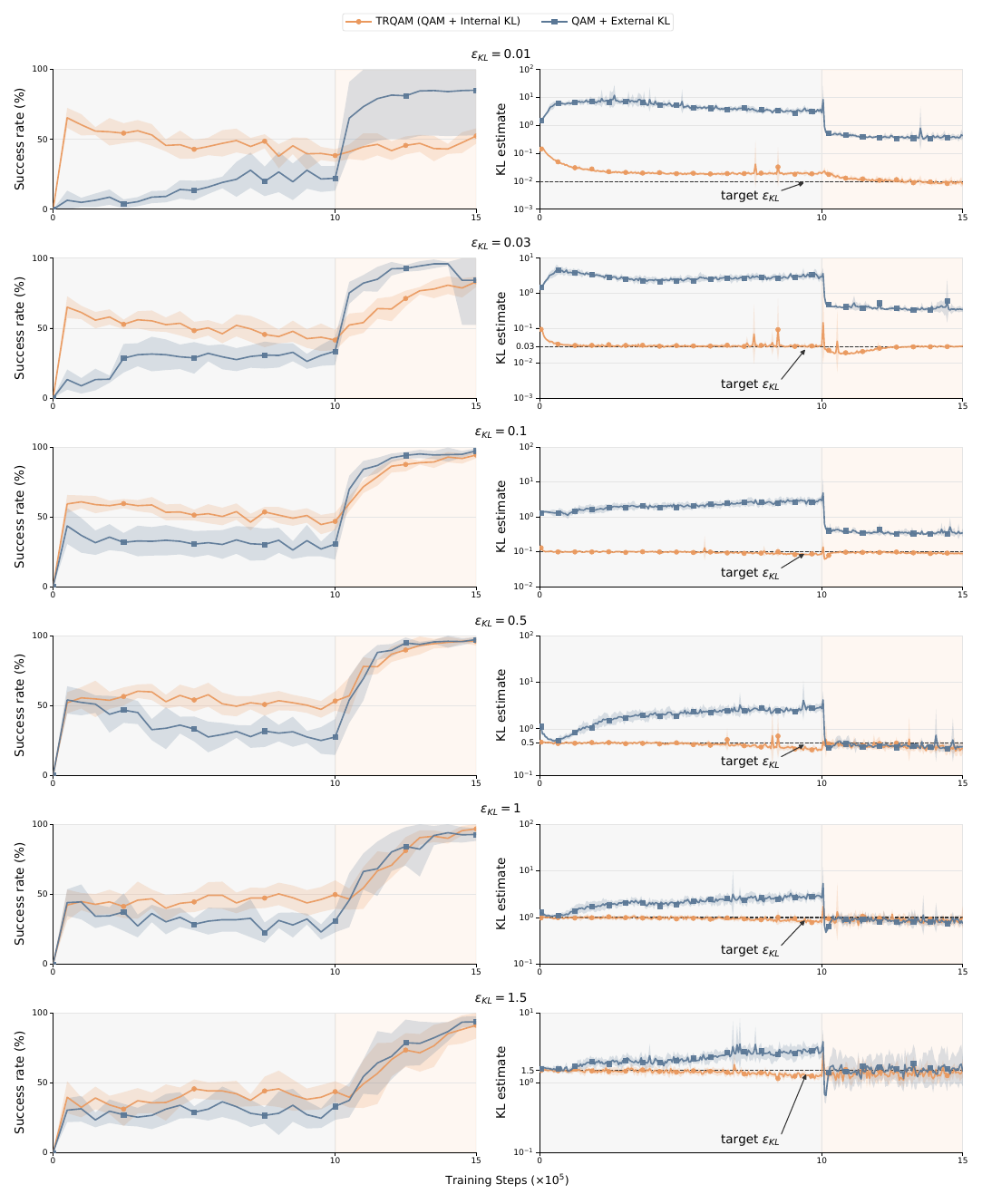}
    \caption{\textbf{Internal vs.\ external KL regularization on
    \texttt{Robomimic-square} across all budgets (8 seeds).} Each row plots
    one target KL budget $\varepsilon_{\mathrm{KL}} \in \{0.01,\allowbreak 0.03,\allowbreak 0.1,\allowbreak 0.5,\allowbreak 1.0,\allowbreak 1.5\}$.
    The left column shows success rate over training steps, and the right column shows realized
    path-space KL with the target shown as a dashed line. TRQAM (orange) tracks
    each prescribed budget tightly across offline and online training,
    whereas QAM with external KL regularization (blue) lets the realized KL
    drift well over it, with corresponding success rate degradation. Shaded regions
    denote $\pm$1 standard deviation across seeds.}
    \label{fig:full_robomimic_square_kl_compare_grid}
\end{figure}

\begin{table}[p] 
    \centering
    \caption{\textbf{Full offline results at 1M training steps (8 seeds).} Mean success rate (\%) $\pm$ one standard deviation.}
    \label{tab:full-results}
    \setlength{\abovecaptionskip}{4pt}

    \newcommand{\gval}[2]{{\color{gray!70}#1}\,{\color{gray!50}\scriptsize$\pm$#2}}
    \newcommand{\bmax}[2]{\underline{#1}\,{\color{gray!50}\scriptsize$\pm$#2}}
    
    \newcommand{\gavg}[1]{{\color{gray!70}#1}}
    \newcommand{\bavg}[1]{\underline{#1}}
    \newcommand{\ourmax}[2]{{\color{GoogleBlue}\textbf{\underline{#1}}}\,{\color{gray!50}\scriptsize$\pm$#2}}
    \newcommand{\ouravg}[1]{{\color{GoogleBlue}\textbf{\underline{#1}}}}

    \makebox[\textwidth][c]{
        \begin{minipage}{\textwidth}
            \centering
            \renewcommand{\arraystretch}{1.0} 
            \setlength{\aboverulesep}{2pt}     
            \setlength{\belowrulesep}{2pt}     

            \adjustbox{max width=\textwidth, max totalheight=0.97\textheight}{
                \begin{tabular}{cl cccccc >{\columncolor{ourlightblue}}c}
                \toprule
                 &  & \texttt{FQL} & \texttt{CGQL-L} & \texttt{DSRL} & \texttt{IFQL} & \texttt{QAM} & \texttt{QAM-E} & \texttt{TRQAM} \\
                \midrule
                 & \texttt{task1} & \gval{24}{44} & \gval{62}{38} & \gval{45}{12} & \gval{7}{11} & \bmax{96}{3} & \gval{93}{5} & \ourmax{95}{4} \\
                 & \texttt{task2} & \gval{0}{0} & \gval{21}{36} & \gval{73}{7} & \gval{9}{14} & \gval{36}{49} & \bmax{91}{4} & \ourmax{87}{6} \\
                 & \texttt{task3} & \gval{73}{14} & \gval{88}{5} & \gval{84}{5} & \gval{62}{24} & \gval{86}{7} & \bmax{93}{3} & \ourmax{93}{3} \\
                 & \texttt{task4} & \gval{0}{0} & \gval{0}{0} & \gval{0}{0} & \gval{35}{22} & \gval{0}{0} & \gval{66}{9} & \ourmax{75}{17} \\
                 & \texttt{task5} & \gval{92}{4} & \gval{68}{42} & \gval{63}{8} & \gval{34}{24} & \bmax{94}{3} & \gval{89}{5} & \ourmax{96}{3} \\
                \multirow{-6}{*}{\texttt{antmaze-large}} & \texttt{agg. (5 tasks)} & \gval{38}{9} & \gval{48}{7} & \gval{53}{2} & \gval{29}{8} & \gval{62}{9} & \gval{86}{3} & \ourmax{89}{4} \\
                \midrule
                 & \texttt{task1} & \gval{1}{1} & \gval{11}{12} & \gval{2}{3} & \gval{0}{0} & \gval{37}{9} & \gval{0}{0} & \ourmax{48}{5} \\
                 & \texttt{task2} & \gval{0}{0} & \gval{14}{12} & \gval{0}{0} & \bmax{43}{13} & \gval{18}{19} & \gval{0}{0} & \gval{21}{12} \\
                 & \texttt{task3} & \gval{0}{0} & \gval{0}{0} & \gval{0}{0} & \bmax{14}{8} & \gval{2}{2} & \gval{0}{0} & \gval{5}{5} \\
                 & \texttt{task4} & \gval{0}{0} & \gval{0}{0} & \gval{0}{0} & \gval{1}{2} & \gval{20}{24} & \gval{0}{0} & \ourmax{58}{13} \\
                 & \texttt{task5} & \gval{10}{29} & \gval{8}{23} & \gval{2}{2} & \gval{2}{2} & \gval{69}{8} & \gval{28}{38} & \ourmax{75}{7} \\
                \multirow{-6}{*}{\texttt{antmaze-giant-10M}} & \texttt{agg. (5 tasks)} & \gval{2}{6} & \gval{7}{5} & \gval{1}{1} & \gval{12}{3} & \gval{29}{4} & \gval{6}{8} & \ourmax{41}{4} \\
                \midrule
                
                 & \texttt{task1} & \gval{84}{21} & \gval{84}{9} & \gval{24}{23} & \bmax{93}{3} & \gval{33}{34} & \gval{28}{30} & \ourmax{87}{5} \\
                 & \texttt{task2} & \bmax{99}{1} & \bmax{99}{1} & \gval{88}{6} & \gval{94}{4} & \bmax{100}{1} & \bmax{98}{2} & \gval{89}{7} \\
                 & \texttt{task3} & \gval{89}{9} & \gval{0}{0} & \gval{64}{27} & \bmax{96}{4} & \gval{90}{8} & \gval{74}{18} & \ourmax{90}{6} \\
                 & \texttt{task4} & \gval{0}{0} & \gval{0}{0} & \gval{0}{0} & \bmax{82}{8} & \gval{0}{0} & \gval{0}{0} & \gval{61}{7} \\
                 & \texttt{task5} & \bmax{99}{1} & \bmax{100}{1} & \gval{88}{5} & \bmax{99}{2} & \bmax{100}{1} & \bmax{99}{1} & \ourmax{95}{4} \\
                \multirow{-6}{*}{\texttt{humanoidmaze-medium}} & \texttt{agg. (5 tasks)} & \gval{74}{5} & \gval{57}{2} & \gval{53}{10} & \bmax{93}{2} & \gval{64}{7} & \gval{60}{6} & \gval{84}{3} \\
                \midrule
                
                 & \texttt{task1} & \gval{1}{2} & \gval{16}{11} & \gval{1}{1} & \gval{32}{5} & \gval{6}{10} & \gval{9}{14} & \ourmax{45}{8} \\
                 & \texttt{task2} & \gval{0}{0} & \gval{0}{0} & \gval{0}{0} & \gval{6}{9} & \gval{0}{0} & \gval{0}{0} & \ourmax{10}{6} \\
                 & \texttt{task3} & \gval{7}{4} & \gval{16}{8} & \gval{2}{3} & \bmax{76}{6} & \gval{7}{5} & \gval{3}{4} & \gval{34}{15} \\
                 & \texttt{task4} & \gval{0}{0} & \gval{0}{1} & \gval{0}{0} & \gval{35}{26} & \gval{0}{0} & \gval{0}{0} & \ourmax{44}{13} \\
                 & \texttt{task5} & \gval{0}{1} & \gval{0}{0} & \gval{1}{1} & \gval{0}{0} & \gval{5}{11} & \gval{9}{18} & \ourmax{45}{6} \\
                \multirow{-6}{*}{\texttt{humanoidmaze-large}} & \texttt{agg. (5 tasks)} & \gval{2}{1} & \gval{6}{3} & \gval{1}{1} & \gval{30}{7} & \gval{4}{3} & \gval{4}{5} & \ourmax{36}{4} \\
                \midrule
                
                 & \texttt{task1} & \bmax{100}{0} & \bmax{100}{0} & \bmax{100}{0} & \bmax{99}{1} & \bmax{100}{0} & \bmax{100}{0} & \ourmax{100}{0} \\
                 & \texttt{task2} & \bmax{100}{1} & \bmax{99}{1} & \bmax{100}{0} & \gval{2}{2} & \bmax{100}{0} & \bmax{100}{0} & \ourmax{100}{0} \\
                 & \texttt{task3} & \gval{82}{5} & \gval{93}{7} & \bmax{100}{1} & \gval{77}{7} & \gval{94}{4} & \gval{93}{4} & \ourmax{100}{1} \\
                 & \texttt{task4} & \gval{60}{22} & \gval{0}{0} & \bmax{100}{0} & \gval{2}{2} & \gval{26}{21} & \gval{22}{30} & \ourmax{93}{6} \\
                 & \texttt{task5} & \bmax{8}{7} & \gval{0}{0} & \gval{0}{0} & \gval{0}{0} & \gval{0}{0} & \gval{0}{0} & \gval{0}{1} \\
                \multirow{-6}{*}{\texttt{scene}} & \texttt{agg. (5 tasks)} & \gval{70}{5} & \gval{58}{1} & \bmax{80}{0} & \gval{36}{1} & \gval{64}{4} & \gval{63}{6} & \ourmax{79}{1} \\
                \midrule
                
                 & \texttt{task1} & \gval{87}{11} & \gval{0}{0} & \bmax{100}{0} & \bmax{100}{1} & \gval{75}{15} & \gval{99}{2} & \ourmax{100}{0} \\
                 & \texttt{task2} & \gval{37}{39} & \gval{0}{0} & \bmax{100}{0} & \gval{23}{7} & \gval{0}{0} & \bmax{100}{0} & \ourmax{100}{0} \\
                 & \texttt{task3} & \gval{0}{0} & \gval{0}{0} & \bmax{100}{0} & \gval{82}{11} & \gval{0}{0} & \gval{76}{11} & \ourmax{100}{0} \\
                 & \texttt{task4} & \gval{0}{0} & \gval{0}{0} & \bmax{100}{0} & \gval{29}{8} & \gval{0}{0} & \gval{87}{6} & \ourmax{100}{1} \\
                 & \texttt{task5} & \gval{3}{6} & \gval{0}{0} & \bmax{100}{0} & \gval{86}{9} & \gval{0}{1} & \gval{84}{20} & \ourmax{100}{0} \\
                \multirow{-6}{*}{\texttt{puzzle-3x3}} & \texttt{agg. (5 tasks)} & \gval{25}{10} & \gval{0}{0} & \bmax{100}{0} & \gval{64}{4} & \gval{15}{3} & \gval{89}{4} & \ourmax{100}{0} \\
                \midrule
                
                 & \texttt{task1} & \gval{3}{3} & \gval{1}{1} & \gval{90}{5} & \gval{74}{11} & \gval{3}{4} & \gval{90}{10} & \ourmax{100}{0} \\
                 & \texttt{task2} & \gval{6}{9} & \gval{0}{0} & \gval{9}{9} & \gval{4}{3} & \gval{1}{1} & \gval{43}{22} & \ourmax{99}{1} \\
                 & \texttt{task3} & \gval{23}{15} & \gval{0}{0} & \gval{81}{20} & \gval{84}{6} & \gval{2}{3} & \gval{57}{10} & \ourmax{98}{2} \\
                 & \texttt{task4} & \gval{7}{5} & \gval{0}{0} & \gval{76}{28} & \gval{46}{13} & \gval{0}{1} & \gval{31}{16} & \ourmax{100}{1} \\
                 & \texttt{task5} & \gval{9}{16} & \gval{0}{0} & \gval{49}{43} & \gval{3}{1} & \gval{1}{1} & \gval{50}{22} & \ourmax{98}{5} \\
                \multirow{-6}{*}{\texttt{puzzle-4x4-10M}} & \texttt{agg. (5 tasks)} & \gval{9}{7} & \gval{0}{0} & \gval{61}{8} & \gval{42}{4} & \gval{1}{1} & \gval{54}{8} & \ourmax{99}{1} \\
                \midrule
                
                 & \texttt{task1} & \gval{75}{10} & \gval{83}{4} & \gval{85}{7} & \gval{18}{7} & \gval{96}{3} & \gval{94}{4} & \ourmax{98}{2} \\
                 & \texttt{task2} & \gval{59}{11} & \gval{66}{6} & \gval{85}{5} & \gval{9}{2} & \gval{81}{6} & \gval{82}{7} & \ourmax{92}{3} \\
                 & \texttt{task3} & \gval{39}{8} & \gval{72}{7} & \bmax{82}{10} & \gval{8}{4} & \gval{75}{7} & \gval{67}{8} & \ourmax{80}{12} \\
                 & \texttt{task4} & \gval{6}{3} & \gval{15}{5} & \gval{35}{6} & \gval{2}{2} & \gval{24}{6} & \gval{29}{6} & \ourmax{54}{8} \\
                 & \texttt{task5} & \gval{44}{9} & \gval{37}{6} & \gval{71}{4} & \gval{7}{3} & \gval{78}{6} & \bmax{82}{6} & \ourmax{82}{6} \\
                \multirow{-6}{*}{\texttt{cube-double}} & \texttt{agg. (5 tasks)} & \gval{44}{4} & \gval{55}{2} & \gval{72}{4} & \gval{9}{2} & \gval{71}{2} & \gval{71}{3} & \ourmax{81}{3} \\
                \midrule
                
                 & \texttt{task1} & \gval{26}{20} & \gval{2}{3} & \gval{70}{18} & \gval{51}{22} & \gval{56}{22} & \gval{26}{11} & \ourmax{87}{8} \\
                 & \texttt{task2} & \gval{1}{1} & \gval{0}{0} & \gval{16}{7} & \gval{16}{3} & \gval{15}{11} & \gval{14}{9} & \ourmax{42}{9} \\
                 & \texttt{task3} & \gval{8}{6} & \gval{0}{0} & \gval{40}{8} & \gval{28}{8} & \gval{21}{11} & \gval{9}{6} & \ourmax{45}{9} \\
                 & \texttt{task4} & \gval{2}{3} & \gval{0}{0} & \gval{15}{5} & \gval{1}{1} & \gval{4}{3} & \gval{3}{3} & \ourmax{30}{13} \\
                 & \texttt{task5} & \gval{1}{1} & \gval{0}{0} & \gval{31}{12} & \gval{23}{8} & \gval{1}{1} & \gval{2}{4} & \ourmax{48}{17} \\
                \multirow{-6}{*}{\texttt{cube-triple-10M}} & \texttt{agg. (5 tasks)} & \gval{7}{5} & \gval{0}{1} & \gval{34}{6} & \gval{24}{7} & \gval{19}{6} & \gval{11}{4} & \ourmax{50}{5} \\
                \midrule
                
                 & \texttt{task1} & \gval{36}{25} & \gval{7}{5} & \gval{36}{12} & \gval{15}{11} & \bmax{70}{10} & \gval{45}{19} & \ourmax{66}{16} \\
                 & \texttt{task2} & \gval{4}{5} & \gval{0}{1} & \gval{5}{3} & \gval{5}{7} & \gval{2}{4} & \gval{0}{0} & \ourmax{18}{11} \\
                 & \texttt{task3} & \gval{3}{3} & \gval{0}{0} & \gval{2}{2} & \gval{6}{8} & \bmax{16}{11} & \gval{2}{4} & \ourmax{11}{7} \\
                 & \texttt{task4} & \gval{0}{0} & \gval{0}{0} & \gval{0}{1} & \gval{0}{0} & \bmax{2}{2} & \gval{0}{0} & \ourmax{3}{2} \\
                 & \texttt{task5} & \gval{0}{0} & \gval{0}{0} & \gval{1}{1} & \bmax{2}{3} & \gval{0}{0} & \gval{0}{0} & \gval{0}{0} \\
                \multirow{-6}{*}{\texttt{cube-quadruple-100M}} & \texttt{agg. (5 tasks)} & \gval{9}{5} & \gval{1}{1} & \gval{9}{3} & \gval{6}{3} & \bmax{18}{3} & \gval{9}{3} & \ourmax{19}{5} \\
                \midrule
                
                \texttt{all} & \texttt{agg. (50 tasks)} & \gavg{28} & \gavg{23} & \gavg{46} & \gavg{35} & \gavg{35} & \gavg{45} & \ouravg{68} \\
                
                \bottomrule
                \end{tabular}
            }
        \end{minipage}
    }
\end{table}
\begin{figure}[p]
    \centering
    \setlength{\abovecaptionskip}{4pt}
    \setlength{\belowcaptionskip}{0pt}
    \includegraphics[width=\linewidth,height=0.85\textheight,keepaspectratio]{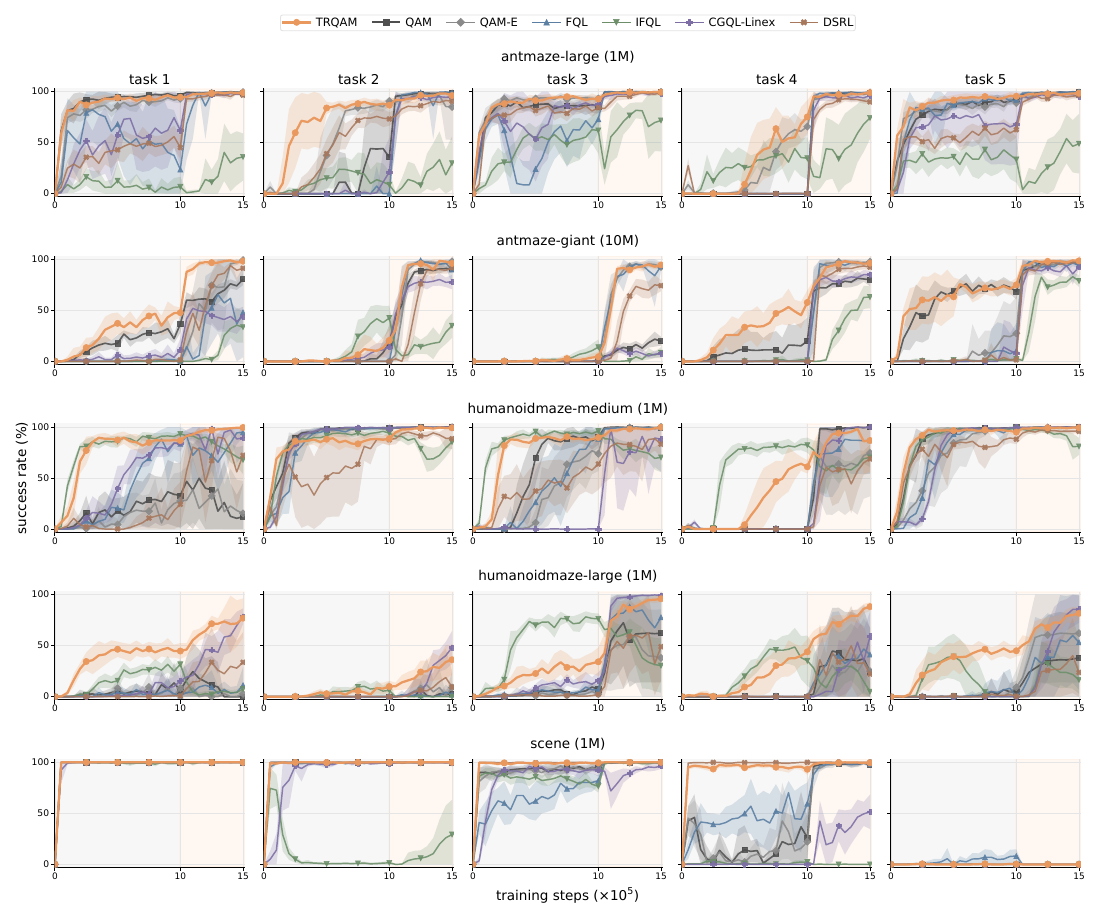}
    \caption{\textbf{Per-task offline-to-online learning curves on
    OGBench~\citep{ogbench_park2025} (8 seeds).} The domains are
    \texttt{antmaze-large}, \texttt{antmaze-giant},
    \texttt{humanoidmaze-medium}, \texttt{humanoidmaze-large}, and
    \texttt{scene}. Each panel reports success rate over training steps
    for all seven methods (TRQAM, QAM, QAM-E, FQL, IFQL, CGQL-Linex, DSRL). The offline-to-online transition occurs at $10^6$ steps. Shaded regions
    denote $\pm$1 standard deviation across seeds.}
    \label{fig:per_task_curves_1}
\end{figure}
\begin{figure}[p]
    \centering
    \setlength{\abovecaptionskip}{4pt}
    \setlength{\belowcaptionskip}{0pt}
    \includegraphics[width=\linewidth,height=0.85\textheight,keepaspectratio]{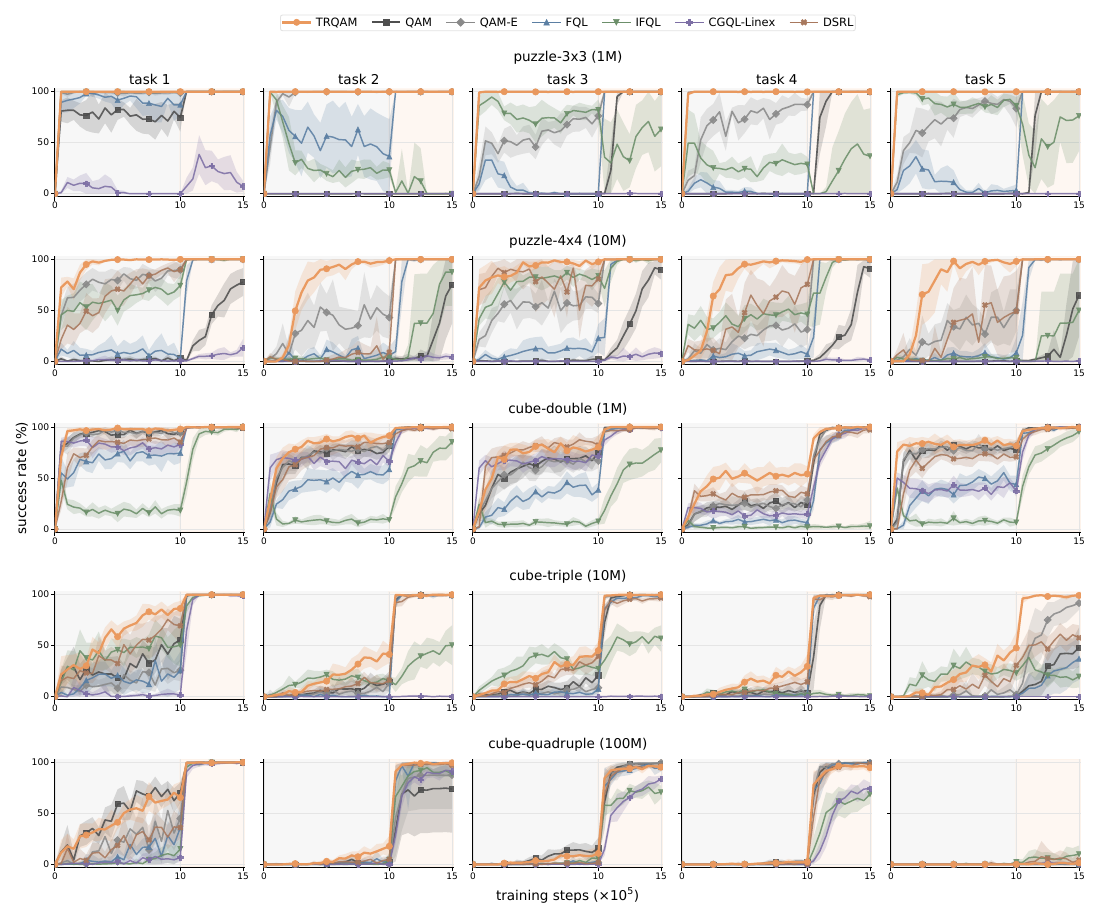}
    \caption{\textbf{Per-task offline-to-online learning curves on
    OGBench~\citep{ogbench_park2025} (8 seeds).} The domains are
    \texttt{puzzle-3x3}, \texttt{puzzle-4x4},
    \texttt{cube-double}, \texttt{cube-triple}, and
    \texttt{cube-quadruple}. Each panel reports success rate over
    training steps for all seven methods (TRQAM, QAM, QAM-E, FQL, IFQL,
    CGQL-Linex, DSRL). The offline-to-online transition occurs at $10^6$
    steps. Shaded regions denote $\pm$1 standard deviation across seeds.}
    \label{fig:per_task_curves_2}
\end{figure}

\clearpage
\section{Broader impacts}
\label{app:impacts}
TRQAM is a method for stable off-policy fine-tuning of pretrained flow policies, evaluated only on simulated benchmarks. More reliable fine-tuning can reduce the data and compute needed to adapt pretrained robot policies to new tasks. However, a KL bound on the policy does not rule out unsafe behavior once such policies are deployed in the real world, so safety filters and careful evaluation remain necessary. We do not release models or datasets that pose high misuse risk.

\section{Limitations and future work}
\label{app:limitations}
\paragraph{Computational cost.} Computing the Adjoint Matching loss requires a vector-Jacobian product (VJP) through the velocity field at each step of the backward ODE. This cost grows with the model size. Its share of per-step training time does not grow without bound, but it reaches about a fifth for larger actors in our measurements (\cref{app:vjp_cost}), so it is hard to ignore for large policies. Scaling TRQAM to substantially larger pretrained policies, such as vision-language-action models with visual inputs, remains to be tested.

\paragraph{Dependence on the pretrained policy.} TRQAM keeps the fine-tuned policy close to $\pi_{\mathrm{base}}$ by design. Its performance is therefore tied to the quality of the pretrained policy, and a poor prior limits what any trust-region method can recover.

\paragraph{Memoryless OT schedule.} Our closed-form per-step KL estimator relies on the memoryless OT schedule, under which the discretized transitions are Gaussian with a shared covariance. Other schedules whose Euler transitions share this structure admit an analogous estimator, but the exact expressions change, and we leave them to future work.

\paragraph{Finite-sample budget tracking.} \cref{thm:path_kl} is an exact identity for a given control, and the discretized sampler admits a closed-form per-step KL. In practice, however, TRQAM tracks the budget with a Monte Carlo estimate smoothed by an exponential moving average, while the critic, the policy, and $\lambda$ are updated simultaneously. The realized KL therefore stays close to $\varepsilon_{\mathrm{KL}}$ but does not satisfy the budget exactly at every iteration.

\paragraph{Choice of the KL budget.} $\varepsilon_{\mathrm{KL}}$ remains a hyperparameter. A small budget works well on most domains, but domains with large state spaces such as \texttt{puzzle-4x4} benefit from larger budgets (\cref{app:sensitivity}). Because TRQAM tracks time-varying budgets (\cref{app:time_varying}), adjusting $\varepsilon_{\mathrm{KL}}$ automatically from a training signal is a natural direction for future work.

\end{document}